%% file: main.tex
\documentclass{article}

\usepackage{microtype}
\usepackage{graphicx}
\usepackage{subcaption}
\usepackage{booktabs} 
\usepackage{enumitem}

\usepackage{hyperref}

\usepackage[accepted]{template/icml2026}

\usepackage{amsmath}
\usepackage{amssymb}
\usepackage{mathtools}
\usepackage{amsthm}
\usepackage{bbm}
\usepackage{pifont} 
\usepackage[most]{tcolorbox}

\newcommand{\hypbox}[2]{%
\begin{tcolorbox}[colback=white!98!black,colframe=white!30!black,boxsep=1.1pt,top=6.75pt]%
\vspace{1.75pt}%
\textbf{#1}\\[-0.575em]
\noindent\makebox[\textwidth]{\rule{\textwidth}{0.4pt}}
\\[0.25em]
#2
\end{tcolorbox}
}

\newcommand{\mat}[1]{\mathbf{#1}}      
\newcommand{\vect}[1]{\mathbf{#1}}     
\newcommand{\Prob}{\mathbb{P}}         
\newcommand{\E}{\mathbb{E}}            
\newcommand{\R}{\mathbb{R}}            

\usepackage{aliascnt} 
\usepackage[nameinlink,noabbrev]{cleveref}

\theoremstyle{plain}
\newtheorem{theorem}{Theorem}[section]

\newaliascnt{proposition}{theorem}
\newtheorem{proposition}[proposition]{Proposition}
\aliascntresetthe{proposition}

\newaliascnt{lemma}{theorem}
\newtheorem{lemma}[lemma]{Lemma}
\aliascntresetthe{lemma}

\newaliascnt{corollary}{theorem}
\newtheorem{corollary}[corollary]{Corollary}
\aliascntresetthe{corollary}

\theoremstyle{definition}
\newaliascnt{definition}{theorem}
\newtheorem{definition}[definition]{Definition}
\aliascntresetthe{definition}

\newaliascnt{assumption}{theorem}
\newtheorem{assumption}[assumption]{Assumption}
\aliascntresetthe{assumption}

\theoremstyle{remark}
\newaliascnt{remark}{theorem}

\aliascntresetthe{remark}

\crefname{theorem}{Theorem}{Theorems}
\Crefname{theorem}{Theorem}{Theorems}

\crefname{proposition}{Proposition}{Propositions}
\Crefname{proposition}{Proposition}{Propositions}

\crefname{lemma}{Lemma}{Lemmas}
\Crefname{lemma}{Lemma}{Lemmas}

\crefname{corollary}{Corollary}{Corollaries}
\Crefname{corollary}{Corollary}{Corollaries}

\crefname{definition}{Definition}{Definitions}
\Crefname{definition}{Definition}{Definitions}

\crefname{assumption}{Assumption}{Assumptions}
\Crefname{assumption}{Assumption}{Assumptions}

\crefname{remark}{Remark}{Remarks}
\Crefname{remark}{Remark}{Remarks}

\crefname{appendix}{Appendix}{Appendices}
\Crefname{appendix}{Appendix}{Appendices}
\crefname{subappendix}{Appendix}{Appendices}
\Crefname{subappendix}{Appendix}{Appendices}
\crefname{subsubappendix}{Appendix}{Appendices}
\Crefname{subsubappendix}{Appendix}{Appendices}

\usepackage[acronym,nogroupskip,nopostdot]{glossaries}
\makeglossaries

\newacronym{cka}{CKA}{Centered Kernel Alignment}
\newacronym{rsa}{RSA}{Representational Similarity Analysis}
\newacronym{cca}{CCA}{Canonical Correlation Analysis}
\newacronym{svcca}{SVCCA}{Singular Vector Canonical Correlation Analysis}
\newacronym{pwcca}{PWCCA}{Projection Weighted Canonical Correlation Analysis}
\newacronym{hsic}{HSIC}{Hilbert-Schmidt Independence Criterion}
\newacronym{mknn}{mKNN}{mutual $k$-Nearest Neighbors}
\newacronym{rmt}{RMT}{Random Matrix Theory}
\newacronym{evt}{EVT}{Extreme Value Theory}
\newacronym{prh}{PRH}{Platonic Representation Hypothesis}
\newacronym{cdf}{CDF}{Cumulative Distribution Function}
\newacronym{kde}{KDE}{Kernel Density Estimation}
\newacronym{snr}{SNR}{Signal-to-Noise Ratio}
\newacronym{fwer}{FWER}{Family-Wise Error Rate}
\newacronym{bh}{BH}{Benjamini-Hochberg}
\newacronym{by}{BY}{Benjamini-Yekutieli}
\newacronym{prds}{PRDS}{Positive Regression Dependency on Subsets}
\newacronym{wuc}{WUC}{Whitened Unbiased Covariance}
\glsdisablehyper

\icmltitlerunning{
Revisiting the Platonic Representation Hypothesis: An Aristotelian View
}

\begin{document}
\frenchspacing

\twocolumn[
  \icmltitle{
  Revisiting the Platonic Representation Hypothesis:
  An Aristotelian View
  }


  \icmlsetsymbol{equal}{*}

  \begin{icmlauthorlist}
    \icmlauthor{Fabian Gr\"oger}{equal,epfl,unibas,hslu}
    \icmlauthor{Shuo Wen}{equal,epfl}
    \icmlauthor{Maria Brbi\'c}{epfl}
  \end{icmlauthorlist}

  \icmlaffiliation{epfl}{EPFL}
  \icmlaffiliation{unibas}{University of Basel}
  \icmlaffiliation{hslu}{HSLU}

  \icmlcorrespondingauthor{Maria Brbi\'c}{maria.brbic@epfl.ch}

  \icmlkeywords{
    Platonic Representation Hypothesis, 
    Representation Similarity,
    Hypothesis Testing,
    Representation Learning,
    Unsupervised Learning,
   }

  \vskip 0.3in
]

\printAffiliationsAndNotice{%
\hspace*{-1.6em}
\textbf{\fontsize{9.5pt}{10pt}
\selectfont Project Page:} 
{\fontsize{9.5pt}{10pt}\selectfont
\href{https://brbiclab.epfl.ch/projects/aristotelian/}{\texttt{brbiclab.epfl.ch/aristotelian}}}
\\
{\textbf{\fontsize{9.5pt}{10pt}\selectfont Code:}} \href{https://github.com/mlbio-epfl/Aristotelian}{\texttt{github.com/mlbio-epfl/aristotelian}}
\\
\icmlEqualContribution}

\begin{abstract}
The Platonic Representation Hypothesis suggests that representations from neural networks are converging to a common statistical model of reality. 
We show that the existing metrics used to measure representational similarity are \textit{confounded by network scale}: increasing model depth or width can systematically inflate representational similarity scores.
To correct these effects, we introduce a permutation-based null-calibration framework that transforms any representational similarity metric into a calibrated score with statistical guarantees.
We revisit the Platonic Representation Hypothesis with our calibration framework, which reveals a nuanced picture: the apparent convergence reported by global spectral measures largely disappears after calibration, while local neighborhood similarity, but not local distances, retains significant agreement across different modalities. 
Based on these findings, we propose the \textit{Aristotelian Representation Hypothesis}: representations in neural networks are converging to shared local neighborhood relationships. 
\end{abstract}

\section{Introduction}\label{sec:intro}
\input{sections/Intro}

\section{Related work}\label{sec:related}
\input{sections/Related}

\section{Problem setup}\label{sec:setup}
\input{sections/Setup}

\section{Theoretical motivation: spurious alignment}\label{sec:theory}
\input{sections/TheoryMotivation}

\section{Representational similarity calibration}
\label{sec:methods}
\input{sections/Method}

\section{Experiments}\label{sec:exp-setup}
\input{sections/Experiment}

\section{Conclusion}\label{sec:conclusion}
\input{sections/Conclusion}

\section*{Acknowledgements}
We thank Artyom Gadetsky, Siba Smarak Panigrahi, Debajyoti Dasgupta, David Frühbuss, Shin Matsushima, Rishubh Singh, Adriana Moreno Castan, and Gioele La Manno for their valuable suggestions, which helped improve the manuscript.
We are especially grateful to Simone Lionetti for additional input and support.
We gratefully acknowledge the support of the Swiss National Science Foundation (SNSF) starting grant TMSGI2\_226252/1, SNSF grant IC00I0\_231922, SNSF grant 10.004.411, and the Swiss AI Initiative Large Call \#32. 
M.B. is a CIFAR Fellow in the Multiscale Human Program.

\clearpage
\section*{Impact statement}

This paper presents work whose goal is to advance the field of machine learning. 
There are many potential societal consequences of our work, none of which we feel must be specifically highlighted here.

\bibliography{references}
\bibliographystyle{template/icml2026}

\newpage
\appendix
\crefalias{section}{appendix}
\crefalias{subsection}{appendix}
\crefalias{subsubsection}{appendix}
\onecolumn

\input{sections/Appendix}

\end{document}

%% file: sections/Intro.tex
\begin{figure}[t]
    \hypbox{The Aristotelian Representation Hypothesis}{%
    Neural networks, trained with different objectives on different data and modalities, are converging to shared \emph{local neighborhood relationships}.}
    \centering
    \includegraphics[width=\linewidth]{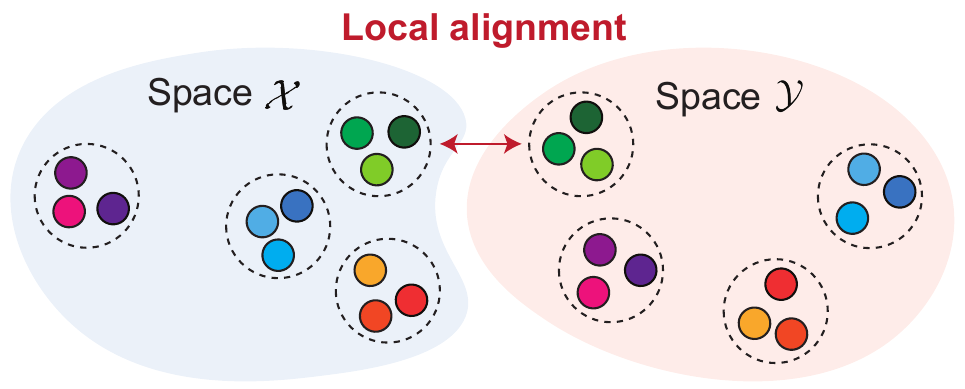}
    \caption{
    \textbf{The Aristotelian Representation Hypothesis:} Local relations (``who is near whom''), rather than distances between data points, are preserved across different representation spaces.
   Representation learning algorithms converge toward shared \emph{local neighborhood relationships}.
    }
    \label{fig:local-prh}
    \vspace{-4mm}
\end{figure}

Quantifying the similarity between neural network representations is central to understanding the geometry of learned representation spaces~\citep{raghu2017svcca,nguyen2021do}, guiding transfer learning decisions~\citep{pmlr-v97-kornblith19a,neyshabur2020being,groger2026limited}, and relating artificial representations to neural measurements in neuroscience~\citep{schrimpf2018brain}.
The Platonic Representation Hypothesis~\citep{pmlr-v235-huh24a} posits that as neural networks scale, representations across different modalities become increasingly similar, suggesting convergence to a shared statistical model of reality. 
This hypothesis has motivated a growing literature that uses representational similarity to study whether scaling produces universal structure across models~\citep{pmlr-v235-huh24a,Maniparambil_2024_CVPR,tjandrasuwita2025understanding,zhu2025dynamic}.
To measure representational similarity across models, different metrics have been proposed, such as Centered Kernel Alignment~\citep{pmlr-v97-kornblith19a}, Canonical Correlation Analysis~\citep{weenink2003canonical}, Representational Similarity Analysis~\citep{kriegeskorte2008rsa}, and mutual $k$-Nearest Neighbors~\citep{pmlr-v235-huh24a}.

\begin{figure}[t]
  \centering
  \includegraphics[width=\linewidth]{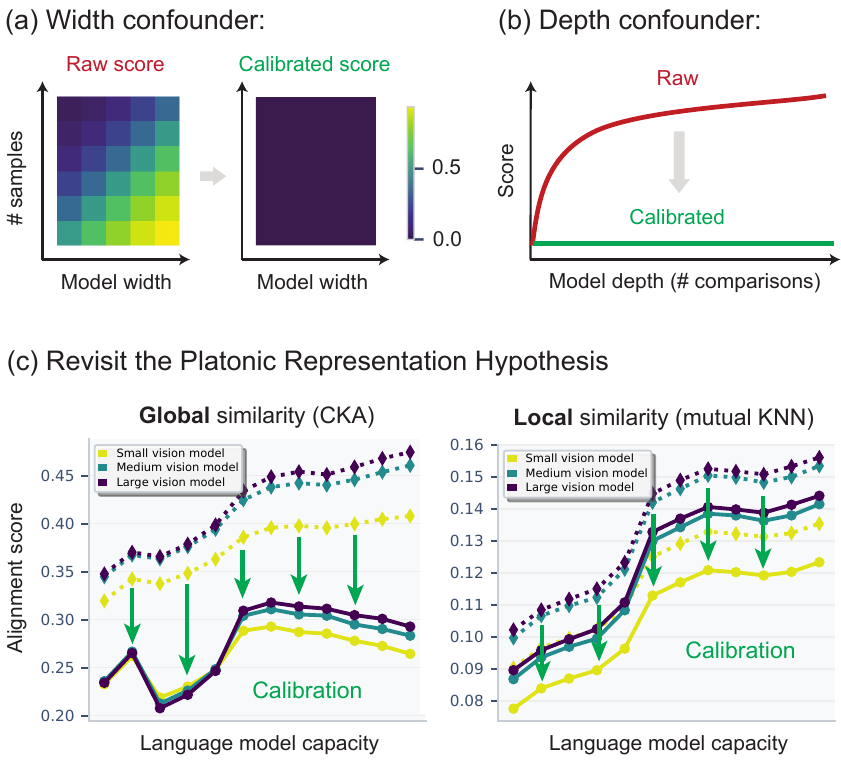}
  \caption{
  \textbf{Null calibration removes width and depth confounders.}
  (a)~\emph{Width} confounder: raw scores exhibit positive null baselines that increase with the ratio of dimension (width) of the spaces and the number of samples; calibration collapses them to zero.
  (b)~\emph{Depth} confounder: selection-based summaries (max over layers) inflate with search space size; aggregation-aware calibration removes this.
  (c)~After calibration, global metrics lose their convergence trend, while local metrics retain significant alignment.
  }
  \vspace{-5mm}
  \label{fig:intro}
\end{figure}

In this work, we identify two pervasive confounders that distort representational similarity measurements.
The first is the \emph{model width}: when the embedding dimension increases relative to the sample size, interaction-matrix-based similarity metrics exhibit a systematic positive baseline even when representations are independent.
This spurious similarity is a general consequence of dimensionality-driven null inflation: the expected similarity under independence does not vanish but instead depends on both the representation dimensionality and the sample size (\Cref{fig:intro}a).
As a result, wider models can appear more aligned simply because their representations live in higher-dimensional spaces.
The second confounder is the \emph{model depth}.
Many analyses do not compare individual layer pairs, because it is unknown where similarity arises~\citep{schrimpf2018brain,pmlr-v235-huh24a}.
Instead, they search over all pairs and report a summary statistic such as the maximum.
This is widespread practice: studies routinely compute similarity over many layer pairs and report an aggregate such as the maximum or best-matching layer~\citep{raghu2017svcca,pmlr-v97-kornblith19a,schrimpf2018brain,pmlr-v235-huh24a,kapoor2025bridging}.
Taking a maximum over many comparisons inflates the reported score even if there is no similarity, since the expected maximum of independent draws exceeds the mean.
This inflation grows with the number of comparisons, so deeper models can appear more aligned simply because more layer pairs are compared (\Cref{fig:intro}b).
Together, these confounders undermine the comparative use of representational similarity without calibration.

To address these issues, we introduce the \emph{null-calibration for representational similarity}, a general permutation-based framework that transforms any similarity metric into a calibrated score with a principled null reference, here defined as no relationship (\Cref{fig:intro}).
The core idea is to measure how extreme an observed similarity is relative to an empirical null distribution obtained by breaking sample correspondences.
For scalar comparisons (\textit{i.e.}, width confounder), we estimate a critical threshold from the null distribution and define a calibrated score that is zero when the observed similarity falls below this threshold and rescaled to preserve the maximum at one.
For selection-based summaries (\textit{i.e.}, depth confounder), we apply \emph{aggregation-aware} calibration.
We compute the null distribution of the same aggregate statistic that is ultimately reported (\textit{e.g.}, the maximum over all layer pairs), thereby calibrating the selection step itself.

These observations raise a question: \textit{Does the Platonic Representation Hypothesis still hold once similarity is calibrated? }
We find that, after calibration, the previously reported convergence in global metrics~\citep{pmlr-v235-huh24a,Maniparambil_2024_CVPR,tjandrasuwita2025understanding} largely disappears, suggesting it was driven primarily by width and depth confounders, whereas local neighborhood-based metrics retain significant cross-modal alignment (\Cref{fig:intro}c). 
However, we also observe that the convergence in local distances is not preserved, suggesting that only local neighborhood relationships are aligned.
Motivated by these results, we refine the original Platonic Representation Hypothesis and propose the \textit{Aristotelian Representation Hypothesis}\footnote{Calling this refinement \textit{Aristotelian}: it emphasizes learned representations converging on relations among instances (who is near whom) rather than the idea of convergence toward a globally matching structure.}: Neural networks, trained with different objectives on different data and modalities, converge to shared local neighborhood relationships~(\Cref{fig:local-prh}). 
We name it after the Greek philosopher Aristotle, who was a student of Plato and, in his Categories, established the principles of relatives \cite{aristotle}. 

%% file: sections/Related.tex
\paragraph{Representational similarity metrics.}
\looseness=-1
A long line of work compares representation spaces using a variety of similarity measures.
\Gls{cca}~\citep{hotelling1992relations} and variants such as \gls{svcca}~\citep{raghu2017svcca} and \gls{pwcca}~\citep{morcos2018insights} compare subspaces up to linear transformations, while Procrustes- and shape-based distances compare representations up to restricted alignment classes~\citep{ding2021grounding,williams2021shapemetrics}.
\Gls{cka}~\citep{pmlr-v97-kornblith19a} has become a dominant tool for comparing deep representations, with kernelized variants extending to nonlinear similarity.
\Gls{rsa}~\citep{kriegeskorte2008rsa}, originating in neuroscience, compares representational dissimilarity matrices rather than feature bases.
Neighborhood-based approaches, such as \gls{mknn}~\citep{pmlr-v235-huh24a}, capture local topological consistency rather than global alignment.
However, recent evaluations stress that different metrics encode different invariances and can yield qualitatively different conclusions, motivating more robust reporting practices~\citep{klabunde2023survey,klabunde2025resi,ding2021grounding,harvey2024decodable,bo2024functional}.
This sensitivity is especially consequential in the neuro-AI literature, where the choice of similarity measure can profoundly change conclusions about which models best align with the brain~\citep{soni2024conclusions}, which is why scores that are calibrated and comparable across metrics and scales are needed.

\paragraph{Reliability of representational similarity metrics.}
\looseness=-1
In finite-sample, high-dimensional regimes, raw similarity scores can be systematically biased.
Recent works~\citep{murphy2024biasedcka,chun2025sparsecka} propose debiased \gls{cka}, but these corrections are \emph{metric-specific}.
For neighborhood-based metrics, no analogous debiasing methods exist despite distance concentration effects that inflate random $k$-NN overlap~\citep{beyer1999nearest,aggarwal2001surprising}.
Other approaches address confounding from \emph{input population structure}. 
For instance, \citet{cui2022deconfounded} propose regression-style deconfounding to remove effects of shared input statistics on \gls{rsa}/\gls{cka}.
Analogous concerns arise for regression-based predictivity in neuroscience, where high-dimensional fits can inflate alignment scores and may fail to identify genuinely brain-like models~\citep{schaeffer2024position}.
A separate reliability issue arises from layer search, where max or top-$k$ aggregation across many layer pairs introduces multiple-comparison inflation.
While resampling-based ``maxT'' procedures~\citep{westfall1993resampling,nichols2002permutation} can calibrate such aggregates, this has not yet been applied in representational similarity studies.
Our calibration framework addresses both finite-sample bias and selection inflation in a unified, metric-agnostic way.

\paragraph{The Platonic Representation Hypothesis.}
A growing body of work examines whether neural networks trained under different conditions converge toward similar representations.
The Platonic Representation Hypothesis~\citep{pmlr-v235-huh24a} posits that as models scale, their representations increasingly converge across architectures and even across modalities such as vision and language, with convergence reported under both global and local similarity measures.
Follow-up work has examined factors influencing these trends, including model size, training duration, and data distribution~\citep{raugel2025disentangling}, and has explored analogous convergence effects in broader settings such as video models~\citep{zhu2025dynamic} and comparisons to biological vision~\citep{marcos2025convergent}.
Prior work attributes convergence to shared factors such as training data distribution, scale, and objective functions~\citep{raugel2025disentangling,pmlr-v235-huh24a}. Our calibration-based analysis does not contradict these findings but refines their interpretation, showing that the evidence for convergence depends on whether it is measured by global or local similarity.
In this work, we revisit the Platonic Representation Hypothesis using our null-calibration framework that controls for width and depth confounders.

%% file: sections/Setup.tex
\subsection{Representation spaces and similarity score}

Let $\mathcal{X} \subseteq \R^{d_x}$ and $\mathcal{Y} \subseteq \R^{d_y}$ be two representation spaces, where $d_x$ and $d_y$ are the respective space dimensions.
For a set of $n$ input samples, let $\mat{X} \in \R^{n \times d_x}$ and $\mat{Y} \in \R^{n \times d_y}$ be the corresponding embeddings in $\mathcal{X}$ and $\mathcal{Y}$. 
We assume row-wise alignment such that the $i$-th row of $\mat{X}$ and $\mat{Y}$ correspond to paired inputs.
We use a similarity score $s(\mathcal{X},\mathcal{Y}) \in \R$ to quantify the agreement between $\mathcal{X}$ and $\mathcal{Y}$.
In practice, we compute it from $\mat{X},\mat{Y}$ and, by a slight abuse of notation, denote it with $s(\mat{X},\mat{Y})$.

We consider three families of metrics:
\textit{(i)}~\textit{spectral:} 
metrics defined on the spectrum of cross-covariance or Gram matrices (\textit{e.g.}, \gls{cka}, \gls{cca}),
\textit{(ii)}~\textit{neighborhood:}
metrics measuring local topological overlap (\textit{e.g.}, \gls{mknn}), 
and \textit{(iii)}~\textit{geometric:}
second-order isomorphism metrics (\textit{e.g.}, \gls{rsa}).
\Cref{app:metrics} provides definitions of the metrics used in this paper.

\subsection{The null hypothesis of independence} 
We claim that a similarity score $s(\mat{X},\mat{Y})$ is uninterpretable without a baseline. 
To provide this baseline, we define the null hypothesis $H_0$ as the absence of a relationship between $\mat{X}$ and $\mat{Y}$ beyond their marginal statistics.
We operationalize $H_0$ via a permutation group $\Pi_n$ acting on sample indices: draw $\pi\sim\mathrm{Unif}(\Pi_n)$
independently of $(\mat{X},\mat{Y})$ and evaluate $s(\mat{X},\pi(\mat{Y}))$, where $\pi(\mat{Y})$ permutes the rows of $\mat{Y}$.

\begin{assumption}[Exchangeability under the null]
\label{assm:exchangeability}
Under $H_0$, the joint distribution of paired samples is invariant to relabeling of correspondences.
For any permutation $\pi\in\Pi_n$, $\Prob_{H_0}(\mat{X},\mat{Y}) = \Prob_{H_0}(\mat{X},\pi(\mat{Y}))$.
\end{assumption}

This assumption implies that if no true relationship exists, the observed pairing is statistically indistinguishable from a random shuffling of the data.
It allows us to construct an empirical null distribution by holding $\mat{X}$ fixed and shuffling the rows of $\mat{Y}$.

\subsection{Baseline problem: non-zero null expectations} 
Ideally, under $H_0$, we desire $\E_{\pi}[s(\mat{X}, \pi(\mat{Y}))] \approx 0$. 
However, for commonly used raw or \emph{biased} estimators, the expected similarity under the null is \textit{not zero},
\begin{equation}
    \mu_{0}(n, d_x, d_y) \,\coloneqq\, \E_{\pi}[s(\mat{X}, \pi(\mat{Y}))].
\end{equation}
This baseline $\mu_0$ is metric- and preprocessing-dependent and can deviate from zero in finite samples.
It also varies with sample size and dimension, thus acting as a confounding variable in comparative studies. 

%% file: sections/TheoryMotivation.tex
We motivate and formalize \emph{why} raw representational similarity metrics fail in cross-scale model comparisons. 
We identify two distinct sources of confounding: \textit{(i)} \textit{the width confounder} driven by representation dimension, and \textit{(ii)} \textit{the depth confounder} driven by the number of layers considered when comparing models. 

\subsection{The width confounder}
\label{sec:theory_rmt}
Many spectral-family similarity metrics, \textit{e.g.}, linear/kernel \gls{cka} and the RV coefficient, can be written as functionals of an \emph{interaction operator} constructed from two representations.
One such operator is the (normalized) cross-covariance
\begin{equation}
    \widetilde{\mat{C}} \;=\; \frac{1}{n-1}\mat{X}_{c}^{\top}\mat{Y}_{c}\in\R^{d_x\times d_y},
\end{equation}
where $\mat{X}_{c}$ and $\mat{Y}_{c}$ denote row-centered representations (\Cref{app:metrics_preproc}).

A common but misleading intuition is that if $\mat{X}$ and $\mat{Y}$ are independent, then $\widetilde{\mat{C}}\approx \mat{0}$
and therefore spectral aggregates should be near zero. 
In high dimension this fails: the null interaction energy is typically non-zero.

\begin{proposition}[Non-vanishing null interaction energy]
\label{prop:nonzero_energy}
Assume the rows are i.i.d.\ with $\E[\vect{x}_i]=\E[\vect{y}_i]=0$,
$\mathrm{Cov}(\vect{x}_i)=\mat{I}_{d_x}$, $\mathrm{Cov}(\vect{y}_i)=\mat{I}_{d_y}$, and $\vect{x}_i$ and $\vect{y}_i$ are independent.
Then
\begin{equation}
    \E_{H_0}\left[\|\widetilde{\mat{C}}\|_F^2\right] \;=\; \frac{d_x d_y}{n-1}.
\end{equation}
\end{proposition}
\noindent\emph{Proof.} See \Cref{app:theory_rmt}.

Since \gls{cka} is scaled by the normalized self-similarity terms, which each scale as $\mathcal{O}(\sqrt{d})$, the resulting null baseline for the metric is thus $\mathcal{O}(d/n)$ to leading order.
We refer to this $\mathcal{O}(d/n)$ effect as the \emph{width confounder}: at a fixed sample size $n$ it grows with the representation width $d$, and it persists even when two representations share a genuine signal (\Cref{prop:width_h1}). 

This aligns with insights from random matrix theory: in high-dimensional regimes ($d\sim n$), the null singular spectrum of interaction operators (after centering/whitening) concentrates into a non-trivial ``noise bulk'' whose upper edge depends on $d/n$ and preprocessing, rather than collapsing to zero \citep{wachter1978strong,muller2002random,livan2018introduction}.
Our framework estimates this null baseline directly via permutation, providing a metric- and pipeline-independent alternative to asymptotic formulas.

\paragraph{Neighborhood metrics follow a different regime.}
While spectral metrics have null baselines scaling as $\mathcal{O}(d/n)$, neighborhood-based metrics such as mutual $k$-NN exhibit different behavior, as they rely on set comparisons rather than interactions.

\begin{proposition}[Null baseline for neighborhood metrics]
\label{prop:null_mknn}
Assume the rows are i.i.d.\ with $\vect{x}_i$ and $\vect{y}_i$ independent, and that pairwise distances are almost surely distinct (\textit{e.g.}, under absolutely continuous distributions).
Then for any $k < n$,
\begin{equation}
    \E_{H_0}\bigg[\mathrm{mKNN}(\mat{X},\mat{Y})\bigg] \;=\; \frac{k}{n-1}.
\end{equation}
\end{proposition}
\noindent\emph{Proof.} See \Cref{app:neighborhood-null}.

In particular, neighborhood metrics have null baselines scaling as $\mathcal{O}(k/n)$.

The difference in null baseline between spectral and neighborhood metrics is substantial:
\textit{(i)} The neighborhood scale $k$ can be fixed consistently across experiments, whereas the embedding dimension $d$ is determined by the model architecture, making it difficult to control in comparison studies.
\textit{(ii)} The neighborhood metrics are much less confounded since $k \ll d$ in typical settings.

Geometric metrics (\textit{e.g.}, Procrustes), similarly to the spectral metrics, are functions of all-pairs distances or inner products, which are the interaction quantities of \Cref{prop:nonzero_energy}.
They thus inherit a width-dependent $\mathcal{O}(d/n)$ baseline rather than the $\mathcal{O}(k/n)$ baseline of rank-based neighborhood metrics. 
We verify this empirically in \Cref{app:null-drift-full}.

\subsection{The depth confounder}
\label{sec:theory_evt}
\looseness=-1
A subtle yet pervasive issue is the comparison of \emph{selection-based} alignment summaries across models.
Let \mbox{$S_{\ell,\ell'} := s(\mat{X}^{(A)}_\ell, \mat{Y}^{(B)}_{\ell'})$} be the similarity between layer $\ell$ of model $A$ and layer $\ell'$ of model $B$.
It is common to summarize the similarity between two models by the maximum alignment score $T_{\max}=\max_{\ell,\ell'} S_{\ell,\ell'}$~\citep{schrimpf2018brain,pmlr-v235-huh24a,raghu2017svcca,pmlr-v97-kornblith19a}.
Let $M=L_A L_B$ be the number of layer pairs searched, where $L_A$ and $L_B$ are the depths of models $A$ and $B$.
Even under $H_0$, taking a maximum over $M$ comparisons inflates the reported score, a ``look-elsewhere'' effect.
This is an instance of the classical multiple comparisons problem~\citep{benjamini1995controlling,bonferroni1936teoria}: as $M$ increases, the probability that at least one null similarity exceeds any fixed threshold grows, inflating the expected maximum.
Consequently, when alignment is summarized via a max or top-$k$ statistic without correction, unrelated representations can exhibit spuriously high reported similarity, as the inflation depends on model depth, making raw summaries non-comparable across architectures.

\looseness=-1
Characterizing this inflation does not require independence across pairs.
It follows from a uniform right-tail bound.
Assume there exist a common mean $\mu\in\R$ and $\sigma>0$ such that the null fluctuations satisfy, for all $(\ell,\ell')$ and all $t\ge 0$,
\begin{equation}
    \Prob(S_{\ell,\ell'}-\mu \ge t)\le \exp\!\left(-\frac{t^2}{2\sigma^2}\right).
    \label{eq:subg_tail_main}
\end{equation}
For bounded similarities $S_{\ell,\ell'}\in[s_{\min},s_{\max}]$, Hoeffding's inequality implies a sub-Gaussian right-tail bound of the form \Cref{eq:subg_tail_main} with $\sigma\le (s_{\max}-s_{\min})/2$.
This covers many common bounded metrics (\textit{e.g.}, \gls{cka}/\gls{rsa}/\gls{mknn}).
Crucially, only the right tail is needed for bounding the maximum.
Then a union bound gives 
\begin{equation}
    \Prob(T_{\max}-\mu \ge t)\le M\exp\!\left(-\frac{t^2}{2\sigma^2}\right),
\end{equation}
and consequently for a constant $C$ 
\begin{equation}
    \E_{H_0}\left[T_{\max}\right] \;\le\; \mu + C\,\sigma \sqrt{\log M}.
    \label{eq:evt_scaling}
\end{equation}
\noindent\emph{Proof.} See \Cref{app:theory_evt}.

This creates a depth confounder. 
Deeper models (larger $M = L_A L_B$) can attain higher raw ``max-alignment'' scores purely because of a larger search space.
Correlations across neighboring layers reduce the \emph{effective} number of comparisons, but the inflation remains monotone in the search space size in typical workflows.
Therefore, raw scaling plots of $T_{\max}$ (or top-$k$ summaries) are not comparable across architectures unless the \emph{selection step itself} is calibrated.

The two confounders are not independent but hierarchical: the depth confounder operates on top of the width confounder. 
The width confounder gives each individual layer pair a positive chance similarity, and deeper models supply a larger pool of such chance similarities, so the selection maximum drifts higher. 

%% file: sections/Method.tex
To overcome the issues of the width and depth confounders, we introduce the \textit{null-calibration for representational similarity.}
The key idea is to compare observed similarity scores against an empirical null distribution obtained by permuting sample correspondences, thereby establishing a principled zero point that accounts for finite-sample, high-dimensional artifacts.

\subsection{Null-calibrated similarity}
\label{sec:scalar-calibration}

We propose \textit{null-calibrated} similarity measures to correct for width and depth confounders by transforming raw similarity scores into an effect size with a principled zero point.

Given representations $\mat{X}\in\R^{n\times d_x}$ and $\mat{Y}\in\R^{n\times d_y}$ aligned by rows, we operationalize the null hypothesis $H_0$ (no relationship beyond marginal statistics) by permuting sample correspondences. 
For permutations $\pi_k\in\Pi_n$ drawn i.i.d.\ uniformly from $\Pi_n$ and independently of $(\mat{X},\mat{Y})$, we form null scores
\begin{equation}
    s^{(k)} = s(\mat{X}, \pi_k(\mat{Y})), \qquad k=1,\dots,K.
\end{equation}
Let $s_{\mathrm{obs}} := s(\mat{X},\mat{Y})$ denote the observed score.
Let $s_{(1)} \le s_{(2)} \le \cdots \le s_{(K+1)}$ denote the order statistics of the \emph{combined} multiset $\{s_{\mathrm{obs}}, s^{(1)},\dots,s^{(K)}\}$ (with ties allowed).
We define a right-tail rank-based critical value:
\begin{equation}
    \tau_\alpha \;:=\; s_{(\lceil (1-\alpha)(K+1)\rceil)},
    \label{eq:get_tau_alpha}
\end{equation}
where $\lceil (1-\alpha)(K+1)\rceil$ is the $(1-\alpha)$-quantile of the $(K+1)$-sized multiset and the empirical right-tail $p$-value:
\begin{equation}
    p \;=\; \frac{1 + \#\{k\in\{1,\dots,K\} : s^{(k)} \ge s_{\mathrm{obs}}\}}{K+1}.
    \label{eq:pvalue_scalar}
\end{equation}

The critical value $\tau_\alpha$ defines a robust zero point: values below $\tau_\alpha$ are typical under $H_0$ at level $\alpha$, while $p$ provides an evidence measure that can be combined with multiple-testing correction when many comparisons are performed.


The proposed calibration framework relies on \emph{randomization} (permutation) to construct a null distribution for any similarity statistic.
This yields finite-sample guarantees under an exchangeability condition (\Cref{assm:exchangeability}), and it implies useful invariances that make calibrated scores comparable across metrics and implementations.

The permutation $p$-value in \Cref{eq:pvalue_scalar} is \emph{super-uniform} under $H_0$ (\textit{i.e.}, $\Prob_{H_0}(p \le \alpha) \le \alpha$ for all $\alpha \in [0,1]$), a standard consequence of randomization inference \citep{nichols2002permutation,phipson2010pvalues,good2005permutation} (see \Cref{app:perm_validity} for formal definitions and proofs).

\begin{corollary}[Type-I control for calibrated scores]
\label{cor:gating_type1}
Let $s_{\mathrm{obs}}=s(\mat{X},\mat{Y})$ and $s^{(k)}=s(\mat{X},\pi_k(\mat{Y}))$ for $k=1,\dots,K$.
Define the add-one permutation $p$-value $p$ as in \Cref{eq:pvalue_scalar}, and equivalently define the rank-based critical value
$\tau_\alpha := s_{(\lceil (1-\alpha)(K+1)\rceil)}$ from the sorted combined set $\{s_{\mathrm{obs}}, s^{(1)},\dots,s^{(K)}\}$.
Under \Cref{assm:exchangeability},
\begin{equation}
    \Prob_{H_0}\big(p \le \alpha\big) \le \alpha
    \quad\text{and hence}\quad
    \Prob_{H_0}\big(s_{\mathrm{obs}} > \tau_\alpha\big) \le \alpha,
\end{equation}
so the gating rule ``$s_{\mathrm{cal}}>0$'' (where $s_{\mathrm{cal}}$ is the calibrated score defined in \Cref{eq:scalar_gate_rescale}, which implies $p\le \alpha$) is a finite-sample $\alpha$-level declaration of similarity above chance.
\end{corollary}
\noindent\emph{Proof.} Follows directly from \Cref{lem:perm_super_uniform}; see \Cref{app:perm_validity}.

\paragraph{Calibrated score (scalar case).}
While $p$-values and null percentiles are rank-based and therefore invariant under monotone transformations of the raw score (\Cref{prop:monotone_invariance}; see \Cref{app:monotone}), effect sizes serve a complementary purpose: they quantify \emph{how much} similarity exceeds chance on an interpretable scale.
The calibrated score achieves this by rescaling the excess over the null threshold $\tau_\alpha$ to the interval $[0,1]$.
This rescaling is not monotone-invariant, by design.
A purely rank-based calibration would be equivalent to a score shift and would be unable to correct for the scale-dependent null baselines identified in \Cref{sec:theory}.
The calibrated score instead adapts to the actual null distribution, providing a meaningful zero point.

For similarity metrics with a known maximum (upper bound) $s_{\max}$ (often $s_{\max}=1$), we define a max-preserving calibrated score
\begin{equation}
    s_{\mathrm{cal}} =
    \max\!\left(
    \frac{s_{\mathrm{obs}} - \tau_\alpha}{s_{\max} - \tau_\alpha},
    0\right).
    \label{eq:scalar_gate_rescale}
\end{equation}
This calibrated score depends on the chosen level $\alpha$ through $\tau_\alpha$ (\Cref{eq:get_tau_alpha}).
We therefore also report the corresponding permutation $p$-value and/or null percentile for an $\alpha$-free summary.
This score satisfies $s_{\mathrm{cal}}=0$ whenever $s_{\mathrm{obs}} \le \tau_\alpha$ (\textit{i.e.}, below the estimated right-tail critical value of the permutation null), and $s_{\mathrm{cal}}=1$ when $s_{\mathrm{obs}}=s_{\max}$ (\textit{i.e.}, perfect similarity remains $1$). 
When $s_{\max}$ is unknown, or the metric is unbounded, we default to the unnormalized effect size $[s-\tau_\alpha]_+ = \max(s-\tau_\alpha,0)$.

\subsection{Aggregation-aware null-calibration}
\label{sec:selection_bias}

To analyze the similarity between two models $A$ and $B$ with depths $L_A$ and $L_B$, a common approach is to compute a layer-by-layer similarity matrix $\mat{S} \in \R^{L_A \times L_B}$ by evaluating a similarity score for every pair of layers:
\begin{equation}
S_{\ell,\ell'} = s\!\left(\mat{X}^{(A)}_\ell, \mat{Y}^{(B)}_{\ell'}\right),
\end{equation}
\looseness=-1
where $\mat{X}^{(A)}_\ell \in \R^{n \times d_\ell}$ and $\mat{Y}^{(B)}_{\ell'} \in \R^{n \times d_{\ell'}}$ are the representations of models $A$ and $B$ at layers $\ell$ and $\ell'$ respectively, evaluated on $n$ samples, and $s(\cdot,\cdot)$ is a similarity metric.
A common practice is then to summarize $\mat{S}$ by a \emph{selection-based} aggregation operator, such as taking the maximum.
These summaries are attractive because they support statements such as ``there exists a layer in $A$ that matches some layer in $B$'' or ``each layer of $A$ best matches a layer in $B$''.
However, selection introduces a statistical effect: even under the null hypothesis of no relationship between representations, selection-based summaries are systematically inflated.

As analyzed in \Cref{sec:theory_evt}, this inflation grows with the number of layer pairs and makes na\"ive post-selection $p$-values anti-conservative.
Our aggregation-aware calibration addresses this by calibrating the \emph{reported} statistic directly: the null distribution must match the \emph{entire} analysis pipeline. 
Let the aggregate score be $T(\mat{S})$ (\textit{e.g.}, a maximum), then the appropriate null is the distribution of $T(\mat{S})$ under a valid null transformation (\textit{e.g.}, permuting sample correspondences).
We therefore define an aggregation-aware permutation null.
When $T$ is the maximum over layer pairs, this recovers the classical ``maxT'' procedure of \citet{westfall1993resampling}. 
Our formulation generalizes it to arbitrary selection-based aggregates (\textit{e.g.}, top-$k$) and to any similarity metric.

\paragraph{Consistency of permutations across layers.}
For each draw $\pi_k\in\Pi_n$, we apply the \emph{same} sample permutation to all layers of model $B$ and define
\begin{gather}
S^{(k)}_{\ell,\ell'} := s\!\left(\mat{X}^{(A)}_\ell, \pi_k\!\left(\mat{Y}^{(B)}_{\ell'}\right)\right),
\\
\ell=1,\dots,L_A,\quad \ell'=1,\dots,L_B, \nonumber
\end{gather}
then compute $T^{(k)} := T(\mat{S}^{(k)})$.
Let $T_{\mathrm{obs}} := T(\mat{S})$ denote the observed aggregate.
Let $T_{(1)}\le\cdots\le T_{(K+1)}$ denote the order statistics of the combined set $\{T_{\mathrm{obs}},T^{(1)},\dots,T^{(K)}\}$ (with ties allowed). 
We define
\begin{equation}
    \tau_\alpha^{\mathrm{agg}} := T_{(\lceil (1-\alpha)(K+1)\rceil)},
\end{equation}
where $\lceil (1-\alpha)(K+1)\rceil$ is the $(1-\alpha)$-quantile of the $(K+1)$-sized multiset.
We report the right-tail permutation $p$-value
\begin{equation}
p_{\mathrm{agg}} = \frac{1 + \#\{k \in \{1, \ldots, K\} : T^{(k)} \ge T_{\mathrm{obs}}\}}{K+1},
\label{eq:pvalue_agg}
\end{equation}
By the same exchangeability argument as for scalar calibration, $p_{\mathrm{agg}}$ is super-uniform under $H_0$ (see \Cref{prop:agg_validity}).

\paragraph{Calibrated score (aggregate case).}
For bounded similarities with maximum $s_{\max}$ (often $s_{\max}=1$), we report a max-preserving calibrated aggregate
\begin{equation}
T_{\mathrm{cal}} =
\max\!\left(
\frac{T_{\mathrm{obs}} - \tau_\alpha^{\mathrm{agg}}}{s_{\max} - \tau_\alpha^{\mathrm{agg}}},
\, 0
\right).
\label{eq:max_preserving_gate}
\end{equation}
This score satisfies $T_{\mathrm{cal}}=0$ when $T_{\mathrm{obs}} \le \tau_\alpha^{\mathrm{agg}}$ and $T_{\mathrm{cal}}=1$ when $T_{\mathrm{obs}}=s_{\max}$.
As above, $T_{\mathrm{cal}}$ depends on $\alpha$ via $\tau_\alpha^{\mathrm{agg}}$; we therefore report both $T_{\mathrm{cal}}$ (magnitude above null) and $p_{\mathrm{agg}}$ (evidence against null), applying multiplicity correction \citep{holm1979simple,benjamini1995controlling} when many model pairs are evaluated.


\subsection{Summary}
\label{sec:method_summary}

To compute a calibrated similarity score:
\textit{(i)}~fix a significance level $\alpha$ (\textit{e.g.}, $\alpha=0.05$);
\textit{(ii)}~generate $K$ null scores by permuting sample correspondences;
\textit{(iii)}~compute critical value $\tau$ as the $\lceil (1-\alpha)(K+1) \rceil$-th order statistic of the combined set (observed + null scores);
\textit{(iv)}~return calibrated score, either $s_{\mathrm{cal}}$ or $T_{\mathrm{cal}}$.

Use \textit{scalar calibration} (\Cref{sec:scalar-calibration}) when comparing a single pair of representations.
Use \textit{aggregation-aware calibration} (\Cref{sec:selection_bias}) when reporting a summary statistic (\textit{e.g.}, maximum) over multiple layer pairs.
\Cref{app:implementation} provides pseudocode for both procedures.

%% file: sections/Experiment.tex
\begin{figure}[b]
  \vspace{-5mm}
  \centering
  \includegraphics[width=\linewidth]{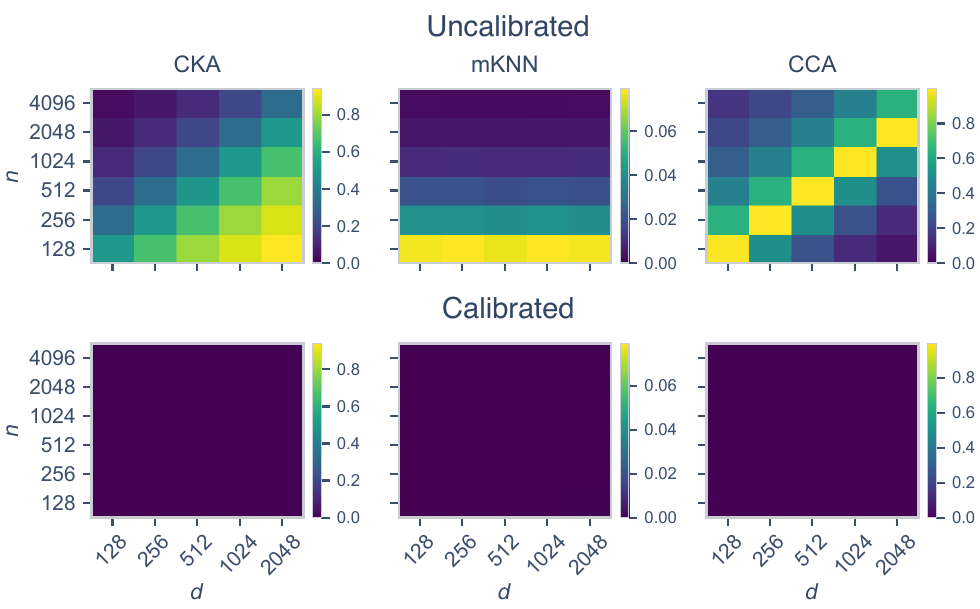}
  \caption{
  \textbf{Calibration eliminates spurious similarity across metrics.}
  Raw scores (top) drift with $d/n$; calibrated scores (bottom) collapse to zero.
  Results for heavy-tailed distributions and additional metrics are in \Cref{app:null-drift-full}. 
  }
  \label{fig:null-drift-aggregate}
  \vspace{-5mm}
\end{figure}

\looseness=-1
We quantify the effects of the width and depth confounders in controlled synthetic experiments and show that our calibration framework effectively removes them. 
We then revisit the Platonic Representation Hypothesis using our calibration framework, assessing which convergence trends remain robust after controlling for these confounding factors.

\subsection{Null-calibration removes width confounder}
\label{sec:exp-calibration-works}

We validate that our calibration eliminates width-related inflation of similarity across metrics, regimes, and noise distributions, without metric-specific derivations.

We design controlled synthetic experiments as follows.
Under $H_0$, we draw $\mat{X}, \mat{Y} \in \R^{n \times d}$ independently from Gaussian and heavy-tailed (Student-$t$, Laplace) distributions.
We vary the number of samples $n \in \{128, 256, 512, 1024, 2048, 4096\}$ and the dimension $d \in \{128, 256, 512, 1024, 2048\}$.
Under $H_1$, we inject a shared low-rank signal component and vary the signal-to-noise ratio.
We evaluate representative metrics spanning three families.
For spectral similarity, we use linear and RBF \gls{cka}, as well as \gls{cca}/SVCCA/PWCCA; for neighborhood similarity, we use \gls{mknn} (with $k=10$); and for geometric similarities, we use RSA and Procrustes. 
\Cref{fig:null-drift-aggregate} reports a subset of these metrics for readability; additional metrics are reported in \Cref{app:null-drift-full}.
For calibration, we use $K=200$ permutations with $\alpha=0.05$.

Under $H_0$, uncalibrated scores are systematically inflated, while our calibrated scores stay at zero across settings (\Cref{fig:null-drift-aggregate}).
This confirms that the similarity scores of wider models can arise purely from high-dimensional finite-sample effects, and our calibration removes this spurious baseline.
Importantly, the magnitude of the null baseline is metric-dependent, consistent with our theory: CKA's baseline scales as $\mathcal{O}(d/n)$ (\Cref{prop:nonzero_energy}), while \gls{mknn}'s baseline scales as $\mathcal{O}(k/n)$ (\Cref{prop:null_mknn}). 
Intuitively, \gls{mknn} compares local neighborhood overlap at a fixed $k$, thus only comparing relationships instead of local distances, making its null baseline insensitive to representation width $d$, which explains the order-of-magnitude gap observed in raw scores.
The same pattern holds for heavy-tailed noise (\Cref{app:null-drift-full}).

The same width inflation arises under a genuine shared signal, and calibration corrects it there as well, as we show analytically (\Cref{prop:width_h1}) and on real networks at fixed $n$ (\Cref{app:width-h1-real}).

Next, we verify the statistical guarantees of our empirical null calibration.
For Type-I error control, rejection rates stay at or below the nominal $\alpha=0.05$ across $(n, d/n)$ configurations (\Cref{fig:statistical-guarantees}a). 
Crucially, our calibration does not sacrifice sensitivity to real alignment: detection rates increase rapidly with signal strength (\Cref{fig:statistical-guarantees}b).
Overall, our calibration preserves signal structure: in the high-signal regime, raw and calibrated scores show the same pattern, while in the low-signal regime, calibration correctly gates scores to zero (\Cref{app:snr-heatmaps}).

Furthermore, we verify that our empirical calibration closely matches existing analytical bias corrections for \gls{cka}~\citep{murphy2024biasedcka}, recovering the width correction without metric-specific derivation (\Cref{app:analytical-comparison}).

Additionally, we perform ablations on different noise distributions used in the synthetic experiments (\Cref{app:phase-diagrams}), different calibration approaches (\Cref{app:per-metric}), and an ablation on the influence of the number of permutations $K$ used for calibration (\Cref{app:perm-budget}).  \vspace{-2mm}

\begin{figure}[h]
  \centering
  \includegraphics[width=\linewidth]{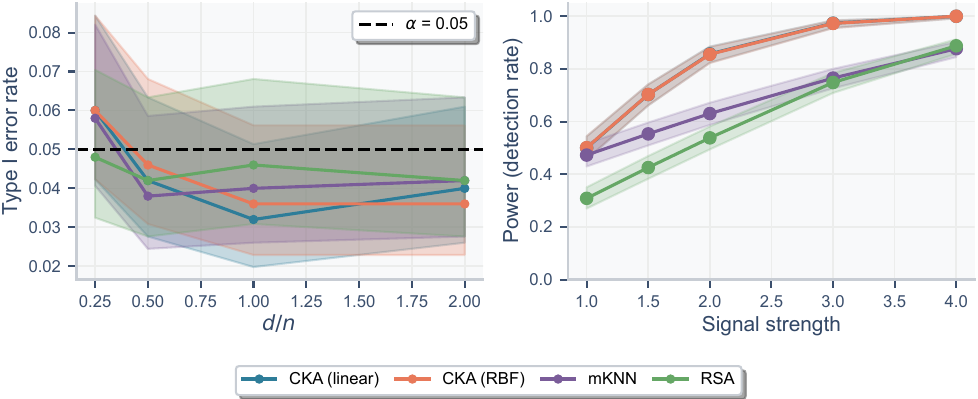}
  \caption{
  \textbf{Statistical guarantees.}
  (Left)~Type-I error stays at or below $\alpha$ across configurations.
  (Right)~Power increases rapidly with signal strength; calibration does not sacrifice sensitivity.
  }
  \label{fig:statistical-guarantees}
  \vspace{-4mm}
\end{figure}

\begin{figure*}[t]
  \centering
\begin{subfigure}[t]{0.49\linewidth}
    \includegraphics[width=\linewidth]{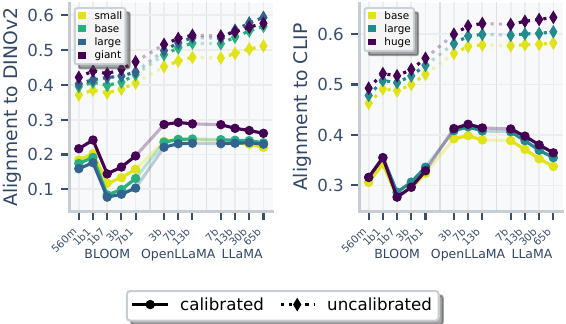}
    \caption{\Gls{cka} RBF: Global spectral alignment.}
  \end{subfigure}
  \hfill
  \begin{subfigure}[t]{0.49\linewidth}
    \includegraphics[width=\linewidth]{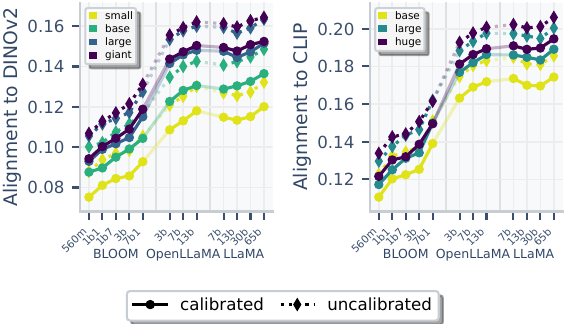}
    \caption{\gls{mknn}: Local neighborhood overlap.}
  \end{subfigure}
  \caption{
  \textbf{Revisiting the Platonic Representation Hypothesis.}
  Models are ranked according to their language performance \citep{pmlr-v235-huh24a}.
  Solid lines connect the models within the same family, while semi-transparent lines connect the models across different families.
  (a)~Global spectral metrics lose their convergence trend; calibrated scores show no systematic increase with scale.
  (b)~Local neighborhood metrics keep their trend even after calibration.
  Full results for all vision families and metrics in \Cref{app:prh-full}.
  }
  \label{fig:prh-alignment}
  \vspace{-4mm}
\end{figure*}

\subsection{Null-calibration removes depth confounder}
\label{sec:exp-aggregation}

We validate that our aggregation-aware null-calibration eliminates the depth confounder.
To build a controlled synthetic setting, we construct two synthetic models, $A$ and $B$, each with $L$ layers. 
Under $H_0$, we sample layer representations $\{\mat{X}_\ell\}_{\ell=1}^L$ and $\{\mat{Y}_{\ell'}\}_{\ell'=1}^L$, where each $\mat{X}_\ell, \mat{Y}_{\ell'} \in \R^{n\times d}$ has i.i.d.\ $\mathcal{N}(0,1)$ entries (independent across layers and between models), using $d/n=8$ to match the upper range of the \acrlong{prh} setting.
We then compute the layerwise similarity matrix $S_{\ell,\ell'} = \operatorname{CKA}(\mat{X}_\ell, \mat{Y}_{\ell'})$ and summarize it with standard aggregates. 

The uncalibrated max-aggregated scores inflate with layer count even under $H_0$ (\Cref{fig:agg-calibration}): raw max-scores are systematically higher at $L=128$ than at $L=2$, despite no true signal.
Our aggregation-aware calibration eliminates this bias: calibrated aggregates remain stable regardless of depth.
We further show that naively calibrating each scalar comparison still leads to inflation, highlighting the importance of calibrating the final statistic.
Furthermore, since deeper models tend to be wider as well, raw comparisons are doubly confounded.  \vspace{-2mm}

\begin{figure}[H]
  \centering
  \includegraphics[width=\linewidth]{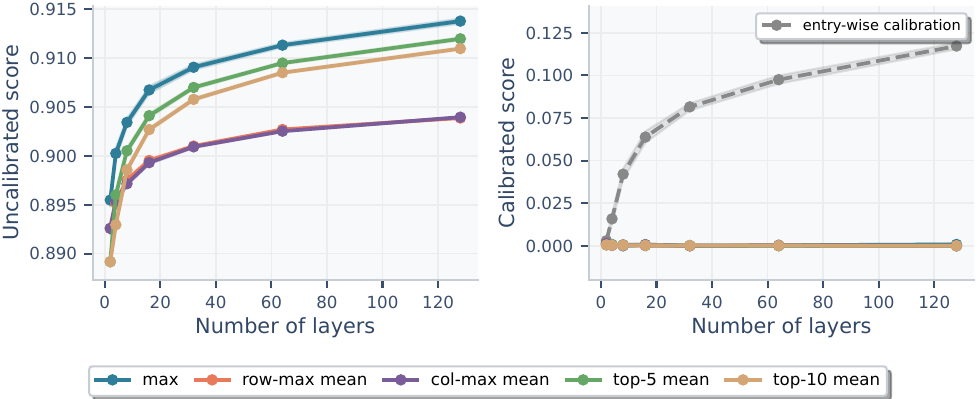}
  \caption{
  \textbf{Aggregation-aware calibration removes depth confounding.}
  Raw max-aggregates of linear \gls{cka} scores inflate with layer count under the null; calibrated aggregates are stable and show that naive entry-wise calibration still leads to inflation.
  }
  \vspace{-2mm}
  \label{fig:agg-calibration}
\end{figure}

\subsection{Revisiting the Platonic Representation Hypothesis}
\label{sec:prh}

A central claim behind the Platonic Representation Hypothesis is that, as models become more capable, their representations begin to \emph{converge} across modalities.
We revisit this claim through our calibration framework to determine whether the observed alignment reflects genuine shared representation structure or instead arises from width and depth confounders.

\looseness=-1
We follow the experimental protocol of \citet{pmlr-v235-huh24a} using $n = 1024$ image--text pairs (WIT; \citet{srinivasan2021wit}) and embeddings from three language model families (BLOOM, OpenLLaMA, LLaMA) and five vision model families (ImageNet-21K, MAE, DINOv2, CLIP, CLIP-finetuned) across multiple scales.
As in \citet{pmlr-v235-huh24a}, we extract per-layer representations using the class token for vision models and mean pooling over tokens for language models.
This yields 204 vision--language model pairs spanning $d/n \in [0.19, 8]$.
For each pair, we compute layer-wise similarity and report the maximum across layers, as in the original work.
We evaluate both global spectral metrics (\gls{cka} linear/RBF) and local neighborhood metrics (\gls{mknn}, cycle-$k$NN, CKNNA). 
Following \citet{pmlr-v235-huh24a}, we evaluate \gls{mknn}, cycle-$k$NN, and CKNNA with $k=10$.
We further apply Benjamini-Hochberg FDR correction~\citep{benjamini1995controlling} to control for multiple comparisons across model pairs.

For the \emph{global} similarity, we find that uncalibrated \gls{cka} scores increase with model scale (dotted lines in \Cref{fig:prh-alignment}a), reproducing the trend interpreted as evidence of cross-modal convergence~\citep{pmlr-v235-huh24a}. 
However, this trend disappears after our calibration (solid lines): calibrated \gls{cka} shows no systematic increase with model size.
This indicates that global convergence in uncalibrated \gls{cka} is largely attributable to width and depth confounders rather than a genuine increase in representational similarity.
The same holds for the other global metrics, including shape- and geometry-based ones: \gls{svcca}, the RV coefficient, and Procrustes distance all lose their apparent scaling trend after calibration (\Cref{app:prh-full}).

In contrast, for the \emph{local} similarity, evidence of cross-modal convergence remains strong for neighborhood-based metrics even under our calibration (\Cref{fig:prh-alignment}b).
The same qualitative conclusion holds for other neighborhood-based measures (cycle-$k$NN and CKNNA; \Cref{app:prh-full}) and different choices of $\alpha$ (\Cref{app:alpha-sensitivity}).
Further analysis (\Cref{app:locality-analysis}) reveals that models converge in local neighborhood structure: models increasingly agree on which points are neighbors, but do not agree on the pairwise distances, since CKA-RBF with a small bandwidth shows no alignment after calibration.

\gls{cka} and \gls{mknn} capture different invariances, so their raw magnitudes are not directly comparable. We therefore compare, for each metric separately, how strongly its score tracks language-model capability before and after calibration (its Pearson correlation with the model ranking). After calibration this correlation collapses for global metrics (for linear \gls{cka} it falls from $0.86$ to $0.45$, and for Procrustes distance from $0.89$ to $0.39$) but is essentially unchanged for local ones (\gls{mknn} stays near $0.85$ and CKNNA near $0.87$).
\Cref{app:prh-full} reports all metrics.

To test whether these findings generalize beyond images and text, we extend our analysis to video--language alignment following \citet{zhu2025dynamic}.
We compare video encoders (VideoMAE small/base/large/huge, fine-tuned on Kinetics) against the same language model families.
Consistent with our previous findings, the global similarity (\gls{cka}) shows no trend with model capacity (\Cref{fig:v2t}). 
In contrast, for local similarity (\gls{mknn}), a clear scaling trend emerges with VideoMAE-Large/Huge, whereas smaller video encoders appear to act as a bottleneck, limiting alignment regardless of language model size.
This confirms that local neighborhood convergence extends to video--language alignment, provided that representations are sufficiently powerful.
\Cref{app:video-language} further compares non-finetuned VideoMAE versions and a variety of image models at the frame level on the same dataset, showing the same trend.

\begin{figure}[H]
  \centering
  \includegraphics[width=\linewidth]{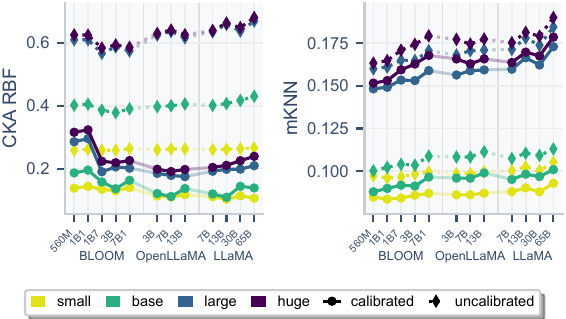}
  \vspace{-1.5em}
  \caption{
  \textbf{Video--language alignment.}
  Extending the \acrlong{prh} analysis to video encoders (VideoMAE small/base/large/huge) yields the same pattern: calibrated \gls{cka} drops substantially while \gls{mknn} retains alignment.
  }
  \label{fig:v2t}
  \vspace{-4mm}
\end{figure}

\looseness=-1
Taken together, these results suggest a refined version of the Platonic Representation Hypothesis.
After calibration, we find little evidence that representations converge in global spectral structure as models scale, at least under the considered setting.
What reliably persists is local geometric alignment: different models preserve similar neighborhood relationships among inputs.
We therefore propose the alternative \textbf{Aristotelian Representation Hypothesis}: \emph{As models become capable, their representations converge to shared local neighborhood relationships}.

%% file: sections/Conclusion.tex
Representational similarity metrics are widely used to study learned features, but their interpretation is systematically distorted by two artifacts: width-dependent null baselines and depth-dependent selection inflation. 
We introduced a unified null-calibration framework that corrects both, turning similarity scores into effect sizes with principled zero points and valid $p$-values. 
Applying our framework to the Platonic Representation Hypothesis reveals that previously reported global spectral convergence is largely confounded by width and depth, whereas local neighborhood alignment remains significant, motivating an Aristotelian Representation Hypothesis.

\paragraph{Relationship to the Platonic hypothesis.}
The Aristotelian hypothesis refines rather than refutes the Platonic one. Global convergence implies local convergence, because matching the full geometry preserves neighborhoods, but the converse does not hold.
The Platonic hypothesis therefore implies the Aristotelian one, which is a weaker criterion.
Our experiments support this local form and, after calibration, find little evidence for the global one. 
Prior reports of convergence thus remain compatible with our analysis, and the Aristotelian hypothesis offers a more precise lens on what converges across models and modalities.

\paragraph{Limitations and outlook.}
Our conclusions come with two caveats. 
First, representational similarity has no ground-truth scale, so we report the presence or absence of calibrated evidence for convergence rather than proving it. 
Second, our guarantees assume exchangeability, so grouped or clustered samples require restricted permutations that preserve their dependence structure. 
Why representations converge in local neighborhoods but not in global geometry remains the key open question.

%% file: sections/Appendix.tex
\section{Existing calibration approaches for representational similarity metrics}
\label{app:similarity-metrics}

\begin{table}[ht]
\centering
\caption{
Comparison of prior works.
\textbf{Y}=yes, \textbf{N}=no, \textbf{P}=partial/indirect.
``Debias'' indicates an explicit null correction of the reported similarity.
``Bounded'' indicates whether the corrected score preserves an interpretable upper bound (\textit{e.g.}, 1 for perfect alignment).
``Agg-aware'' indicates calibration of selection-based aggregates (\textit{e.g.}, max over layer pairs).
}
\label{tab:related_work_comparison}
\small
\setlength{\tabcolsep}{5pt}
\resizebox{0.7\linewidth}{!}{
\begin{tabular}{l l c c c}
\toprule
\textbf{Ref} & \textbf{Metric(s)} & \textbf{Debias?} & \textbf{Bounded?} & \textbf{Agg-aware?} \\
\midrule
\citet{murphy2024biasedcka} & \acrshort{cka} & Y & N & N \\
\citet{chun2025sparsecka} & \acrshort{cka} & Y & N & N \\
\citet{cui2022deconfounded} & \acrshort{rsa}/\acrshort{cka} & P & N & N \\
\citet{diedrichsen2020wuc} & \acrshort{rsa} (cv/\acrshort{wuc}) & Y & P & N \\
\citet{cai2019brsa} & \acrshort{rsa} (Bayes) & P & N & N \\
\citet{smilde2009rv} & RV / adj.\ RV & Y & N & N \\
\midrule
\textbf{Ours} & \textbf{Any bounded metric} & \textbf{Y} & \textbf{Y} & \textbf{Y} \\
\bottomrule
\end{tabular}
}
\end{table}

\section{Metrics and score definitions}
\label{app:metrics}

This appendix gives the definitions of the similarity metrics $s(\mat{X},\mat{Y})$ used throughout the paper.
The main text focuses on the calibration procedure (\Cref{sec:scalar-calibration,sec:selection_bias}).
Here we provide concrete instantiations of the metrics referenced in \Cref{sec:setup} and \Cref{sec:exp-setup}.

\subsection{Preprocessing and basic notation}
\label{app:metrics_preproc}

Let $\mat{X}\in\R^{n\times d_x}$ and $\mat{Y}\in\R^{n\times d_y}$ denote row-aligned representations evaluated on the same $n$ inputs.
We use the centering matrix
\begin{equation}
    \mat{H} \;=\; \mat{I}_n - \frac{1}{n}\mathbbm{1}_n\mathbbm{1}_n^\top,
    \label{eq:centering_matrix}
\end{equation}
where $\mat{I}_n\in\R^{n\times n}$ is the identity matrix and $\mathbbm{1}_n\in\R^n$ is the all-ones vector.
We define row-centered representations $\mat{X}_c=\mat{H}\mat{X}$ and $\mat{Y}_c=\mat{H}\mat{Y}$.
Unless stated otherwise, similarities are computed on centered representations.

\subsection{Raw similarity metrics}
\label{app:metrics_raw}

This section provides formal definitions of the similarity metrics used throughout the paper.
In the main text, we primarily use CKA (linear and RBF kernel)~\citep{pmlr-v97-kornblith19a}, RSA~\citep{kriegeskorte2008rsa}, and mutual $k$-NN~\citep{pmlr-v235-huh24a} as representative metrics from the spectral, geometric, and neighborhood families, respectively.
Additional metrics (SVCCA, PWCCA, cycle-$k$NN, CKNNA, RV coefficient, Procrustes) are included for completeness and used in supplementary experiments.

\subsubsection{Spectral metrics}

\paragraph{Linear \acrfull*{cka}.}
\label{app:metrics_cka_linear}

Linear \gls{cka}~\citep{pmlr-v97-kornblith19a} can be written as a normalized Frobenius energy of the \emph{sample cross-covariance} operator.
With $\mat{X}_c,\mat{Y}_c$ as above, define the sample (cross-)covariances
\begin{equation}
    \widetilde{\mat{\Sigma}}_{XX}:=\frac{1}{n-1}\mat{X}_c^\top\mat{X}_c,\qquad
    \widetilde{\mat{\Sigma}}_{YY}:=\frac{1}{n-1}\mat{Y}_c^\top\mat{Y}_c,\qquad
    \widetilde{\mat{C}}:=\widetilde{\mat{\Sigma}}_{XY}:=\frac{1}{n-1}\mat{X}_c^\top\mat{Y}_c.
\end{equation}
The biased linear \gls{hsic} energy equals $\|\widetilde{\mat{C}}\|_F^2$.
The commonly used linear \gls{cka} normalization can be written as
\begin{equation}
    \mathrm{CKA}_{\mathrm{lin}}(\mat{X},\mat{Y})
    \;=\;
    \frac{\|\widetilde{\mat{C}}\|_F^2}{\|\widetilde{\mat{\Sigma}}_{XX}\|_F\,\|\widetilde{\mat{\Sigma}}_{YY}\|_F}
    \;=\;
    \frac{\|\mat{X}_c^\top\mat{Y}_c\|_F^2}{\|\mat{X}_c^\top\mat{X}_c\|_F\,\|\mat{Y}_c^\top\mat{Y}_c\|_F}
    \;\in\;[0,1],
    \label{eq:cka_linear}
\end{equation}
where the second equality follows by cancellation of common $1/(n-1)$ factors.
\newline
\emph{What it measures:} linear \gls{cka} treats two representations as similar when pairs of points with high inner product in one space also have high inner product in the other, \textit{i.e.}, when their global second-order geometry agrees. It is invariant to rotations and isotropic rescaling, but not to general invertible linear maps.

\paragraph{Kernel \acrlong*{cka}.}
\label{app:metrics_cka_kernel}

Kernel \gls{cka}~\citep{pmlr-v97-kornblith19a} generalizes linear \gls{cka} by replacing dot products with kernel functions.
Let $k_X:\R^{d_x}\times\R^{d_x}\to\R$ and $k_Y:\R^{d_y}\times\R^{d_y}\to\R$ be positive semidefinite kernel functions (\textit{e.g.}, RBF kernel $k_X(\vect{x},\vect{x}')=\exp(-\|\vect{x}-\vect{x}'\|^2/2\sigma^2)$).
Let $\mat{K}_X\in\R^{n\times n}$ and $\mat{K}_Y\in\R^{n\times n}$ be Gram matrices with entries $(\mat{K}_X)_{ij}=k_X(\vect{x}_i,\vect{x}_j)$ and $(\mat{K}_Y)_{ij}=k_Y(\vect{y}_i,\vect{y}_j)$.
Let $\widetilde{\mat{K}}_X=\mat{H}\mat{K}_X\mat{H}$ and $\widetilde{\mat{K}}_Y=\mat{H}\mat{K}_Y\mat{H}$ denote centered Gram matrices.
Kernel \gls{cka} is defined as:
\begin{equation}
    \mathrm{CKA}_{k_X,k_Y}(\mat{X},\mat{Y})
    \;=\;
    \frac{\langle \widetilde{\mat{K}}_X,\widetilde{\mat{K}}_Y\rangle_F}{\|\widetilde{\mat{K}}_X\|_F\,\|\widetilde{\mat{K}}_Y\|_F}.
    \label{eq:cka_kernel}
\end{equation}
where $\langle A,B\rangle_F=\mathrm{tr}(A^\top B)$.
With positive semidefinite kernels and the biased \gls{hsic} estimator, the numerator is nonnegative, and kernel \gls{cka} typically lies in $[0,1]$.
\newline
\emph{What it measures:} kernel \gls{cka} treats representations as similar when their kernel-induced similarity patterns over points agree. With an RBF kernel of small bandwidth, this emphasizes local, nonlinear similarity structure rather than global linear geometry.

\paragraph{Unbiased \acrlong*{cka}.}
\label{app:metrics_cka_unbiased}

The biased \gls{hsic} estimator can yield inflated similarity scores at finite sample sizes.
\citet{song2012hsic} derived an unbiased \gls{hsic} estimator by recognizing that \gls{hsic} can be formulated as a U-statistic.
Following \citet{pmlr-v97-kornblith19a}, we substitute the unbiased estimator into the \gls{cka} formula.
Let $\mathring{\mat{K}}_X$ denote the Gram matrix $\mat{K}_X$ with its diagonal entries set to zero (and likewise $\mathring{\mat{K}}_Y$).
The unbiased \gls{hsic} estimator is:
\begin{equation}
    \mathrm{HSIC}_u(\mat{K}_X, \mat{K}_Y) = \frac{1}{n(n-3)}\left( \mathrm{tr}(\mathring{\mat{K}}_X \mathring{\mat{K}}_Y) + \frac{\mathbf{1}^\top \mathring{\mat{K}}_X \mathbf{1} \cdot \mathbf{1}^\top \mathring{\mat{K}}_Y \mathbf{1}}{(n-1)(n-2)} - \frac{2}{n-2} \mathbf{1}^\top \mathring{\mat{K}}_X \mathring{\mat{K}}_Y \mathbf{1} \right).
    \label{eq:hsic_unbiased}
\end{equation}
Unbiased \gls{cka} replaces both numerator and denominator of \Cref{eq:cka_kernel} with this estimator.
Unlike the biased version, unbiased \gls{cka} can take small negative values at finite $n$.
\newline
\emph{What it measures:} the same agreement of similarity patterns as kernel \gls{cka}, but with the finite-sample inflation of \gls{hsic} removed, so a value near zero indicates no more agreement than expected by chance.

\paragraph{\Acrfull{cca}-based similarity.}
\label{app:metrics_cca}

\Gls{cca}~\citep{weenink2003canonical} measures linear subspace alignment.
The sample canonical correlations $\{\rho_i\}_{i=1}^r$ (with $r=\mathrm{rank}(\widetilde{\mat{\Sigma}}_{XY})$) are the singular values of the whitened cross-covariance operator
\begin{equation}
    \widetilde{\mat{T}}_{\mathrm{CCA}}
    \;=\;
    \widetilde{\mat{\Sigma}}_{XX}^{-\tfrac{1}{2}}\,\widetilde{\mat{\Sigma}}_{XY}\,\widetilde{\mat{\Sigma}}_{YY}^{-\tfrac{1}{2}}.
    \label{eq:cca_operator}
\end{equation}
Common scalar summaries include the mean canonical correlation $\frac{1}{r}\sum_{i=1}^r \rho_i$ or a weighted average as used in \gls{svcca}~\citep{raghu2017svcca} and \gls{pwcca}~\citep{morcos2018insights}.
\newline
\emph{What it measures:} \gls{cca} treats two representations as similar when one can be mapped onto the other by an invertible linear transformation, \textit{i.e.}, when they span the same linear subspace. It ignores rotations, scaling, and other invertible linear changes within each space.

\paragraph{\Acrfull{svcca}.}
\label{app:metrics_svcca}

\Gls{svcca}~\citep{raghu2017svcca} combines dimensionality reduction via singular value decomposition (SVD) with \gls{cca}.
First, truncated SVD is applied to each representation to retain the top principal components, yielding $\mat{X}' \in \R^{n \times p}$ and $\mat{Y}' \in \R^{n \times q}$.
Then \gls{cca} is applied to the reduced representations, yielding canonical correlations $\{\rho_i\}_{i=1}^r$.
The \gls{svcca} similarity is the mean canonical correlation:
\begin{equation}
    \mathrm{SVCCA}(\mat{X}, \mat{Y}) = \frac{1}{r} \sum_{i=1}^{r} \rho_i.
    \label{eq:svcca}
\end{equation}
\emph{What it measures:} \gls{svcca} treats representations as similar when their high-variance subspaces are linearly aligned, discarding low-variance (noise) directions before measuring linear correspondence.

\paragraph{\Acrfull*{pwcca}.}
\label{app:metrics_pwcca}

\Gls{pwcca}~\citep{morcos2018insights} improves upon \gls{svcca} by weighting canonical correlations according to their importance in explaining the original representations.
Let $\vect{h}_i^X$ and $\vect{h}_i^Y$ denote the $i$-th canonical variables (projections onto canonical directions).
The weight for the $i$-th canonical correlation is the overall magnitude of the corresponding canonical variable across the dataset:
\begin{equation}
    \alpha_i = \sum_{m=1}^{n} \left| (\vect{h}_i^X)_m \right| = \lVert \vect{h}_i^X \rVert_1,
    \label{eq:pwcca_weights}
\end{equation}
the $\ell_1$ norm of the $i$-th canonical variable over the $n$ samples, following the reference implementation of \citet{morcos2018insights}.
The \gls{pwcca} similarity is the weighted mean:
\begin{equation}
    \mathrm{PWCCA}(\mat{X}, \mat{Y}) = \frac{\sum_{i=1}^{r} \alpha_i \rho_i}{\sum_{i=1}^{r} \alpha_i}.
    \label{eq:pwcca}
\end{equation}
This weighting ensures that canonical correlations corresponding to principal directions receive higher weight than those corresponding to noise dimensions.
\newline
\emph{What it measures:} like \gls{cca}, \gls{pwcca} treats representations as similar when they linearly correspond, but it counts agreement along high-variance directions more, so two representations are similar when their dominant directions align.

\paragraph{RV coefficient.}
\label{app:metrics_rv}

The RV (``Relation between two sets of Variables'') coefficient~\citep{robert1976unifying,smilde2009rv} is a multivariate generalization of the squared Pearson correlation.
It measures the similarity between two configuration matrices via their inner-product (Gram) matrices.
Let $\mat{W}_X = \mat{X}_c\mat{X}_c^\top$ and $\mat{W}_Y = \mat{Y}_c\mat{Y}_c^\top$ be the inner-product (Gram) matrices of the centered representations $\mat{X}_c, \mat{Y}_c$.
The RV coefficient is:
\begin{equation}
    \mathrm{RV}(\mat{X}, \mat{Y}) = \frac{\mathrm{tr}(\mat{W}_X \mat{W}_Y)}{\sqrt{\mathrm{tr}(\mat{W}_X^2) \ \mathrm{tr}(\mat{W}_Y^2)}} \;\in\; [0,1].
    \label{eq:rv_coefficient}
\end{equation}
\emph{What it measures:} the RV coefficient treats representations as similar when their point-by-point inner-product (Gram) matrices match. Like linear \gls{cka}, it captures global second-order geometry and is invariant to rotation but sensitive to scaling.
Computed on centered representations (our default convention), the RV coefficient coincides exactly with linear \gls{cka}; we list it separately as it arises from a different historical motivation and include it for completeness.

\subsubsection{Geometric metrics}

\paragraph{\Acrfull{rsa} via Spearman correlation of dissimilarity matrices.}
\label{app:metrics_rsa}

\Gls{rsa}~\citep{kriegeskorte2008rsa} compares the geometry induced by pairwise dissimilarities.
Let $\delta(\cdot,\cdot)$ be a dissimilarity on representation vectors (\textit{e.g.}, correlation distance $\delta(\vect{u},\vect{v})=1-\mathrm{corr}(\vect{u},\vect{v})$, cosine distance).
Define Representational Dissimilarity Matrices (RDMs)
\begin{equation}
    (\mat{D}_X)_{ij}=\delta(\vect{x}_i,\vect{x}_j),\qquad
    (\mat{D}_Y)_{ij}=\delta(\vect{y}_i,\vect{y}_j),
    \label{eq:rdm_def}
\end{equation}
and let $\mathrm{vec}_\triangle(\mat{D})\in\R^{n(n-1)/2}$ denote vectorization of the strict upper triangle.
\Gls{rsa} is then computed as a rank correlation between the two RDM vectors:
\begin{equation}
    \mathrm{RSA}(\mat{X},\mat{Y})
    \;=\;
    \rho_S\!\left(\mathrm{vec}_\triangle(\mat{D}_X),\;\mathrm{vec}_\triangle(\mat{D}_Y)\right),
    \label{eq:rsa_spearman}
\end{equation}
where Spearman's $\rho$ can be expressed as Pearson correlation of ranks,
\begin{equation}
    \rho_S(\vect{u},\vect{v})
    \;=\;
    \mathrm{corr}\!\left(\mathrm{rank}(\vect{u}),\;\mathrm{rank}(\vect{v})\right).
    \label{eq:spearman_rho}
\end{equation}
\emph{What it measures:} \gls{rsa} treats two representations as similar when they rank pairs of points in the same order of (dis)similarity. It compares the relational geometry of distances rather than the coordinates, and depends only on the pattern of pairwise dissimilarities.

\paragraph{Procrustes distance.}
\label{app:metrics_procrustes}

The orthogonal Procrustes distance~\citep{williams2021shapemetrics} measures the minimal Euclidean distance between two representations after optimal orthogonal alignment.
Assuming $d_x = d_y = d$, the optimal orthogonal matrix $\mat{Q}^* \in \mathcal{O}(d)$ is:
\begin{equation}
    \mat{Q}^* = \operatorname*{argmin}_{\mat{Q} \in \mathcal{O}(d)} \|\mat{X} - \mat{Y}\mat{Q}\|_F^2,
    \label{eq:procrustes_opt}
\end{equation}
which has the closed-form solution $\mat{Q}^* = \mat{V}\mat{U}^\top$ where $\mat{U}\mat{\Sigma}\mat{V}^\top = \mat{X}^\top \mat{Y}$ is the SVD.
The Procrustes distance is:
\begin{equation}
    d_{\mathrm{Proc}}(\mat{X}, \mat{Y}) = \|\mat{X} - \mat{Y}\mat{Q}^*\|_F.
    \label{eq:procrustes_distance}
\end{equation}
We convert this distance to a similarity score $s_{\mathrm{Proc}} = 1 - d_{\mathrm{Proc}} / \|\mat{Y}_c\|_F$, evaluated on centered representations, where $\mat{Y}_c$ denotes the centered $\mat{Y}$. It equals $1$ at perfect orthogonal alignment ($d_{\mathrm{Proc}}=0$) and can be negative for poorly aligned pairs.
\newline
\emph{What it measures:} Procrustes treats two representations as similar when one can be rotated and reflected onto the other with small residual. It compares absolute geometry up to rigid motion and, unlike neighborhood metrics, is sensitive to the actual pairwise distances.

\subsubsection{Neighborhood metrics}

\paragraph{\Acrfull{mknn}.}
\label{app:metrics_mknn}

\gls{mknn}~\citep{pmlr-v235-huh24a} focuses on local topology.
For each anchor sample $i$, define the set of its $k$ nearest neighbors according to a distance measure $\operatorname{dist}(\cdot, \cdot)$ in $\mat{X}$ and $\mat{Y}$,
\begin{equation}
    N_{\mat{X}}(i) = \operatorname{KNN}_k(i;\mat{X}),
    \qquad
    N_{\mat{Y}}(i) = \operatorname{KNN}_k(i;\mat{Y}),
    \label{eq:knn_sets}
\end{equation}
where $\operatorname{KNN}_k(i;\mat{X})$ denotes the indices of the $k$ samples (excluding $i$) that minimize $\mathrm{dist}(\vect{x}_i,\vect{x}_j)$.
\gls{mknn} is then defined as the average fraction of shared neighbors:
\begin{equation}
    \mathrm{mKNN}_k(\mat{X},\mat{Y})
    \;=\;
    \frac{1}{n}\sum_{i=1}^n \frac{|N_{\mat{X}}(i)\cap N_{\mat{Y}}(i)|}{k}
    \;\in\;[0,1].
    \label{eq:mknn_overlap}
\end{equation}
\emph{What it measures:} \gls{mknn} treats two representations as similar when each point keeps the same set of nearest neighbors in both spaces. It captures local topology, depends only on the rank order of distances, and is therefore invariant to any transformation that preserves $k$-nearest-neighbor sets, ignoring exact distances and global geometry.

\paragraph{Cycle-$k$NN (bidirectional $k$-NN).}
\label{app:metrics_cycle_knn}

While \gls{mknn} measures one-directional neighborhood overlap, cycle-$k$NN enforces a round-trip consistency between the two spaces~\citep{pmlr-v235-huh24a}.
An anchor $i$ counts as consistent if at least one of its nearest neighbors $j$ in $\mat{Y}$ has $i$ among its nearest neighbors in $\mat{X}$:
\begin{equation}
    \mathrm{cycle\text{-}kNN}_k(\mat{X},\mat{Y})
    \;=\;
    \frac{1}{n}\sum_{i=1}^n \mathbbm{1}\!\left[\, i \in \bigcup_{j \in N_{\mat{Y}}(i)} N_{\mat{X}}(j) \,\right]
    \;\in\;[0,1].
    \label{eq:cycle_knn}
\end{equation}
This is stricter than \gls{mknn}, as it requires neighbor relations to be mutually recognized across the two spaces.

\emph{What it measures:} cycle-$k$NN treats points as similar only when neighborhood membership is reciprocated across spaces, a stricter local-topology criterion than \gls{mknn}.

\paragraph{CKA with Neighborhood Alignment (CKNNA).}
\label{app:metrics_cknna}

CKNNA~\citep{pmlr-v235-huh24a} combines the kernel formulation of \gls{cka} with local neighborhood structure, restricting the kernel interaction to mutual $k$-nearest-neighbor edges.
Let $\mat{K}_X=\mat{X}\mat{X}^\top$ and $\mat{K}_Y=\mat{Y}\mat{Y}^\top$ be the (linear) Gram matrices, and let $\mat{M}_X,\mat{M}_Y\in\{0,1\}^{n\times n}$ be the $k$-NN indicator matrices with $(\mat{M}_X)_{ij}=\mathbbm{1}[\,j\in N_{\mat{X}}(i)\,]$ and $(\mat{M}_Y)_{ij}=\mathbbm{1}[\,j\in N_{\mat{Y}}(i)\,]$.
Let $\mat{A}=\mat{M}_X\odot\mat{M}_Y$ retain only edges that are $k$-NN in \emph{both} spaces, where $\odot$ is the Hadamard product.
CKNNA aligns the neighbor-masked Gram matrices with \gls{cka}'s normalization:
\begin{equation}
    \mathrm{CKNNA}_k(\mat{X},\mat{Y})
    \;=\;
    \frac{\langle \mat{H}(\mat{A}\odot\mat{K}_X)\mat{H},\; \mat{H}(\mat{A}\odot\mat{K}_Y)\mat{H} \rangle_F}{\|\mat{H}(\mat{M}_X\odot\mat{K}_X)\mat{H}\|_F \, \|\mat{H}(\mat{M}_Y\odot\mat{K}_Y)\mat{H}\|_F}.
    \label{eq:cknna}
\end{equation}
\emph{What it measures:} CKNNA treats representations as similar when kernel values on shared (mutual) nearest-neighbor edges agree, combining \gls{cka}'s normalization with a purely local notion of neighborhood structure.

\section{Theoretical Derivations}
\label{app:theory}

In this section, we provide the theoretical justification for the confounding factors identified in \Cref{sec:theory}.

\subsection{Permutation validity, super-uniformity, and gating}
\label{app:perm_validity}

This section formalizes the finite-sample validity of permutation calibration.

\begin{definition}[Super-uniformity]
\label{def:super_uniform}
A $p$-value $p$ is \emph{super-uniform} under $H_0$ if for all $t\in[0,1]$,
\begin{equation}
    \Prob_{H_0}(p \le t)\le t.
\end{equation}
Equivalently, $p$-values under $H_0$ are stochastically larger than $\mathrm{Unif}(0,1)$, which is sufficient for valid Type-I error control.
\end{definition}

\begin{lemma}[Permutation $p$-values are super-uniform]
\label{lem:perm_super_uniform}
Under \Cref{assm:exchangeability}, the permutation $p$-value in \Cref{eq:pvalue_scalar} satisfies super-uniformity:
$\Prob_{H_0}(p \le \alpha)\le \alpha$ for all $\alpha\in[0,1]$ (finite-sample validity).
\end{lemma}

\begin{proof}[Proof of \Cref{lem:perm_super_uniform}]
Let $s_{\mathrm{obs}}=s(\mat{X},\mat{Y})$ be the observed statistic and let $s^{(k)}=s(\mat{X},\pi_k(\mat{Y}))$ for $k=1,\dots,K$ be the statistics computed on permuted pairings.
Under \Cref{assm:exchangeability}, the vector $(s_{\mathrm{obs}},s^{(1)},\dots,s^{(K)})$ is \emph{exchangeable}: its joint distribution is invariant to permutations of the indices.
Consider the (upper) rank
\begin{equation}
    R \;=\; 1 + \#\{k\in\{1,\dots,K\}: s^{(k)} \ge s_{\mathrm{obs}}\}\ \in\ \{1,\dots,K+1\}.
\end{equation}
If the scores are almost surely distinct, exchangeability implies that the rank of $s_{\mathrm{obs}}$ among $\{s_{\mathrm{obs}},s^{(1)},\dots,s^{(K)}\}$ is uniform on $\{1,\dots,K+1\}$.
With possible ties, the add-one $p$-value of \citet{phipson2010pvalues},
\begin{equation}
    p=\frac{R}{K+1},
\end{equation}
is conservative, implying $\Prob_{H_0}(p\le \alpha)\le \alpha$ for all $\alpha\in[0,1]$.
\end{proof}

\begin{proof}[Proof of \Cref{cor:gating_type1}]
Let $s_{\mathrm{obs}}=s(\mat{X},\mat{Y})$ and $s^{(k)}=s(\mat{X},\pi_k(\mat{Y}))$ for $k=1,\dots,K$. Under
\Cref{assm:exchangeability}, the vector $(s_{\mathrm{obs}},s^{(1)},\dots,s^{(K)})$ is exchangeable.
Let
\[
\tau_\alpha := s_{(\lceil (1-\alpha)(K+1)\rceil)}
\]
be the $(1-\alpha)$-quantile defined via the order statistic of the \emph{combined} multiset $\{s_{\mathrm{obs}},s^{(1)},\dots,s^{(K)}\}$.
Define the (upper) rank
\[
R \;=\; 1 + \#\{k\in\{1,\dots,K\}: s^{(k)} \ge s_{\mathrm{obs}}\}\ \in\ \{1,\dots,K+1\},
\]
and the corresponding add-one $p$-value $p=R/(K+1)$.
By construction of $\tau_\alpha$, the rejection event $\{s_{\mathrm{obs}}>\tau_\alpha\}$ implies that $s_{\mathrm{obs}}$ lies among the largest
$\lfloor \alpha(K+1)\rfloor$ values of $\{s_{\mathrm{obs}},s^{(1)},\dots,s^{(K)}\}$, hence $R \le \alpha(K+1)$ and therefore $p\le \alpha$.
By \Cref{lem:perm_super_uniform}, $\Prob_{H_0}(p\le \alpha)\le \alpha$, which yields
\[
\Prob_{H_0}(s_{\mathrm{obs}}>\tau_\alpha) \le \Prob_{H_0}(p\le \alpha) \le \alpha.
\]
\end{proof}

\paragraph{Restricted permutations under dependence.}
\Cref{assm:exchangeability} treats the $n$ row pairs as exchangeable units. In practice, exchangeability can be violated even without sequential structure (\textit{e.g.}, grouped or clustered samples). Validity is then recovered by using \emph{restricted} permutations that preserve the dependence structure (\textit{e.g.}, permuting within blocks or permuting block labels) and re-running under each restricted permutation.

\subsection{Monotone invariance of rank-based calibration}
\label{app:monotone}

The following proposition is a standard result in randomization inference; we state it here for completeness and to clarify its role in justifying the calibrated score design.

\begin{proposition}[Monotone invariance of rank-based calibration \textnormal{\citep{lehmann2005testing}}]
\label{prop:monotone_invariance}
Let $g:\R\to\R$ be strictly increasing.
Define $p_g$ by applying \Cref{eq:pvalue_scalar} to the transformed statistic $g\circ s$ using the \emph{same} permutations.
Then $p_g=p$, and likewise the null percentile (the rank of $s_{\mathrm{obs}}$ among the combined set) is invariant under $g$.
\end{proposition}

\begin{proof}
Let $g$ be strictly increasing.
For any two real numbers $a,b$, we have $a\ge b$ if and only if $g(a)\ge g(b)$.
Therefore, for each permutation draw $k$,
\begin{equation}
    \mathbbm{1}\{s^{(k)} \ge s_{\mathrm{obs}}\} = \mathbbm{1}\{g(s^{(k)})\ge g(s_{\mathrm{obs}})\}.
\end{equation}
Summing over $k$ shows that the permutation rank $R$ (and thus the add-one $p$-value) is unchanged by applying $g$ to both the observed and permuted statistics.
The same argument applies to the null percentile, since the ordering of samples is preserved under $g$.
\end{proof}

\subsection{Post-selection inflation and aggregation-aware validity}
\label{app:agg_validity}

\begin{proposition}[Validity for aggregation-aware calibration]
\label{prop:agg_validity}
Let $T$ be any measurable aggregation operator applied to a layer-wise similarity matrix $\mat{S}$ (\textit{e.g.}, max, row-max, top-$k$).
If $T_{\mathrm{obs}}=T(\mat{S})$ is calibrated against the permutation null $\{T(\mat{S}^{(k)})\}_{k=1}^K$ as in \Cref{eq:pvalue_agg},
then the resulting $p_{\mathrm{agg}}$ is super-uniform under $H_0$.
\end{proposition}

\begin{proof}[Proof of \Cref{prop:agg_validity}]
Let $T$ be any measurable functional of the full data (representations across all layers), producing the scalar report $T_{\mathrm{obs}}$.
Under \Cref{assm:exchangeability} and consistent layer-wise permutation of sample correspondences, the vector $(T_{\mathrm{obs}},T^{(1)},\dots,T^{(K)})$ is exchangeable.
Applying the same rank argument as in \Cref{app:perm_validity} yields super-uniformity for the add-one $p$-value in \Cref{eq:pvalue_agg}.
\end{proof}

\subsection{The width confounder}
\label{app:theory_rmt}

This appendix provides concrete calculations that justify the width confounder using \gls{rmt}: even under independence, interaction operators have non-trivial magnitude and spectrum when $d$ is not negligible relative to $n$.

\begin{proof}[Proof of \Cref{prop:nonzero_energy}]
Let $\mat{X}\in\R^{n\times d_x}$ and $\mat{Y}\in\R^{n\times d_y}$ have i.i.d.\ rows with mean $0$, identity covariance, and $\vect{x}_i$ and $\vect{y}_i$ independent.
Let $\mat{H}=\mat{I}_n - \frac{1}{n}\mathbbm{1}_n\mathbbm{1}_n^\top$ be the centering matrix, so $\mat{X}_c=\mat{H}\mat{X}$ and $\mat{Y}_c=\mat{H}\mat{Y}$.
Since $\mat{H}$ is symmetric and idempotent ($\mat{H}^2=\mat{H}$), the sample cross-covariance is
\begin{equation}
    \widetilde{\mat{C}} = \frac{1}{n-1}\mat{X}_c^\top\mat{Y}_c = \frac{1}{n-1}\mat{X}^\top\mat{H}\mat{Y}.
\end{equation}
Denote entry $(a,b)$ as $\widetilde{C}_{ab}$.
Expanding via $H_{ij}=\delta_{ij}-\frac{1}{n}$:
\begin{equation}
    \widetilde{C}_{ab} = \frac{1}{n-1}\left(\sum_{i=1}^n X_{ia}Y_{ib} - \frac{1}{n}\Bigl(\sum_{i=1}^n X_{ia}\Bigr)\Bigl(\sum_{j=1}^n Y_{jb}\Bigr)\right).
\end{equation}
We compute $\E[\widetilde{C}_{ab}^2]$ using independence of $\vect{x}_i$ and $\vect{y}_j$ for all $i,j$, zero means, and identity covariance.

\begin{description}
    \item[Term 1:]
    $\E\bigl[(\sum_i X_{ia}Y_{ib})^2\bigr] = \sum_{i,j}\E[X_{ia}X_{ja}]\E[Y_{ib}Y_{jb}]$.
    For $i\!\ne\!j$, independence across rows and zero mean give $\E[X_{ia}X_{ja}]=\E[X_{ia}]\E[X_{ja}]=0$.
    For $i=j$, we have $\E[X_{ia}^2]\E[Y_{ib}^2]=1$.
    Thus $\E\bigl[(\sum_i X_{ia}Y_{ib})^2\bigr]=n$.

    \item[Term 2:]
    $\E\bigl[(\sum_i X_{ia}Y_{ib})(\sum_j X_{ja})(\sum_k Y_{kb})\bigr] = \sum_{i,j,k}\E[X_{ia}X_{ja}]\E[Y_{ib}Y_{kb}]$.
    This is nonzero only when $i=j$ and $i=k$, yielding $\sum_i 1\cdot 1 = n$.

    \item[Term 3:]
    $\E\bigl[(\sum_i X_{ia})^2(\sum_j Y_{jb})^2\bigr] = \E[(\sum_i X_{ia})^2]\E[(\sum_j Y_{jb})^2] = n\cdot n = n^2$.
\end{description}

Combining:
\begin{align}
    \E\left[\widetilde{C}_{ab}^2\right]
    &= \frac{1}{(n-1)^2}\left(n - \frac{2}{n}\cdot n + \frac{n^2}{n^2}\right)
    = \frac{1}{(n-1)^2}(n-2+1)
    = \frac{1}{n-1}.
\end{align}
Summing over all entries:
\begin{equation}
    \E\left[\|\widetilde{\mat{C}}\|_F^2\right]= \sum_{a=1}^{d_x}\sum_{b=1}^{d_y}\E[\widetilde{C}_{ab}^2] = \frac{d_x d_y}{n-1}.
\end{equation}
\end{proof}

\paragraph{Interpretation.}
The null interaction energy is $\mathcal{O}(d_x d_y/n)$.
In the common regime $d_x,d_y\asymp n$, the null energy is $\mathcal{O}(n)$ and therefore \emph{does not vanish}.
Since many spectral similarity metrics aggregate singular values (\textit{e.g.}, via $\|\widetilde{\mat{C}}\|_F^2=\sum_i\sigma_i^2(\widetilde{\mat{C}})$), this already explains a positive baseline under $H_0$ and its dependence on $(n,d_x,d_y)$.

\paragraph{Width inflation persists under genuine signal ($H_1$).}
The width confounder is not specific to the null hypothesis. Even when two representations share a fixed, genuine signal, finite-sample \gls{cka} is inflated by width, because increasing $d$ grows the diagonal and off-diagonal Gram entries at different rates and the \gls{cka} denominator does not cancel the resulting width-dependent self-similarity.

\begin{proposition}[Width inflation under genuine signal]
\label{prop:width_h1}
Fix $n\ge 2$ and a target alignment $\rho\in(0,1)$. For each width $d$, draw $\vect{u}_i,\vect{v}_i,\vect{w}_i\stackrel{\mathrm{i.i.d.}}{\sim}\mathcal{N}(\mat{0},\mat{I}_d)$ independently across $i=1,\dots,n$, and set
\begin{equation}
\vect{x}_i=\sqrt{\rho}\,\vect{u}_i+\sqrt{1-\rho}\,\vect{v}_i,\qquad
\vect{y}_i=\sqrt{\rho}\,\vect{u}_i+\sqrt{1-\rho}\,\vect{w}_i,
\end{equation}
so that $\vect{x}_i$ and $\vect{y}_i$ share the latent signal $\vect{u}_i$. Then the population linear \gls{cka} equals $\rho^2<1$, whereas for \emph{fixed} $n$ the sample centered linear \gls{cka} satisfies $\mathrm{CKA}(\mat{X},\mat{Y})\to 1$ almost surely as $d\to\infty$. Hence the finite-sample estimate is inflated by width, with $\mathrm{CKA}(\mat{X},\mat{Y})-\rho^2\to 1-\rho^2>0$.
\end{proposition}

\begin{proof}
By independence of $\vect{u}_i,\vect{v}_i,\vect{w}_i$, the marginal covariances are $\Sigma_x=\Sigma_y=\mat{I}_d$ and the cross-covariance is $\Sigma_{xy}=\mathrm{Cov}(\vect{x}_i,\vect{y}_i)=\rho\,\mat{I}_d$. The population linear \gls{cka} is therefore
\begin{equation}
\mathrm{CKA}_{\mathrm{pop}}=\frac{\|\Sigma_{xy}\|_F^2}{\|\Sigma_x\|_F\,\|\Sigma_y\|_F}=\frac{\rho^2 d}{\sqrt{d}\,\sqrt{d}}=\rho^2.
\end{equation}
For the sample statistic at fixed $n$, write the columns of $\mat{X}$ as $\vect{g}_1,\dots,\vect{g}_d\in\R^n$. Each entry of $\mat{X}$ is $\mathcal{N}(0,1)$ (since $\rho+(1-\rho)=1$), and the columns are i.i.d.\ across $j$ with $\vect{g}_j\sim\mathcal{N}(\mat{0},\mat{I}_n)$. By the strong law of large numbers,
\begin{equation}
\frac{1}{d}\mat{X}\mat{X}^\top=\frac{1}{d}\sum_{j=1}^d \vect{g}_j\vect{g}_j^\top \xrightarrow{\mathrm{a.s.}} \E[\vect{g}_1\vect{g}_1^\top]=\mat{I}_n.
\end{equation}
Since $n$ is fixed, this convergence also holds in Frobenius norm, so $\tfrac{1}{d}\mat{H}\mat{X}\mat{X}^\top\mat{H}\to\mat{H}$ a.s.; the identical argument gives $\tfrac{1}{d}\mat{H}\mat{Y}\mat{Y}^\top\mat{H}\to\mat{H}$. By scale invariance of \gls{cka},
\begin{equation}
\mathrm{CKA}(\mat{X},\mat{Y})=\frac{\langle \tfrac{1}{d}\mat{H}\mat{X}\mat{X}^\top\mat{H},\,\tfrac{1}{d}\mat{H}\mat{Y}\mat{Y}^\top\mat{H}\rangle_F}{\|\tfrac{1}{d}\mat{H}\mat{X}\mat{X}^\top\mat{H}\|_F\,\|\tfrac{1}{d}\mat{H}\mat{Y}\mat{Y}^\top\mat{H}\|_F}\xrightarrow{\mathrm{a.s.}}\frac{\langle\mat{H},\mat{H}\rangle_F}{\|\mat{H}\|_F^2}=1.
\end{equation}
Subtracting the population value gives $\mathrm{CKA}(\mat{X},\mat{Y})-\rho^2\to 1-\rho^2>0$.
\end{proof}

Thus, at any fixed sample size $n$, increasing the width drives the estimated \gls{cka} toward $1$ regardless of the true alignment $\rho^2$, inflating it by up to $1-\rho^2$. This width-driven gap is distinct from the $\mathcal{O}(1/n)$ sample-size effect that vanishes as $n$ grows, and it is exactly what permutation calibration removes.
\Cref{app:width-h1-real} verifies this on real pretrained networks with nonzero true similarity.

\paragraph{CCA-based scores.}
\Gls{cca}, \gls{svcca}, and \gls{pwcca} fall outside the energy analysis above: their canonical correlations are computed from the \emph{whitened} cross-covariance, which normalizes away the cross-covariance energy of \Cref{prop:nonzero_energy}. They therefore do not follow the $\mathcal{O}(d/n)$ baseline. Their null instead reflects the rank deficiency of the whitening near and beyond $d \approx n$: mean \gls{cca} peaks at $d \approx n$ and decays as $\mathcal{O}(n/d)$ for $d > n$, \gls{pwcca} saturates toward its ceiling once $d \gtrsim n$, and \gls{svcca} is largely width-insensitive because it first projects onto a fixed low-dimensional subspace (\Cref{app:null-drift-full}). Permutation calibration removes this inflation without closed-form solutions, unlike the energy-based metrics.

\paragraph{Why we use permutation rather than closed forms.}
Closed-form bulk edges are ensemble- and normalization-specific and are brittle to the preprocessing used in practice (\textit{e.g.}, centering, whitening, kernelization).
Moreover, finite-$n$ corrections can be non-negligible.
We therefore estimate the relevant right-tail behavior nonparametrically via permutation.
This yields a conservative, implementation-faithful estimate of chance fluctuations without relying on fragile analytical formulas.

\subsection{The depth confounder}
\label{app:theory_evt}

Here we formalize why selection-based summaries (\textit{e.g.}, maximum similarity over layer pairs) inflate with the size of the search space using \gls{evt}.

Let $\mathcal{S}=\{S_{\ell,\ell'}:1\le \ell\le L_A,\,1\le \ell'\le L_B\}$ denote the collection of null similarity fluctuations under $H_0$, and let $M=L_A L_B$.

\begin{assumption}[Uniform sub-Gaussian right tails and integrability]\label{assm:subgaussian}
There exist $\mu \in \mathbb R$ and $\sigma>0$ such that for all $(\ell,\ell')$ and all $t\ge 0$,
\begin{equation}
    \Prob(S_{\ell,\ell'}-\mu \ge t)\le \exp\!\left(-\frac{t^2}{2\sigma^2}\right).
\end{equation}
Moreover, each $S_{\ell,\ell'}$ is integrable: $\E|S_{\ell,\ell'}|<\infty$ for all $(\ell,\ell')$.
\end{assumption}

\begin{proposition}[Maximal inequality, no independence required]
\label{prop:max_ineq}
Under \Cref{assm:subgaussian} and for $M\ge 2$,
\begin{equation}
    \E\Big[\max_{\ell,\ell'} S_{\ell,\ell'}\Big] \;\le\; \mu + C\,\sigma\sqrt{\log M},
\end{equation}
where $C>0$ is a constant (\textit{e.g.}, one can take $C=3$).
\end{proposition}

\begin{proof}
Let $Z := \max_{\ell,\ell'} S_{\ell,\ell'} - \mu$.
Since $M<\infty$ and $\E|S_{\ell,\ell'}|<\infty$ for all $(\ell,\ell')$, we have
\begin{equation}
\E|Z| \le \E\Big[\max_{\ell,\ell'}|S_{\ell,\ell'}|\Big] + |\mu|
\le \sum_{\ell,\ell'}\E|S_{\ell,\ell'}| + |\mu| < \infty,
\end{equation}
so $Z$ is integrable, and the tail-integration formula applies.
By the union bound and 
\Cref{assm:subgaussian},
\begin{equation}
    \Prob(Z\ge t)\le M\exp\!\left(-\frac{t^2}{2\sigma^2}\right)\qquad \text{for all }t\ge 0.
\end{equation}
Using the tail-integration formula for an integrable real-valued random variable $Z$,
\begin{equation}
    \mathbb E[Z] = \int_0^\infty \mathbb P(Z \ge t)\,dt
    - \int_0^\infty \mathbb P(Z \le -t)\,dt
    \le \int_0^\infty \mathbb P(Z \ge t)\,dt,
\end{equation}
and the bound \(\mathbb P(Z \ge t)\le 1\), we obtain
\begin{equation}
\E[Z]\;\le\;\int_{0}^{\infty}\min\!\left\{1,\;M\exp\!\left(-\frac{t^2}{2\sigma^2}\right)\right\}\,dt.
\end{equation}
Let $t_0=\sigma\sqrt{2\log M}$.
This value of $t_0$ is the solution of $M\exp\!\left(-t_0^2/2\sigma^2\right) = 1$, \textit{i.e.}, the crossover where the bound $\min\{1, \cdot\}$ switches.
Splitting the integral at $t_0$ yields
\begin{equation}
\E[Z]\;\le\; t_0 \;+\; M\int_{t_0}^{\infty}\exp\!\left(-\frac{t^2}{2\sigma^2}\right)\,dt.
\end{equation}
Applying the standard Gaussian tail bound $\int_{t_0}^{\infty}e^{-t^2/(2\sigma^2)}dt \le (\sigma^2/t_0)e^{-t_0^2/(2\sigma^2)}$ gives
\begin{equation}
\E[Z]\;\le\;\sigma\sqrt{2\log M}\;+\;\frac{\sigma}{\sqrt{2\log M}}.
\end{equation}
For $M\ge 2$, the right-hand side is at most $3\sigma\sqrt{\log M}$, proving the claim with $C=3$.
\end{proof}

\paragraph{Remark.}
When the $S_{\ell,\ell'}$ are i.i.d.\ (or weakly dependent), classical \acrlong{evt} yields sharper asymptotics.
For example, if $S_{\ell,\ell'}\sim\mathcal{N}(\mu_0,\sigma_0^2)$ i.i.d., the centered maximum converges to a Gumbel distribution and
\begin{equation}
    \E[T_{\max}] \approx \mu_0 + \sigma_0 \left( \sqrt{2 \ln M} - \frac{\ln \ln M + \ln 4\pi}{2\sqrt{2 \ln M}} \right),
\end{equation}
as stated in standard references \citep{cramer1999mathematical,embrechts2013modelling}.
Real layer-wise similarities are dependent, so the approximation above should be treated as heuristic; \Cref{prop:max_ineq} provides a dependence-robust upper bound.

\subsection{Null Baselines for Neighborhood Metrics}
\label{app:neighborhood-null}

The preceding analysis focused on the cross-covariance-energy metrics (\gls{cka} and the RV coefficient), whose null baselines scale with $d/n$.
Neighborhood-based metrics such as mutual $k$-NN follow a fundamentally different regime, which we now characterize.

\begin{definition}[Mutual $k$-NN overlap]
\label{def:mknn}
For representations $\mat{X} \in \R^{n \times d_x}, \mat{Y} \in \R^{n \times d_y}$ and neighborhood size $k < n$, let $N_{\mat{X}}(i) \subseteq \{1, \ldots, n\} \setminus \{i\}$ denote the indices of the $k$ nearest neighbors of sample $i$ in $\mat{X}$ (\textit{e.g.}, Euclidean or cosine), and similarly for $N_{\mat{Y}}(i)$.
The mutual $k$-NN overlap is
\begin{equation}
    \mathrm{mKNN}(\mat{X}, \mat{Y}) = \frac{1}{n} \sum_{i=1}^{n} \frac{|N_{\mat{X}}(i) \cap N_{\mat{Y}}(i)|}{k}.
\end{equation}
\end{definition}


\begin{proposition}[Uniformity of $k$-NN index sets under i.i.d.\ sampling]
\label{prop:knn_uniform}
Fix an anchor index $i\in\{1,\dots,n\}$. Let $\vect{x}_1,\dots,\vect{x}_n\in\R^d$ be i.i.d. and define the $k$-NN set
$N_{\mat{X}}(i)\subseteq\{1,\dots,n\}\setminus\{i\}$ using a fixed distance $\mathrm{dist}(\cdot,\cdot)$.
Assume either \textit{(i)} $\{\mathrm{dist}(\vect{x}_i,\vect{x}_j)\}_{j\neq i}$ are almost surely distinct, or \textit{(ii)} ties are broken by selecting a
uniformly random $k$-subset among the set of minimizers. Then $N_{\mat{X}}(i)$ is uniformly distributed over the
$\binom{n-1}{k}$ $k$-subsets of $\{1,\dots,n\}\setminus\{i\}$.
\end{proposition}

\begin{proof}
Let $\mathcal{I}:=\{1,\dots,n\}\setminus\{i\}$ be the candidate-neighbor index set.
For any permutation $\pi$ of $\mathcal{I}$, i.i.d.\ sampling implies
\[
(\vect{x}_j)_{j\in\mathcal{I}} \stackrel{d}{=} (\vect{x}_{\pi(j)})_{j\in\mathcal{I}}.
\]
The $k$-NN selection rule depends on the candidate points only through their distances to $\vect{x}_i$, so permuting the
candidate indices permutes the resulting neighbor set. Under either the no-ties assumption or the stated uniform tie-break
rule, for any two $k$-subsets $S,S'\subseteq\mathcal{I}$ there exists a permutation $\pi$ with $\pi(S)=S'$ and hence
\[
\Prob\!\big(N_{\mat{X}}(i)=S\big) = \Prob\!\big(N_{\mat{X}}(i)=S'\big).
\]
Since the events $\{N_{\mat{X}}(i)=S\}$ over all $|S|=k$ partition the sample space, each has probability
$\binom{n-1}{k}^{-1}$.
\end{proof}

\begin{theorem}[Null baseline for mutual $k$-NN]
\label{thm:mknn_null}
Let $\mat{X},\mat{Y}\in\R^{n\times d}$ have i.i.d.\ rows, with $\mat{X}$ independent of $\mat{Y}$.
Define $N_{\mat{X}}(i)$ and $N_{\mat{Y}}(i)$ as in \Cref{def:mknn}, using either almost sure absence of distance ties or
uniform random tie-breaking. Then
\[
\E_{H_0}\!\bigg[\mathrm{mKNN}(\mat{X},\mat{Y})\bigg] \;=\; \frac{k}{n-1}.
\]
\end{theorem}

\begin{proof}
Fix an anchor $i$. By \Cref{prop:knn_uniform}, $N_{\mat{X}}(i)$ and $N_{\mat{Y}}(i)$ are each uniform random $k$-subsets
of the $(n-1)$-element set $\{1,\dots,n\}\setminus\{i\}$. Moreover, since $\mat{X}$ and $\mat{Y}$ are independent and
$N_{\mat{X}}(i)$ (resp.\ $N_{\mat{Y}}(i)$) is a measurable function of $\mat{X}$ (resp.\ $\mat{Y}$), the sets
$N_{\mat{X}}(i)$ and $N_{\mat{Y}}(i)$ are independent.

Therefore $|N_{\mat{X}}(i)\cap N_{\mat{Y}}(i)|$ has a hypergeometric distribution with population size $n-1$, number of
``successes'' $k$, and draws $k$, giving
\[
\E_{H_0}\!\bigg[|N_{\mat{X}}(i)\cap N_{\mat{Y}}(i)|\bigg] = \frac{k^2}{n-1}.
\]
Substituting into the definition of $\mathrm{mKNN}$,
\[
\E_{H_0}\bigg[\mathrm{mKNN}(\mat{X},\mat{Y})\bigg]
= \frac{1}{n}\sum_{i=1}^n \E_{H_0}\!\left[\frac{|N_{\mat{X}}(i) \cap N_{\mat{Y}}(i)|}{k}\right]
= \frac{1}{n}\sum_{i=1}^n \frac{k}{n-1}
= \frac{k}{n-1}.
\]
\end{proof}

\begin{proposition}[Per-anchor variance and generic bounds for $\mathrm{mKNN}$ under the null]
\label{prop:mknn_variance}
Under the assumptions of \Cref{thm:mknn_null}, for each anchor $i$ the intersection size
$H_i := |N_{\mat{X}}(i)\cap N_{\mat{Y}}(i)|$ is hypergeometric with mean $k^2/(n-1)$ and variance
\[
\mathrm{Var}[H_i] \;=\; \frac{k^2 (n-1-k)^2}{(n-1)^2 (n-2)}.
\]
Moreover, since $\mathrm{mKNN}(\mat{X},\mat{Y})\in[0,1]$ deterministically, we have the fully general bound
\[
\mathrm{Var}[\mathrm{mKNN}(\mat{X},\mat{Y})] \le \frac{1}{4}.
\]
If one \emph{additionally assumes} that the per-anchor terms $\{|N_{\mat{X}}(i)\cap N_{\mat{Y}}(i)|/k\}_{i=1}^n$ are
independent (this is a modeling assumption, not a consequence of $H_0$), then
$\mathrm{Var}[\mathrm{mKNN}(\mat{X},\mat{Y})]=\mathcal{O}(1/n)$.
\end{proposition}

\begin{proof}
The hypergeometric variance formula gives
\[
\mathrm{Var}[H_i]
=
k\cdot \frac{k}{n-1}\left(1-\frac{k}{n-1}\right)\cdot \frac{(n-1)-k}{(n-1)-1}
=
\frac{k^2 (n-1-k)^2}{(n-1)^2 (n-2)}.
\]
The bound $\mathrm{Var}[\mathrm{mKNN}]\le 1/4$ follows from $\mathrm{mKNN}\in[0,1]$.
Under the stated additional independence assumption across anchors,
\[
\mathrm{Var}\bigg[\mathrm{mKNN}(\mat{X}, \mat{Y})\bigg] = \frac{1}{n}\,\mathrm{Var}\!\left(\frac{H_1}{k}\right) = \frac{1}{n k^2}\,\mathrm{Var}[H_1],
\]
which is $\mathcal{O}(1/n)$ for fixed $k$.
\end{proof}

\section{Implementation}
\label{app:implementation}

A key advantage of null calibration is its simplicity: the framework can be applied to \emph{any} similarity metric with minimal code changes.
This section provides pseudocode for the two main calibration procedures described in the paper.

\paragraph{Scalar null calibration.}
\Cref{alg:scalar_calibration} shows the complete procedure for calibrating a single similarity comparison.
The only requirement is a function \texttt{similarity(X,Y)} that computes the raw metric.
The algorithm returns both a permutation $p$-value and a calibrated score with a principled zero point.

\begin{algorithm}[htbp]
\caption{Scalar Null Calibration}
\label{alg:scalar_calibration}
\begin{algorithmic}[1]
\REQUIRE Representations $\mat{X} \in \R^{n \times d_x}$, $\mat{Y} \in \R^{n \times d_y}$
\REQUIRE Similarity function $\texttt{sim}(\cdot, \cdot)$, permutations $K$, significance level $\alpha$
\ENSURE Calibrated score $s_{\mathrm{cal}}$, $p$-value $p$
\STATE $s_{\mathrm{obs}} \gets \texttt{sim}(\mat{X}, \mat{Y})$ \COMMENT{Observed similarity}
\STATE $\texttt{null\_scores} \gets []$
\FOR{$k = 1$ to $K$}
    \STATE $\pi \gets \texttt{random\_permutation}(n)$ \COMMENT{Permute sample indices}
    \STATE $\mat{Y}_\pi \gets \mat{Y}[\pi, :]$ \COMMENT{Permute rows of $\mat{Y}$}
    \STATE $\texttt{null\_scores}[k] \gets \texttt{sim}(\mat{X}, \mat{Y}_\pi)$
\ENDFOR
\STATE $\texttt{combined} \gets [s_{\mathrm{obs}}] \cup \texttt{null\_scores}$ \COMMENT{Combined set}
\STATE $\tau_\alpha \gets s_{(\lceil (1-\alpha)(K+1) \rceil)}$ \COMMENT{Ceiling order statistic of sorted \texttt{combined} (\Cref{eq:get_tau_alpha}); not an interpolated quantile}
\STATE $p \gets \frac{1 + \sum_{k=1}^{K} \mathbbm{1}[\texttt{null\_scores}[k] \geq s_{\mathrm{obs}}]}{K + 1}$ \COMMENT{Permutation $p$-value}
\STATE $s_{\mathrm{cal}} \gets \max\left(\frac{s_{\mathrm{obs}} - \tau_\alpha}{s_{\max} - \tau_\alpha}, 0\right)$ \COMMENT{Calibrated score (use $s_{\max}=1$ for bounded metrics)}
\RETURN $s_{\mathrm{cal}}, p$
\end{algorithmic}
\end{algorithm}

\paragraph{Aggregation-aware calibration for layer-wise comparisons.}
When comparing models with multiple layers and reporting a summary statistic (\textit{e.g.}, maximum similarity across layer pairs), the aggregation step must also be calibrated.
\Cref{alg:aggregation_calibration} shows how to extend scalar calibration to this setting.
The key insight is that the \emph{same} sample permutation must be applied consistently across all layers.

\begin{algorithm}[htbp]
\caption{Aggregation-Aware Null Calibration}
\label{alg:aggregation_calibration}
\begin{algorithmic}[1]
\REQUIRE Layer representations $\{\mat{X}^{(\ell)}\}_{\ell=1}^{L_A}$, $\{\mat{Y}^{(\ell')}\}_{\ell'=1}^{L_B}$ (all $n$ samples)
\REQUIRE Similarity function $\texttt{sim}(\cdot, \cdot)$, aggregator $T$ (\textit{e.g.}, $\max$), permutations $K$, level $\alpha$
\ENSURE Calibrated aggregate $T_{\mathrm{cal}}$, $p$-value $p_{\mathrm{agg}}$
\STATE \COMMENT{Compute observed similarity matrix}
\FOR{$\ell = 1$ to $L_A$}
    \FOR{$\ell' = 1$ to $L_B$}
        \STATE $\mat{S}[\ell, \ell'] \gets \texttt{sim}(\mat{X}^{(\ell)}, \mat{Y}^{(\ell')})$
    \ENDFOR
\ENDFOR
\STATE $T_{\mathrm{obs}} \gets T(\mat{S})$ \COMMENT{\textit{e.g.}, $\max_{\ell,\ell'} \mat{S}[\ell,\ell']$}
\STATE $\texttt{null\_aggregates} \gets []$
\FOR{$k = 1$ to $K$}
    \STATE $\pi \gets \texttt{random\_permutation}(n)$ \COMMENT{Single permutation for all layers}
    \FOR{$\ell = 1$ to $L_A$}
        \FOR{$\ell' = 1$ to $L_B$}
            \STATE $\mat{S}^{(k)}[\ell, \ell'] \gets \texttt{sim}(\mat{X}^{(\ell)}, \mat{Y}^{(\ell')}[\pi, :])$ \COMMENT{Same $\pi$ for all $\ell'$}
        \ENDFOR
    \ENDFOR
    \STATE $\texttt{null\_aggregates}[k] \gets T(\mat{S}^{(k)})$ \COMMENT{Aggregate under null}
\ENDFOR
\STATE $\texttt{combined} \gets [T_{\mathrm{obs}}] \cup \texttt{null\_aggregates}$ \COMMENT{Combined set}
\STATE $\tau_\alpha^{\mathrm{agg}} \gets T_{(\lceil (1-\alpha)(K+1) \rceil)}$ \COMMENT{Ceiling order statistic of sorted \texttt{combined}; not an interpolated quantile}
\STATE $p_{\mathrm{agg}} \gets \frac{1 + \sum_{k=1}^{K} \mathbbm{1}[\texttt{null\_aggregates}[k] \geq T_{\mathrm{obs}}]}{K + 1}$
\STATE $T_{\mathrm{cal}} \gets \max\left(\frac{T_{\mathrm{obs}} - \tau_\alpha^{\mathrm{agg}}}{s_{\max} - \tau_\alpha^{\mathrm{agg}}}, 0\right)$
\RETURN $T_{\mathrm{cal}}, p_{\mathrm{agg}}$
\end{algorithmic}
\end{algorithm}

\paragraph{Computational cost.}
Let $C_{\mathrm{sim}}$ denote the cost of a single similarity evaluation on $n$ samples.
Scalar calibration (\Cref{alg:scalar_calibration}) performs $K+1$ evaluations at cost $\mathcal{O}\!\left((K{+}1)\,C_{\mathrm{sim}}\right)$, and aggregation-aware calibration (\Cref{alg:aggregation_calibration}) over an $L_A \times L_B$ layer grid costs $\mathcal{O}\!\left((K{+}1)\,L_A L_B\,C_{\mathrm{sim}}\right)$.
The $K$ null evaluations are independent and run in parallel.
The marginal cost per permutation is typically far below $C_{\mathrm{sim}}$, because permuting the rows of $\mat{Y}$ leaves each model's own representations unchanged.
Metrics that are functionals of the $n\times n$ Gram, dissimilarity, or neighbor structures (\gls{cka}, the RV coefficient, \gls{rsa}, CKNNA) and the neighborhood metrics (\gls{mknn}, cycle-$k$NN) build these structures once in $\mathcal{O}(n^2 d)$ and reduce each null draw to a relabeling, at $\mathcal{O}(n^2)$ or $\mathcal{O}(nk)$ and independent of the representation width $d$.
The subspace and shape metrics (\gls{cca}, \gls{svcca}, \gls{pwcca}, Procrustes) are the exception: each null draw recomputes an eigendecomposition or SVD of a $d$-dependent operator, at up to $\mathcal{O}(n d^2 + d^3)$ per permutation, which we batch across permutations on the GPU.
On a single NVIDIA GTX 1080\,Ti with $K=200$, this overhead is $20$\,ms per layer pair for linear \gls{cka} and $52$\,ms for \gls{mknn}, so a full layer-wise comparison between two models completes within a few seconds.
We use $K \in \{200, \dots, 500\}$ permutations throughout, which we find sufficient for stable thresholds and $p$-values (\Cref{app:perm-budget}).

\section{Additional Experimental Results}
\label{app:additional-experiments}

This appendix provides additional analyses that support the main text claims.

\subsection{Phase diagrams across different noise distributions}
\label{app:phase-diagrams}

The theoretical analysis in \Cref{sec:theory} assumes Gaussian entries for tractability, but real neural network activations rarely follow Gaussian distributions.
Instead, they often exhibit heavy tails, sparsity, or multimodality.
A critical question is whether our calibration, which makes no distributional assumptions, remains effective under such deviations.

\Cref{fig:phase-diagram} shows phase diagrams under different noise distributions: Gaussian, Student-$t$ ($\nu=3$), Laplace, and Gaussian mixtures.
Each panel shows raw scores (left) and calibrated scores (right) across the $(d/n, \sigma)$ grid, where $\sigma$ controls the noise level added to a fixed shared signal.
At low $\sigma$, the signal dominates and both raw and calibrated scores correctly indicate high similarity.
At high $\sigma$, noise overwhelms the signal, and similarity should approach zero.
The key finding is that raw scores remain elevated (around 0.4--0.6) even at high noise levels where no detectable signal remains, while calibrated scores correctly collapse to near-zero.
This pattern holds across all noise distributions tested, confirming that permutation-based calibration adapts to the data-generating process without requiring explicit distributional modeling.

\begin{figure}[htbp]
  \centering
   \begin{subfigure}[t]{0.49\linewidth}
    \includegraphics[width=\linewidth]{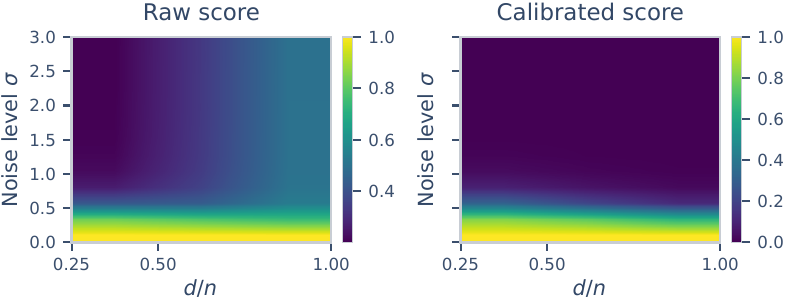}
    \caption{Gaussian}
  \end{subfigure}
  \hfill
  \begin{subfigure}[t]{0.49\linewidth}
    \includegraphics[width=\linewidth]{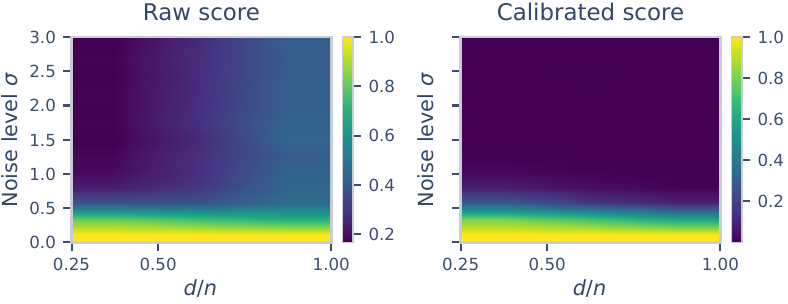}
    \caption{Student-$t$ ($\nu=3$)}
  \end{subfigure}
  \begin{subfigure}[t]{0.49\linewidth}
    \includegraphics[width=\linewidth]{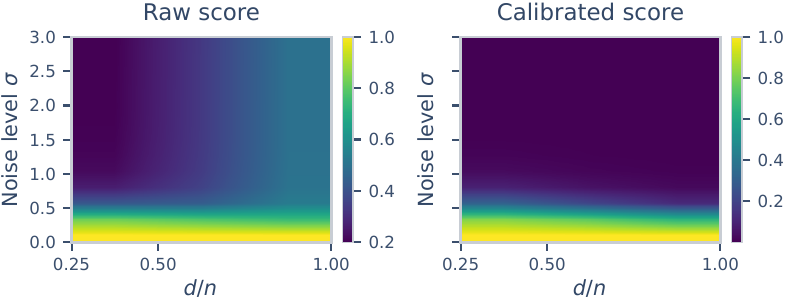}
    \caption{Laplace}
  \end{subfigure}
  \hfill
  \begin{subfigure}[t]{0.49\linewidth}
    \includegraphics[width=\linewidth]{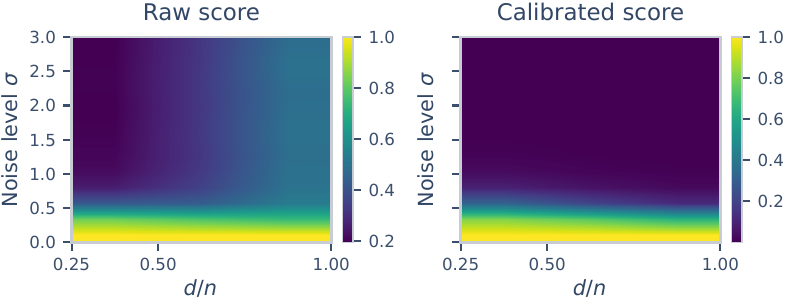}
    \caption{Gaussian mixture}
  \end{subfigure}
  \caption{
  \textbf{Phase diagrams under different noise types.}
  Calibrated scores (right) collapse to near-zero at high noise levels across the $(d/n, \sigma)$ grid, while raw scores (left) exhibit systematic positive bias.
  Calibration remains effective regardless of tail behavior.
  }
  \label{fig:phase-diagram}
\end{figure}

\subsection{SNR heatmaps}
\label{app:snr-heatmaps}

The experiments of the main paper (\Cref{fig:statistical-guarantees}) demonstrated that calibration eliminates false positives under $H_0$ while preserving sensitivity to fixed signals.
This section extends the analysis by characterizing how calibrated similarity varies jointly with signal strength, noise level, and dimensionality ratio, thereby delineating the regimes in which similarity estimation remains reliable.

\Cref{fig:snr-heatmaps} presents heatmaps of raw scores (top row) and calibrated scores (bottom row) across the (Noise level, Signal strength) grid for three signal ranks ($r \in \{1, 5, 10\}$).
The results reveal a clear phase transition structure.
Raw scores (top) show uniformly high values across most of the grid, obscuring the true detection boundary.
Calibrated scores (bottom) reveal the underlying signal: high scores concentrate in the low-noise, high-signal corner (bottom-left), while scores correctly collapse to zero as noise increases (moving right) or signal weakens (moving down).
The detection boundary shifts rightward (tolerating higher noise) as signal rank increases.
This phase structure is meaningful: it delineates when similarity measurements carry information about shared structure versus when they reflect only finite-sample artifacts.

\begin{figure}[htbp]
  \centering
  \begin{subfigure}[t]{0.32\linewidth}
    \includegraphics[width=\linewidth]{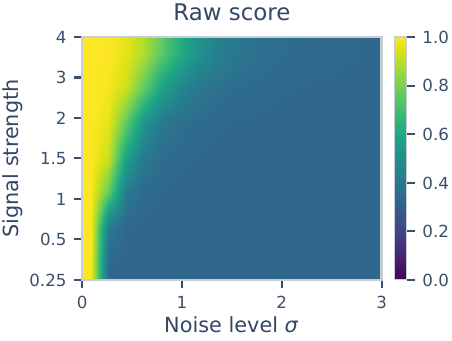}
    \caption{Rank $r=1$}
  \end{subfigure}
  \hfill
  \begin{subfigure}[t]{0.32\linewidth}
    \includegraphics[width=\linewidth]{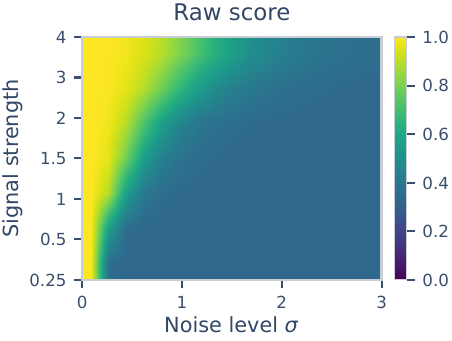}
    \caption{Rank $r=5$}
  \end{subfigure}
  \hfill
  \begin{subfigure}[t]{0.32\linewidth}
    \includegraphics[width=\linewidth]{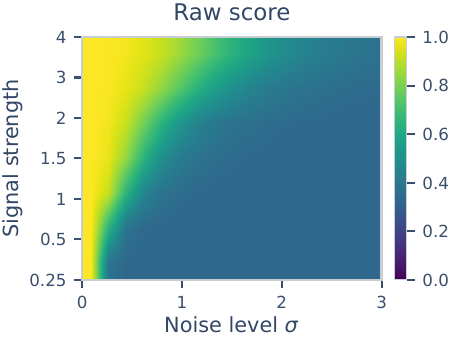}
    \caption{Rank $r=10$}
  \end{subfigure}
  \begin{subfigure}[t]{0.32\linewidth}
    \includegraphics[width=\linewidth]{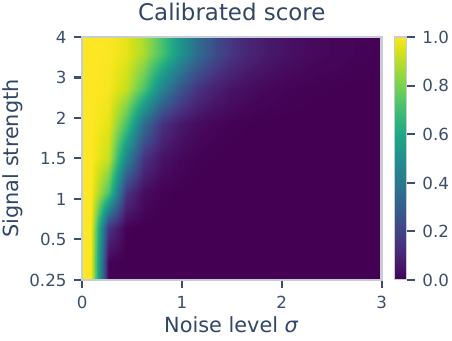}
    \caption{Rank $r=1$}
  \end{subfigure}
  \hfill
  \begin{subfigure}[t]{0.32\linewidth}
    \includegraphics[width=\linewidth]{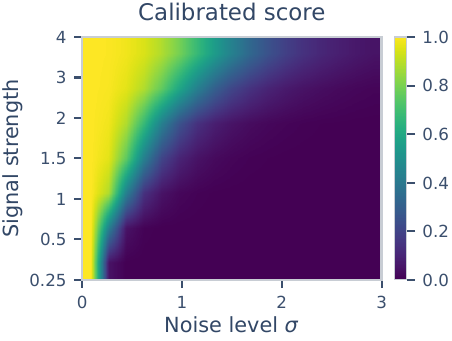}
    \caption{Rank $r=5$}
  \end{subfigure}
  \hfill
  \begin{subfigure}[t]{0.32\linewidth}
    \includegraphics[width=\linewidth]{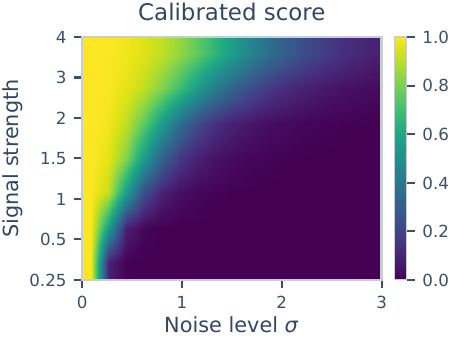}
    \caption{Rank $r=10$}
  \end{subfigure}
  \caption{
  \textbf{SNR heatmaps (calibrated scores).}
  Higher-rank signals are detected at higher noise levels.
  The clear gradient confirms calibration preserves sensitivity to genuine structure.
  }
  \label{fig:snr-heatmaps}
\end{figure}

\Cref{fig:snr-sweep-strength} provides a complementary view by collapsing the 2D heatmaps into 1D curves, plotting calibrated score against noise level for different signal strengths $s$.
As expected, calibrated scores decrease monotonically with noise level: at low noise, scores are high (reflecting the detectable shared signal), while at high noise, scores collapse to zero (reflecting that the signal is buried).
Stronger signals (larger $s$) maintain elevated scores across a wider range of noise levels before eventually succumbing.
Higher-rank signals ($r = 5, 10$) show more gradual decay compared to $r = 1$, consistent with their greater statistical detectability.
All curves converge to zero at high noise, confirming that the null floor is correctly calibrated regardless of signal strength or rank.

\begin{figure}[htbp]
  \centering
  \includegraphics[width=0.7\linewidth]{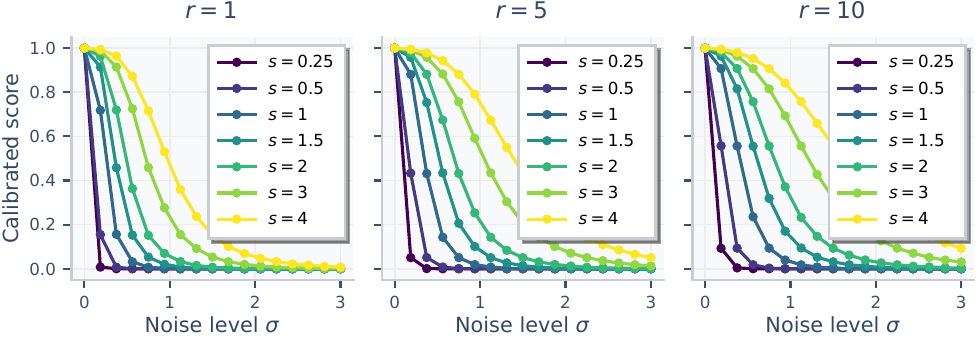}
  \caption{
  \textbf{Calibrated scores decay with noise level.}
  Each curve shows calibrated score versus noise level for a fixed signal strength $s$.
  Stronger signals maintain elevated scores across wider noise ranges; all curves converge to zero at high noise.
  }
  \label{fig:snr-sweep-strength}
\end{figure}

\subsection{Comparing calibration approaches}
\label{app:per-metric}

A natural question is whether the choice of calibration summary affects the correction.
We consider several approaches: \textit{(i)}~\emph{gated score}, which thresholds at a significance level and rescales ($\alpha \in \{0.05, 0.1\}$); \textit{(ii)}~\emph{null-centered}, subtracting the null mean; \textit{(iii)}~\emph{z-score}, standardizing by null mean and standard deviation; and \textit{(iv)}~\emph{ARI-style}, applying the chance-correction formula $(s - \E[s]) / (s_{\max} - \E[s])$.
\Cref{fig:null-drift-per-metric} evaluates these variants across metrics as $d/n$ increases.

The results demonstrate that the gated score, null-centered, and ARI-style corrections all successfully collapse to appropriate null baselines across all metrics, regardless of whether the raw metric exhibits severe inflation (CKA, approaching 0.8) or mild inflation (RSA and mKNN, below 0.1).
The z-score calibration, while correcting the mean, can exhibit artifacts when the null distribution is skewed, as occurs for bounded metrics like CKA at high $d/n$, making it less suitable as a universal correction.

\begin{figure}[htbp]
  \centering
  \begin{subfigure}[t]{\linewidth}
    \includegraphics[width=\linewidth]{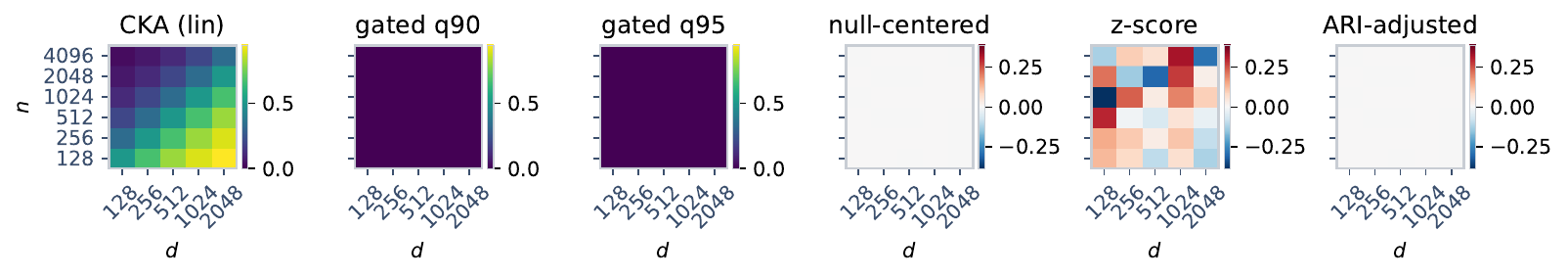}
    \caption{CKA linear}
  \end{subfigure}
  \begin{subfigure}[t]{\linewidth}
    \includegraphics[width=\linewidth]{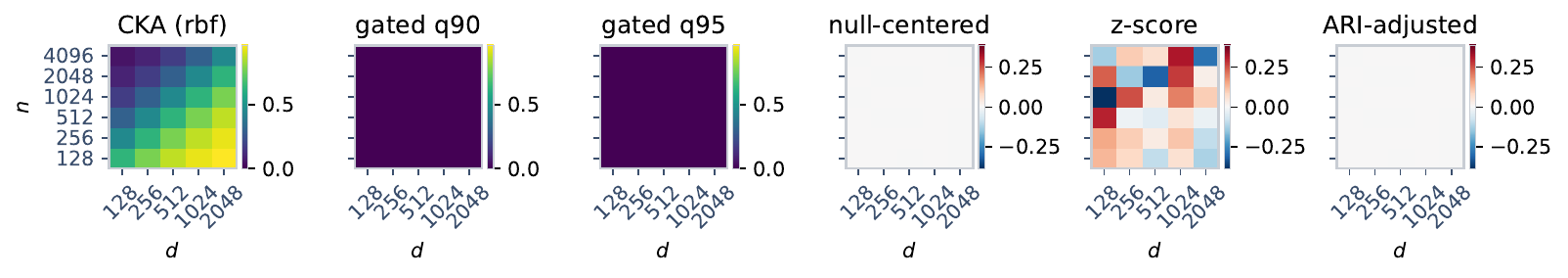}
    \caption{CKA RBF}
  \end{subfigure}
  \begin{subfigure}[t]{\linewidth}
    \includegraphics[width=\linewidth]{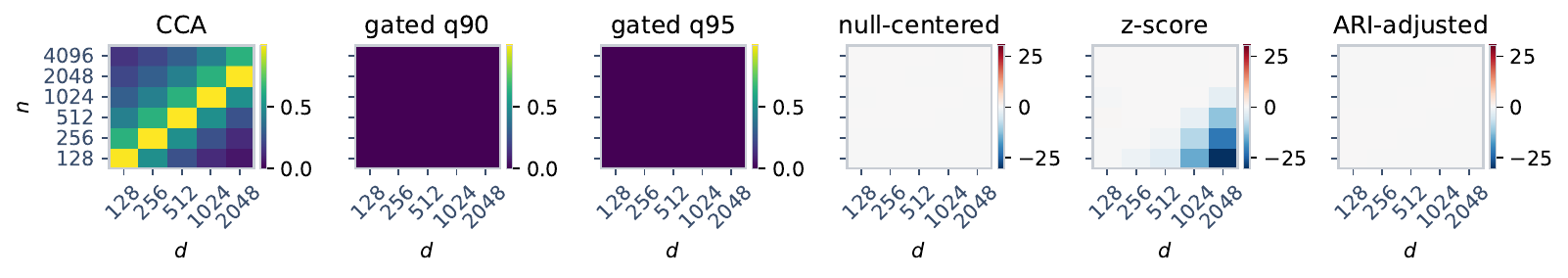}
    \caption{CCA}
  \end{subfigure}
  \begin{subfigure}[t]{\linewidth}
    \includegraphics[width=\linewidth]{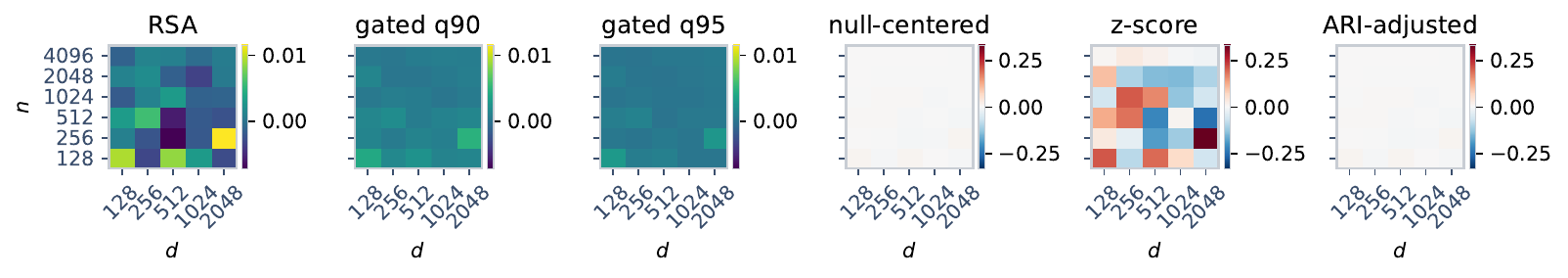}
    \caption{RSA (Spearman)}
  \end{subfigure}
  \begin{subfigure}[t]{\linewidth}
    \includegraphics[width=\linewidth]{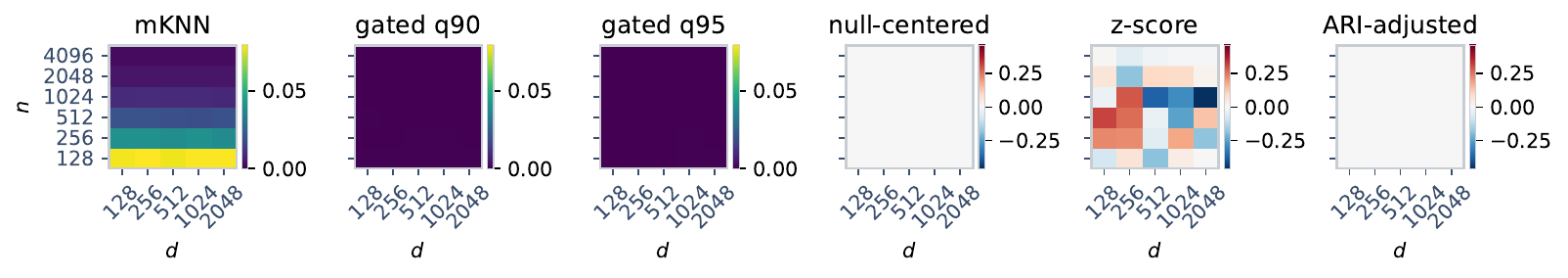}
    \caption{Mutual $k$-NN}
  \end{subfigure}
  \caption{
  \textbf{Comparing calibration approaches across metrics.}
  Each panel shows raw scores alongside four calibration variants (gated score, null-centered, z-score, ARI-style) as $d/n$ increases.
  Gated score, null-centered, and ARI-style corrections collapse to appropriate baselines; z-score exhibits artifacts for skewed null distributions.
  }
  \label{fig:null-drift-per-metric}
\end{figure}

\subsection{Comparison with analytical debiasing}
\label{app:analytical-comparison}

We validate our empirical null calibration by comparing it to existing analytical bias corrections for \gls{cka}.
\Cref{fig:cka-comparison} shows the difference between our calibrated \gls{cka} and two existing estimators: the debiased \gls{cka} of \citet{murphy2024biasedcka} and the dep-cols \gls{cka} of \citet{chun2025sparsecka}.

Our calibrated \gls{cka} closely matches the debiased \gls{cka} estimator, indicating that our calibration automatically corrects the dominant width-induced bias without requiring a metric-specific derivation.
In contrast, dep-cols \gls{cka} is designed to correct column dependence, which is not present in our experimental setup (columns are independent by construction), so it does not address the dominant width-induced inflation. Under genuine signal ($H_1$) it stays near its maximal value across all $d/n$, far above the closely agreeing debiased and calibrated estimates, while under the null ($H_0$) it fluctuates around zero with high variance.

Our calibration targets the same source of finite-sample bias that the unbiased \acrshort{hsic} estimator of \citet{song2012hsic} subtracts analytically: the self-similarity (diagonal) terms. This is consistent with the close agreement between the two corrections in \Cref{fig:cka-comparison}. Unlike such analytical estimators, our permutation calibration requires no metric-specific derivation and applies to metrics for which no debiasing exists.

\begin{figure}[H]
  \centering
  \includegraphics[width=0.8\linewidth]{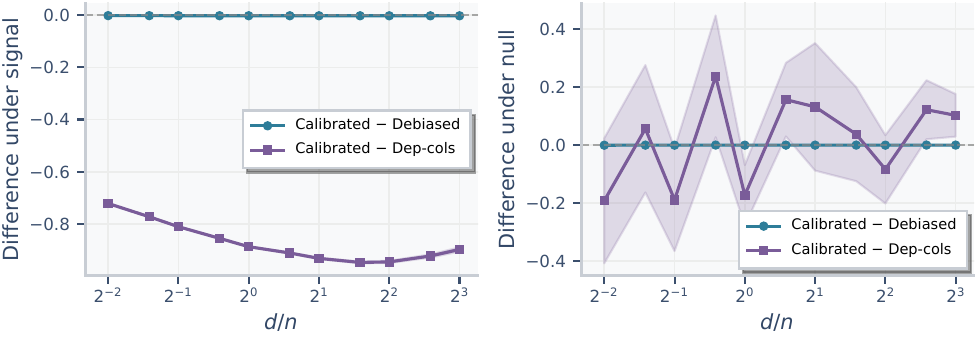}
  \caption{
  \textbf{Calibration recovers analytical debiasing.}
  Difference between calibrated \gls{cka} and existing estimators ($n=1024$, $d/n$ swept).
  (Left)~Under signal.
  (Right)~Under null.
  }
  \label{fig:cka-comparison}
\end{figure}

\subsection{Width confounder under genuine signal on real networks}
\label{app:width-h1-real}

We now verify on real pretrained networks that calibration corrects the width confounder when a genuine signal is present, complementing the analytical result of \Cref{prop:width_h1}. Because the bias scales as $\mathcal{O}(d/n)$ (\Cref{prop:nonzero_energy}), width and sample size can be separated only by varying one while holding the other fixed, which this experiment does.

We extract last-layer features from the DINOv2 and AugReg ViT families. Within a family the models share training objective, data, and architecture and differ only in width $d$ (\textit{e.g.}, ViT-S/B/L/g for DINOv2), so a within-family pair isolates the effect of width. All pairs use the same $n=1024$ WIT images as the \acrshort{prh} setting, and for each pair we compute raw and calibrated linear \gls{cka}.
At fixed $n=1024$, the permutation threshold $\tau$ (the magnitude of the width correction) increases with the total representation width $d_X+d_Y$, with Pearson correlation $r=0.88$ (\Cref{fig:width-h1-real}).
Since $n$ is held constant, this dependence is attributable to width alone: calibration adapts its correction to representation width rather than applying a fixed offset. Together with \Cref{prop:width_h1}, this confirms that calibration corrects the width confounder on real networks with genuine similarity.

\begin{figure}[htbp]
  \centering
  \includegraphics[width=0.33\linewidth]{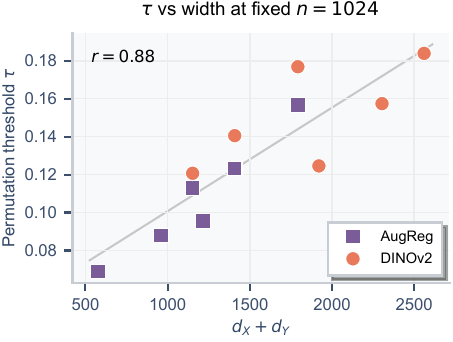}
  \caption{
  \textbf{The width correction scales with representation width (real networks).}
  For DINOv2 and AugReg ViT pairs on the same $n=1024$ WIT images, the permutation threshold $\tau$ (the width correction) grows with total width $d_X+d_Y$ (Pearson $r=0.88$). With $n$ fixed, the increasing correction is driven by width alone.
  }
  \label{fig:width-h1-real}
\end{figure}

\subsection{Permutation budget analysis}
\label{app:perm-budget}

Permutation-based calibration introduces a computational-statistical tradeoff: more permutations yield more stable threshold estimates but increase runtime.
Practitioners need guidance on the minimum budget required for reliable inference.

We analyze the stability of threshold estimates $\tau_\alpha$ and calibrated scores as a function of the permutation budget $K$ across 50 random seeds.
\Cref{fig:perm-budget} shows two panels: the left panel displays threshold estimates, while the right panel shows calibrated scores under $H_0$.
Threshold estimates (left) stabilize rapidly, reaching stable values by approximately $K = 50$ for all metrics tested.
Calibrated scores (right) exhibit more variability at very low budgets ($K < 50$), with occasional spikes due to unstable threshold estimation, but converge to near-zero by $K \approx 100$--$200$.

Based on these results, we recommend $K \geq 200$.
The computational cost scales linearly with $K$, so this recommendation represents a favorable tradeoff between precision and efficiency.

\begin{figure}[htbp]
  \centering
  \includegraphics[width=0.8\linewidth]{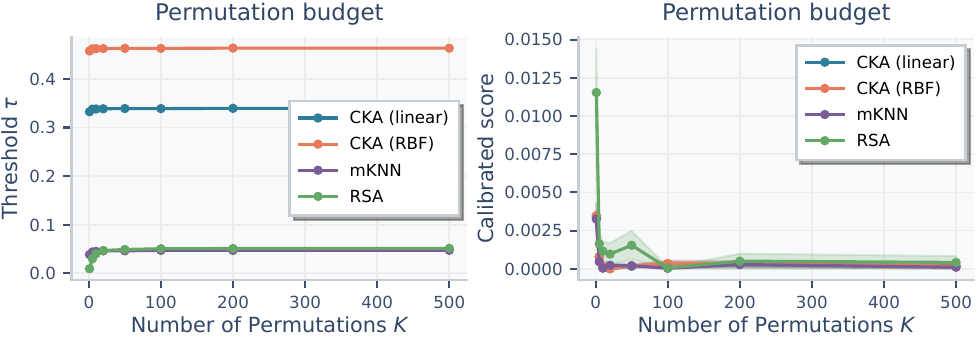}
  \caption{
  \textbf{Permutation budget analysis.}
  Left: threshold $\tau_\alpha$ stabilizes by $K \approx 50$.
  Right: calibrated scores under $H_0$ converge to near-zero by $K \approx 100$--$200$.
  Shaded regions show variability across random seeds.
  }
  \label{fig:perm-budget}
\end{figure}

\subsection{Full null drift results}
\label{app:null-drift-full}

The main text presents null drift results for a representative subset of metrics under Gaussian noise.
Here, we present additional results across all metrics evaluated in this work, including RSA, the RV coefficient, and Procrustes distance, as well as results under heavy-tailed noise distributions.

\Cref{fig:null-drift-full-gaussian} presents results under Gaussian noise for all metrics.
The severity and shape of the null baseline vary substantially across metric families: among the spectral metrics, \gls{cka} variants and the RV coefficient show the strongest monotonic inflation with width, whereas the \gls{cca} family follows a distinct rank-deficiency pattern (mean \gls{cca} peaks near $d \approx n$ then decays, \gls{pwcca} saturates, and \gls{svcca} stays width-insensitive); the neighborhood metrics show the mildest drift.
This reflects the structural sensitivity of the metrics to high-dimensional spurious correlations.
\Gls{rsa} is the exception: as a self-normalized correlation of within-space dissimilarities, its null stays near zero and shows no width-driven drift.
Critically, calibration eliminates drift across all metrics, collapsing scores to zero regardless of the raw bias magnitude.

\Cref{fig:null-drift-full-heavy} extends these results to heavy-tailed noise (Student-$t$, $\nu=3$).
The qualitative pattern is preserved: the confounded metrics exhibit positive drift under the null, and calibration eliminates this drift.
The magnitude of raw bias under heavy-tailed noise is comparable to the Gaussian case (marginally lower for the spectral metrics), and calibration adapts automatically without requiring distributional knowledge.

\begin{figure}[htbp]
  \centering
  \includegraphics[width=1.0\linewidth]{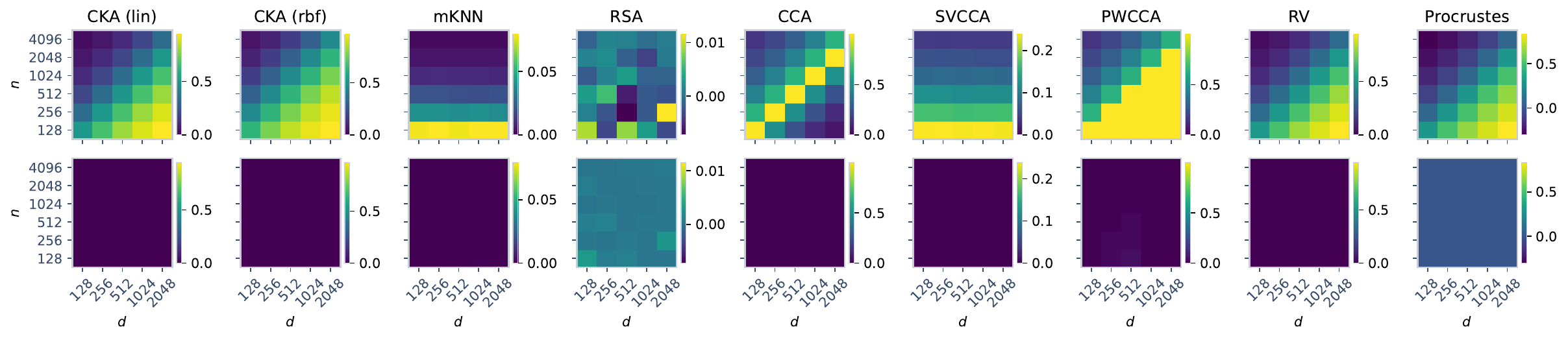}
  \caption{
  \textbf{Full null drift results (Gaussian).}
  Raw scores (top) exhibit systematic positive bias; calibrated scores (bottom) collapse to zero.
  }
  \label{fig:null-drift-full-gaussian}
\end{figure}

\begin{figure}[htbp]
  \centering
  \includegraphics[width=1.0\linewidth]{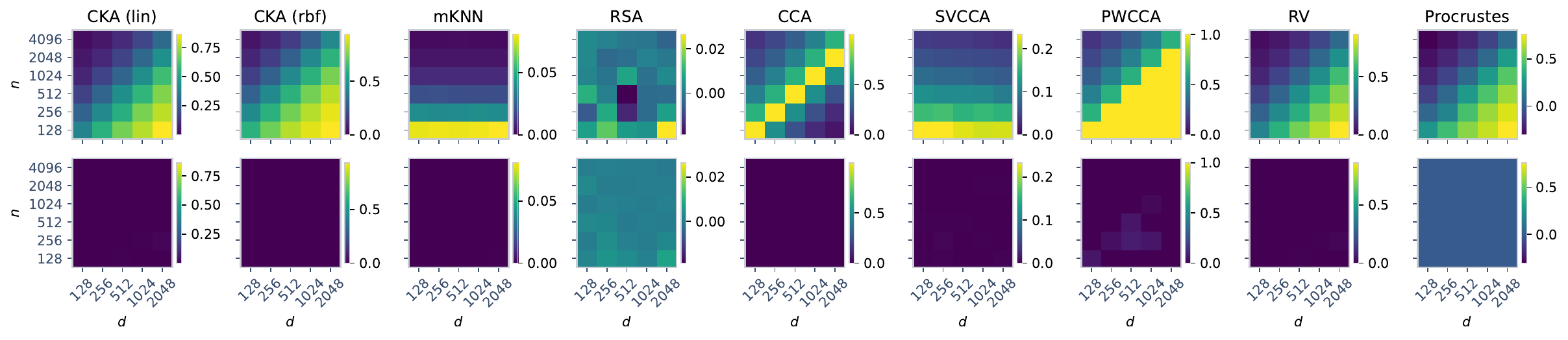}
  \caption{
  \textbf{Full null drift results (heavy-tailed).}
  Student-$t$ ($\nu=3$) noise.
  The pattern is consistent across all metrics: calibration eliminates spurious similarity regardless of noise distribution.
  }
  \label{fig:null-drift-full-heavy}
\end{figure}

\paragraph{Robustness to the generative process.}
The experiments above use i.i.d.\ Gaussian and heavy-tailed noise. To confirm that the width confounder is not tied to these specific generators, we also produce representations as $\vect{x}_i=f_x(\vect{a}_i)$ and $\vect{y}_i=f_y(\vect{a}_i)$, where $\vect{a}_i$ is a shared Gaussian input and $f_x,f_y$ are \emph{independent} random linear maps or MLPs, so the two representations share inputs but no systematic structure. \Cref{fig:width-confounder-generative} shows that raw \gls{cka} still inflates with dimensionality $d$ in every regime, while calibrated \gls{cka} stays at zero. Sharing inputs through unrelated networks thus produces width-driven inflation rather than genuine convergence, and calibration removes it regardless of how the representations are generated.

\begin{figure}[htbp]
  \centering
  \includegraphics[width=0.8\linewidth]{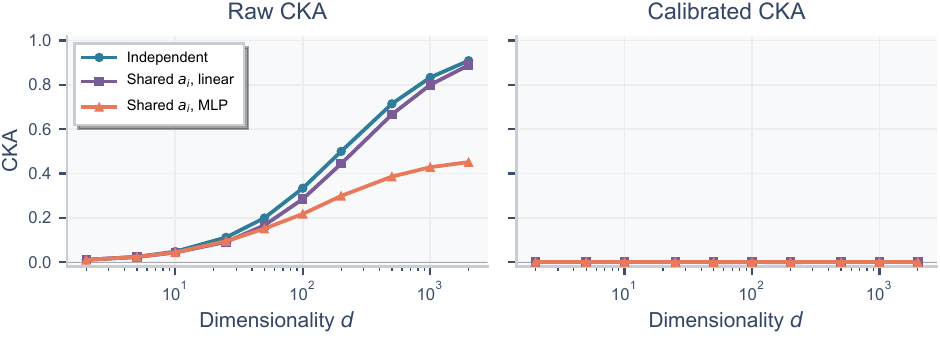}
  \caption{
  \textbf{The width confounder is robust to the generative process.}
  Even when $\mat{X}$ and $\mat{Y}$ are computed from the \emph{same} inputs by \emph{unrelated} random networks (linear or MLP), raw \gls{cka} (left) inflates with dimensionality $d$ although no systematic alignment is present; calibrated \gls{cka} (right) stays at zero. The inflation is a finite-width artifact, not evidence of convergence.
  }
  \label{fig:width-confounder-generative}
\end{figure}

\subsection{Extended \acrshort{prh} alignment results (image--text)}
\label{app:prh-full}

The main text establishes a divergence between local and global similarity metrics when applied to the \gls{prh}: neighborhood-based metrics retain significant cross-modal alignment after calibration, while spectral metrics lose their apparent convergence trend.
A natural question is whether this finding is robust across model families and metric variants.

Here we present comprehensive results across all five vision model families in the \gls{prh} setting (DINOv2, CLIP, ImageNet-21K, MAE, and CLIP-finetuned) and a broad range of metrics spanning the local-to-global spectrum (\Cref{fig:prh-alignment-full,fig:prh-alignment-additional,fig:prh-alignment-expensive}).

The results reinforce and extend the main text findings.
Neighborhood metrics (mKNN, cycle-$k$NN, CKNNA) show a consistent alignment trend across all vision families with a neighborhood size of 10.
This pattern holds for both self-supervised (DINOv2, MAE) and supervised (ImageNet-21K) pretraining objectives, as well as for both CLIP-aligned and CLIP-finetuned variants.
Spectral metrics (CKA linear, CKA RBF, unbiased CKA, RV coefficient, \gls{svcca}) show a different pattern: raw scores suggest increasing alignment with model scale, but calibrated scores show no such scaling trend.

\paragraph{Trend with model scale, before versus after calibration.}
For each metric separately, we measure the Pearson correlation between language-model capability (the model ranking of \citet{pmlr-v235-huh24a}) and the similarity score across all model pairs, before and after calibration (\Cref{app:trend-correlation}). Global metrics lose most of their trend with scale after calibration, while local (neighborhood) metrics retain it across neighborhood sizes.

\begin{table}[ht]
\centering

\caption{\textbf{Trend of each similarity metric with model scale, before versus after calibration.} Pearson correlation between language-model capability and the similarity score in the \gls{prh} setting. The costly spectral and geometric metrics (SVCCA, PWCCA, Procrustes, RV) are evaluated on the reduced DINOv2 subset ($12$ language models $\times\ 4$ vision models); the remaining metrics use the full grid ($204$ pairs). RV coincides with linear \gls{cka} under centering and is listed only for completeness. Global metrics lose their trend with scale after calibration; local (neighborhood) metrics retain it.}
\label{app:trend-correlation}
\small
\begin{tabular}{lcc}
\toprule
\textbf{Metric} & \textbf{Before} & \textbf{After} \\
\midrule
\multicolumn{3}{l}{\textit{Global metrics (spectral and geometric)}} \\
CKA (linear)     & $0.86$ & $0.45$ \\
CKA (RBF)        & $0.83$ & $0.35$ \\
Unbiased CKA     & $0.65$ & $-0.01$ \\
Procrustes       & $0.89$ & $0.39$ \\
RV coefficient   & $0.92$ & $0.31$ \\
\midrule
\multicolumn{3}{l}{\textit{Local metrics (neighborhood)}} \\
mKNN ($k{=}10$)   & $0.85$ & $0.84$ \\
mKNN ($k{=}20$)   & $0.84$ & $0.84$ \\
mKNN ($k{=}50$)   & $0.83$ & $0.82$ \\
mKNN ($k{=}100$)  & $0.79$ & $0.79$ \\
CKNNA ($k{=}10$)  & $0.87$ & $0.87$ \\
CKNNA ($k{=}20$)  & $0.86$ & $0.86$ \\
CKNNA ($k{=}50$)  & $0.86$ & $0.86$ \\
CKNNA ($k{=}100$) & $0.84$ & $0.84$ \\
cycle-$k$NN ($k{=}10$)  & $0.85$ & $0.85$ \\
cycle-$k$NN ($k{=}20$)  & $0.85$ & $0.85$ \\
cycle-$k$NN ($k{=}50$)  & $0.84$ & $0.85$ \\
cycle-$k$NN ($k{=}100$) & $0.77$ & $0.74$ \\
\bottomrule
\end{tabular}
\end{table}

\begin{figure}[htbp]
  \centering
  \begin{subfigure}[t]{\linewidth}
    \includegraphics[width=\linewidth]{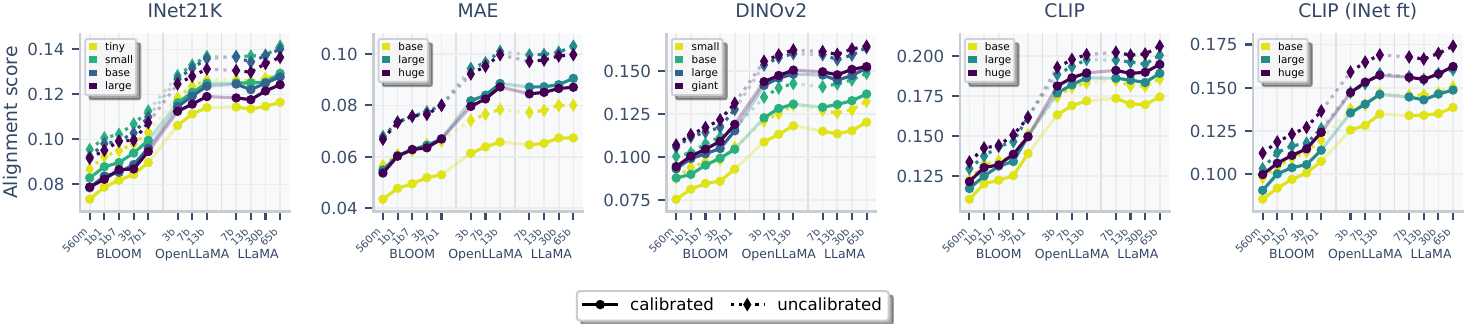}
    \caption{\gls{mknn}: Neighborhood overlap.}
  \end{subfigure}
  \begin{subfigure}[t]{\linewidth}
    \includegraphics[width=\linewidth]{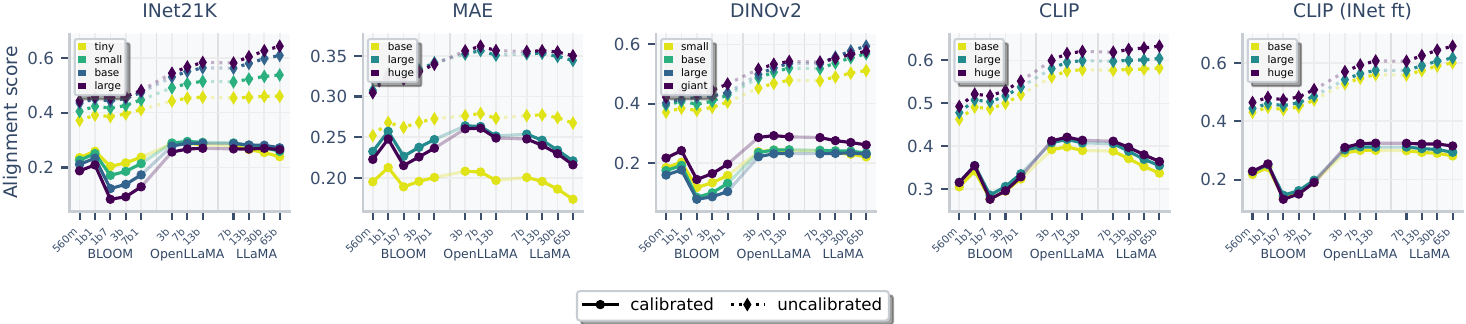}
    \caption{CKA RBF: Spectral alignment.}
  \end{subfigure}
  \begin{subfigure}[t]{\linewidth}
    \includegraphics[width=\linewidth]{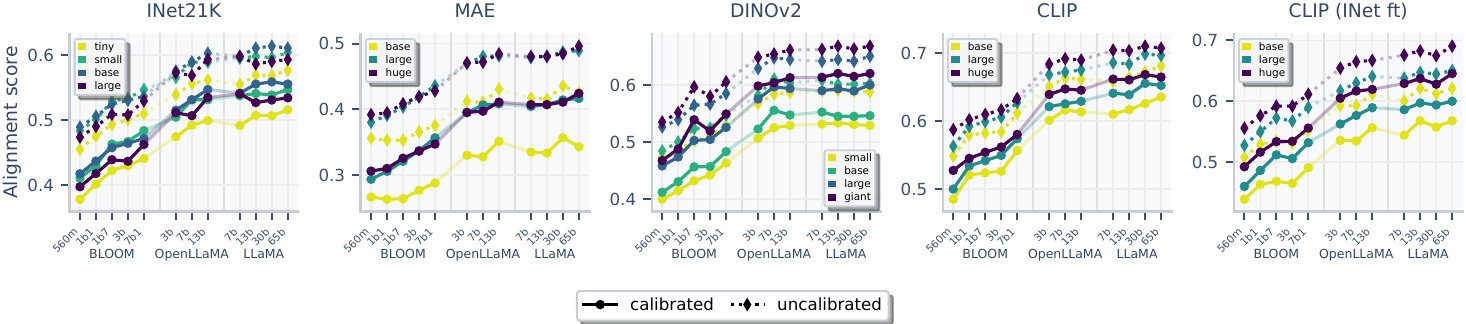}
    \caption{cycle-$k$NN: Bidirectional consistency.}
  \end{subfigure}
  \begin{subfigure}[t]{\linewidth}
    \includegraphics[width=\linewidth]{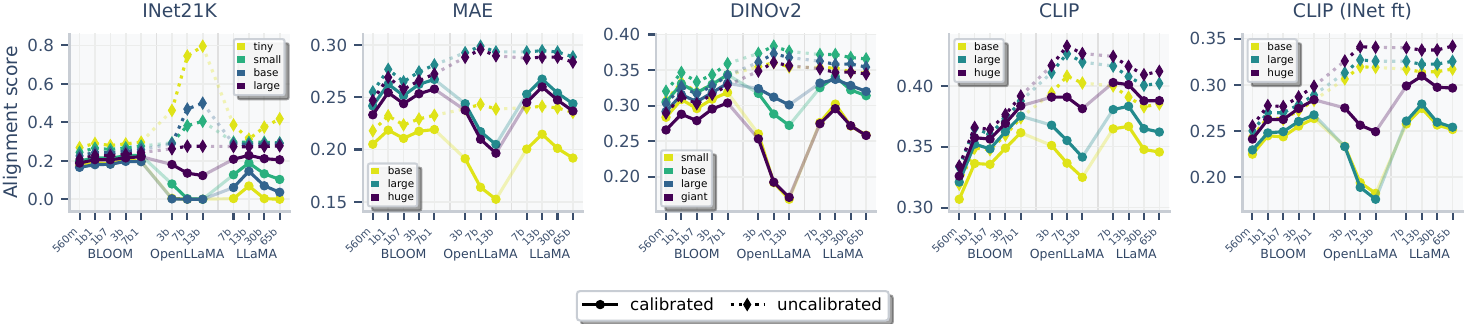}
    \caption{Unbiased CKA.}
  \end{subfigure}
  \caption{
  \textbf{\Acrshort*{prh} alignment results (all vision families).}
  All five vision model families are shown (DINOv2, CLIP, ImageNet-21K, MAE, CLIP-finetuned).
  The divergence between local and global metrics is consistent across all families.
  }
  \label{fig:prh-alignment-full}
\end{figure}

\begin{figure}[htbp]
  \centering
  \begin{subfigure}[t]{\linewidth}
    \includegraphics[width=\linewidth]{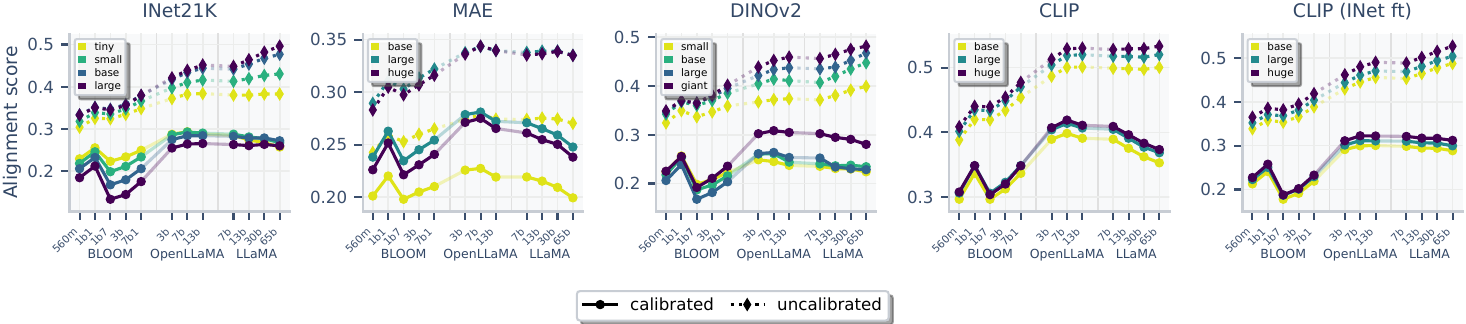}
    \caption{CKA linear.}
  \end{subfigure}
  \begin{subfigure}[t]{\linewidth}
    \includegraphics[width=\linewidth]{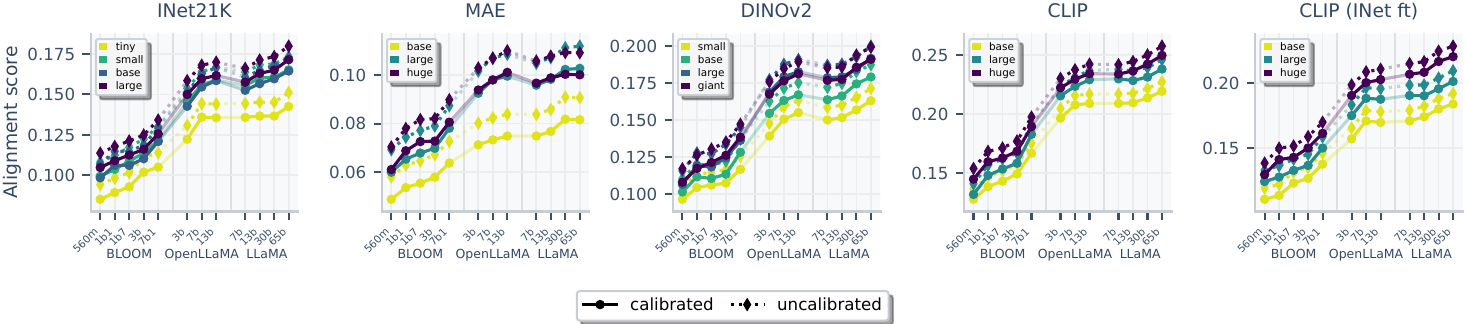}
    \caption{CKNNA.}
  \end{subfigure}
  \caption{
  \textbf{Additional \acrshort{prh} metrics (all vision families).}
  CKA linear (a) shows the same loss of convergence trend as CKA RBF.
  CKNNA (b) shows consistent local alignment across all vision families.
  }
  \label{fig:prh-alignment-additional}
\end{figure}

\begin{figure}[htbp]
  \centering
  \begin{subfigure}[t]{0.32\linewidth}
    \includegraphics[width=\linewidth]{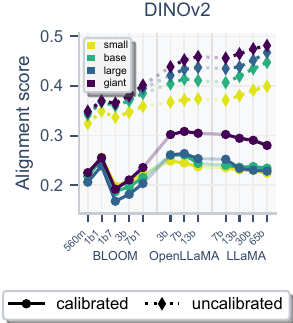}
    \caption{RV coefficient.}
  \end{subfigure}
  \begin{subfigure}[t]{0.32\linewidth}
    \includegraphics[width=\linewidth]{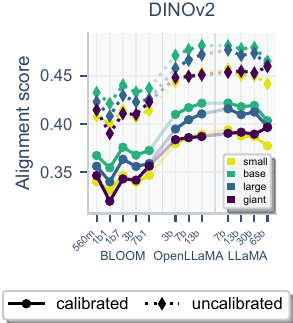}
    \caption{\Gls{svcca}.}
  \end{subfigure}
  \begin{subfigure}[t]{0.32\linewidth}
    \includegraphics[width=\linewidth]{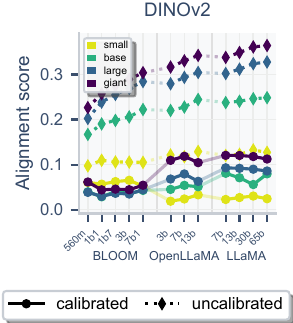}
    \caption{Procrustes.}
  \end{subfigure}
  \caption{
  \textbf{Additional geometric and spectral metrics (DINOv2).}
  The RV coefficient (a), \gls{svcca} (b), and Procrustes distance (c) confirm the same pattern as CKA: calibrated scores show no convergence trend with model scale, so the disappearance of global convergence is not specific to \gls{cka} but extends to shape- and geometry-based metrics.
  Due to the computational cost of permutation-based calibration for these metrics, we report results for DINOv2 only.
  }
  \label{fig:prh-alignment-expensive}
\end{figure}

\paragraph{Continuous language-performance axis.}
The preceding figures place each language model at its performance \emph{rank}.
Because \citet{pmlr-v235-huh24a} instead plot cross-modal alignment against a \emph{continuous} measure of language performance, we reproduce all of the above results on that axis for direct comparison.
We quantify language performance by bits-per-byte (BPB) over OpenWebText, as in \citet{pmlr-v235-huh24a}, oriented so that larger values, $\max(\mathrm{BPB}) - \mathrm{BPB}$, indicate stronger next-token prediction, and place each model at its measured value with a per-encoder least-squares fit (\Cref{fig:prh-alignment-full-langperf,fig:prh-alignment-additional-langperf,fig:prh-alignment-expensive-langperf}).
The local--global divergence is unchanged on this axis: after calibration, neighborhood metrics retain their alignment trend with language performance while spectral metrics do not, so the finding does not depend on whether capability is expressed as a rank or as a continuous performance value.

\begin{figure}[htbp]
  \centering
  
  \begin{subfigure}[t]{\linewidth}
    \includegraphics[width=\linewidth]{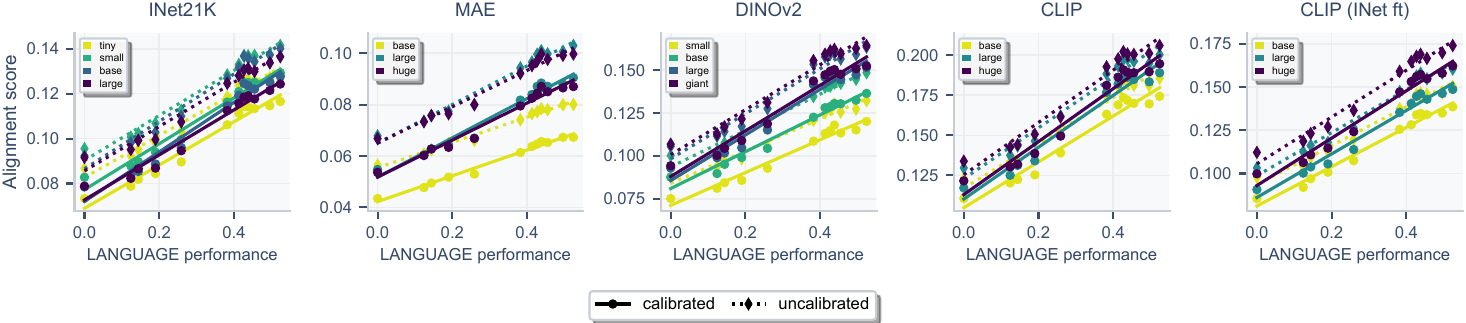}
    \caption{\gls{mknn}: Neighborhood overlap.}
  \end{subfigure}
  \begin{subfigure}[t]{\linewidth}
    \includegraphics[width=\linewidth]{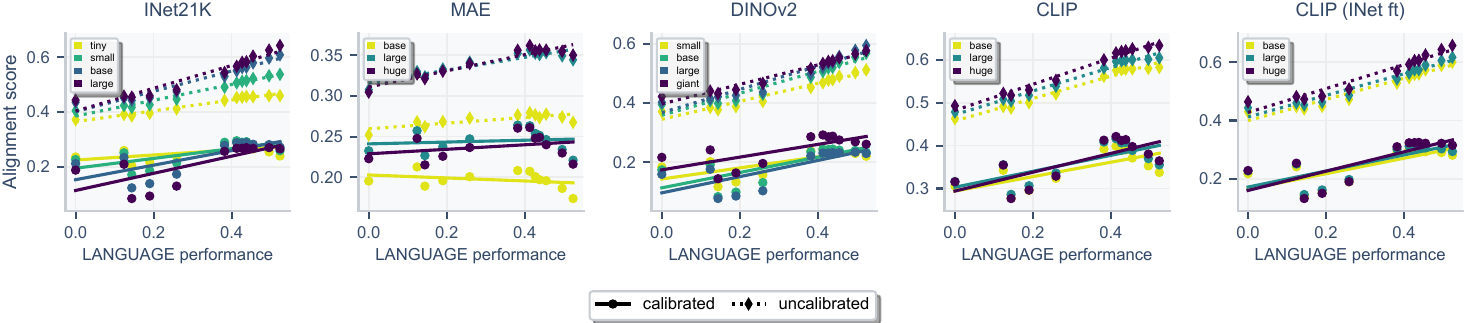}
    \caption{CKA RBF: Spectral alignment.}
  \end{subfigure}
  \begin{subfigure}[t]{\linewidth}
    \includegraphics[width=\linewidth]{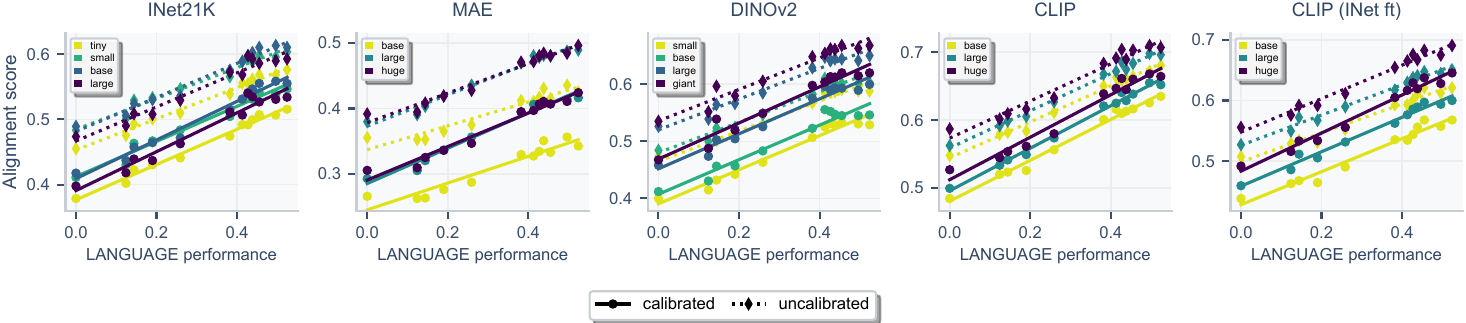}
    \caption{cycle-$k$NN: Bidirectional consistency.}
  \end{subfigure}
  \begin{subfigure}[t]{\linewidth}
    \includegraphics[width=\linewidth]{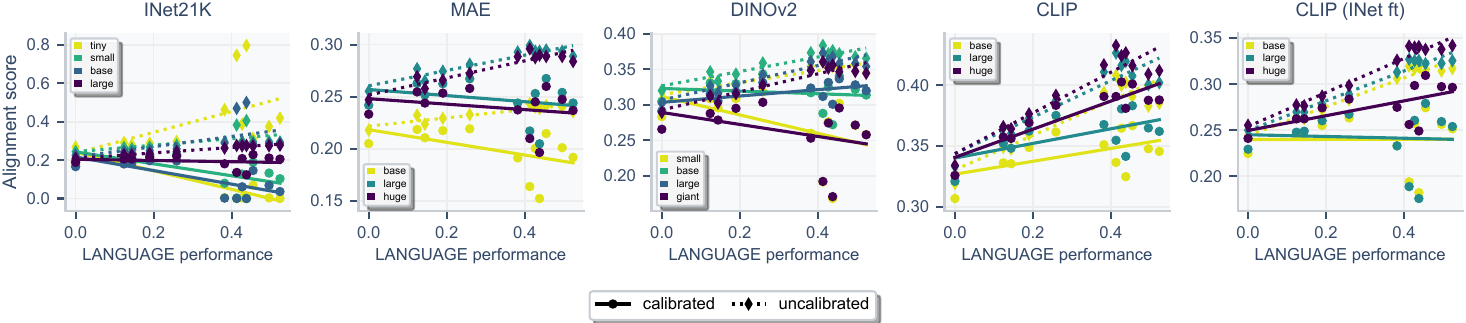}
    \caption{Unbiased CKA.}
  \end{subfigure}
  \caption{
  \textbf{\Acrshort*{prh} alignment versus language performance (all vision families).}
  As \Cref{fig:prh-alignment-full}, but with the horizontal axis given by continuous language performance ($\max(\mathrm{BPB}) - \mathrm{BPB}$ over OpenWebText) and a per-encoder least-squares fit.
  The local--global divergence is unchanged from the rank-based axis.
  }
  \label{fig:prh-alignment-full-langperf}
\end{figure}

\begin{figure}[htbp]
  \centering
  
  \begin{subfigure}[t]{\linewidth}
    \includegraphics[width=\linewidth]{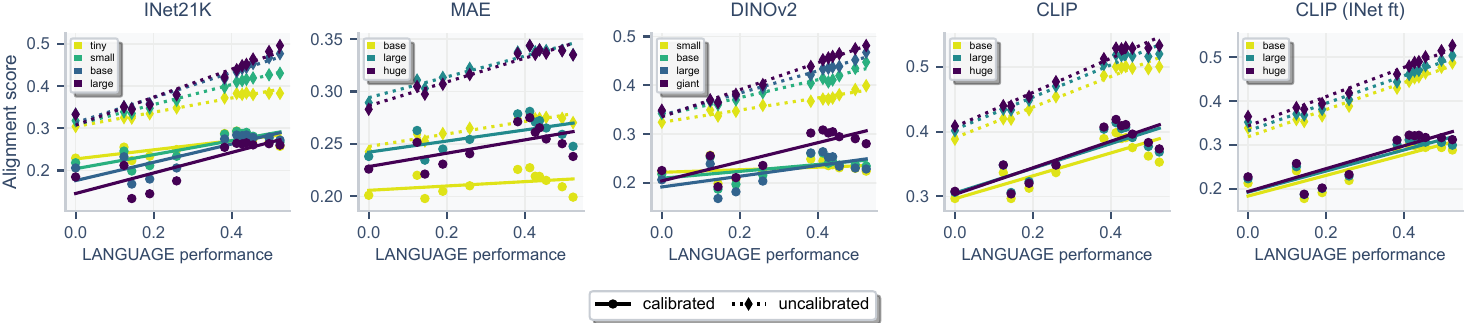}
    \caption{CKA linear.}
  \end{subfigure}
  \begin{subfigure}[t]{\linewidth}
    \includegraphics[width=\linewidth]{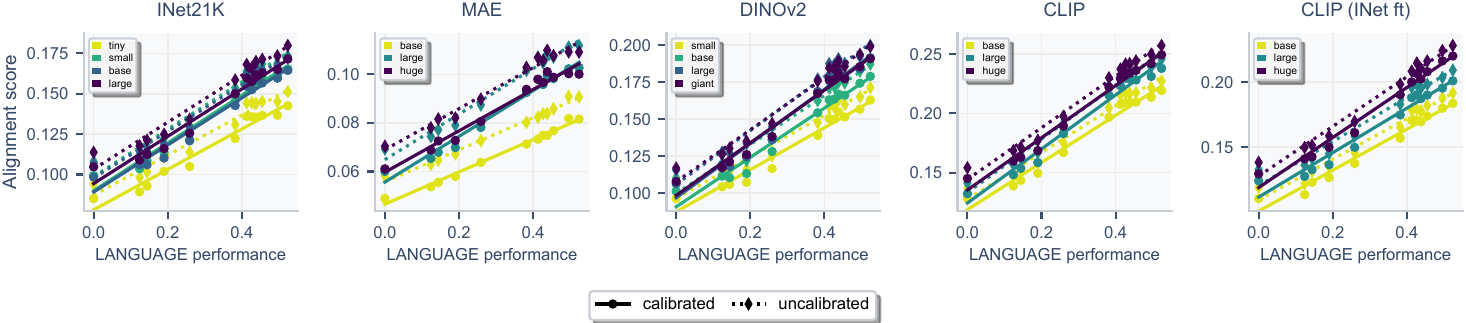}
    \caption{CKNNA.}
  \end{subfigure}
  \caption{
  \textbf{Additional \acrshort{prh} metrics versus language performance (all vision families).}
  As \Cref{fig:prh-alignment-additional}, with the continuous language-performance axis.
  CKA linear (a) loses its trend after calibration; CKNNA (b) retains local alignment across all vision families.
  }
  \label{fig:prh-alignment-additional-langperf}
\end{figure}

\begin{figure}[htbp]
  \centering
  
  \begin{subfigure}[t]{0.32\linewidth}
    \includegraphics[width=\linewidth]{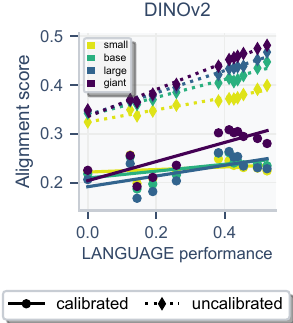}
    \caption{RV coefficient.}
  \end{subfigure}
  \begin{subfigure}[t]{0.32\linewidth}
    \includegraphics[width=\linewidth]{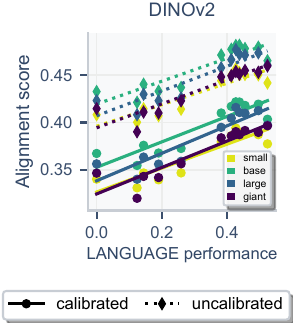}
    \caption{\Gls{svcca}.}
  \end{subfigure}
  \begin{subfigure}[t]{0.32\linewidth}
    \includegraphics[width=\linewidth]{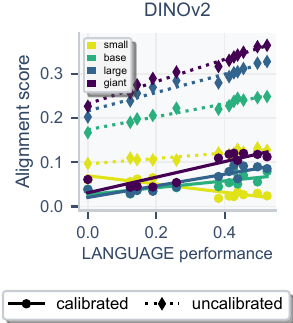}
    \caption{Procrustes.}
  \end{subfigure}
  \caption{
  \textbf{Additional geometric and spectral metrics versus language performance (DINOv2).}
  As \Cref{fig:prh-alignment-expensive}, with the continuous language-performance axis.
  The RV coefficient (a), \gls{svcca} (b), and Procrustes (c) show no calibrated convergence trend with language performance, matching the rank-based axis.
  }
  \label{fig:prh-alignment-expensive-langperf}
\end{figure}

\paragraph{Statistical significance.}
Beyond calibrated scores, we report permutation $p$-values to quantify statistical evidence against the null hypothesis of no cross-modal alignment (\Cref{fig:prh-pvalues}).
All 204 vision--language model pairs are significant at $p < 0.05$, with most achieving $p \approx 0.002$ (the minimum attainable with $K=500$ permutations) for both local and global metrics.
This confirms that cross-modal similarity is statistically significant (\textit{i.e.}, has some alignment) across all model pairs.
The critical distinction between local and global metrics lies not in statistical significance but in the magnitude and trends of calibrated scores.
Local metrics show substantial alignment above the null threshold that persists across scales, whereas global metrics, although significant, show no convergence in calibrated effect sizes.

\begin{figure}[htbp]
  \centering
  \begin{subfigure}[t]{\linewidth}
    \includegraphics[width=\linewidth]{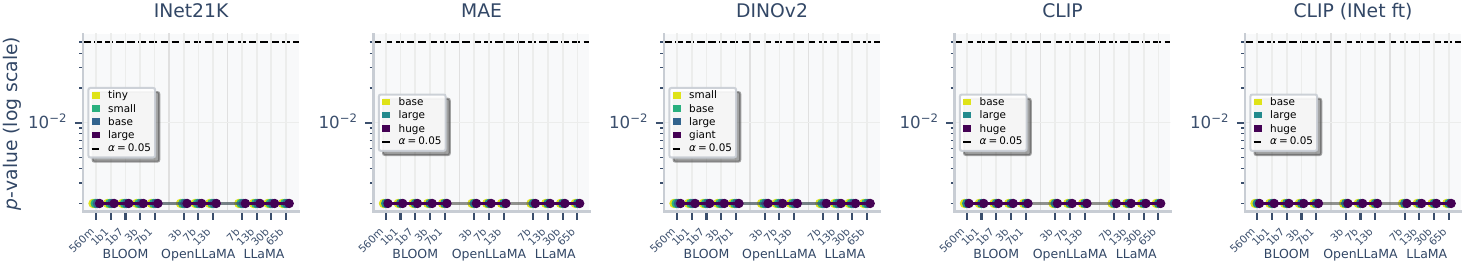}
    \caption{\gls{mknn} ($k=10$).}
  \end{subfigure}
  \begin{subfigure}[t]{\linewidth}
    \includegraphics[width=\linewidth]{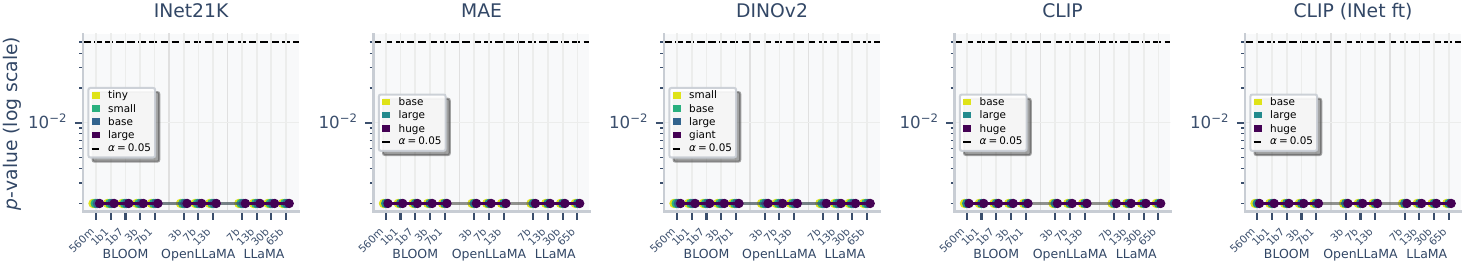}
    \caption{CKA linear.}
  \end{subfigure}
  \caption{
  \textbf{Permutation $p$-values for \acrshort*{prh} alignment.}
  All model pairs are significant at $p < 0.05$, with most achieving $p < 0.005$ for both local (a) and global (b) metrics.
  The difference between metric families lies in calibrated effect sizes, not significance.
  }
  \label{fig:prh-pvalues}
\end{figure}

\subsection{Extended video--language alignment results}
\label{app:video-language}

The main text extends the \gls{prh} analysis to video--language alignment following \citet{zhu2025dynamic}.
Here, we provide additional results to verify that the local-vs-global pattern observed for image--language alignment extends to the video modality.

We use 1024 samples from the PVD \citep{bolya2025perception-encoder,cho2025perceptionlm} test set.
We evaluate video-native models (VideoMAE~\citep{tong2022videomae}) and, as a frame-level baseline, image models (DINOv2 and CLIP) applied to the middle frame of each video, comparing all against the same three language model families used in the image--language experiments (BLOOM, OpenLLaMA, LLaMA) at multiple scales.
For VideoMAE we include both the fine-tuned scale series (small/base/large/huge, fine-tuned on Kinetics) used in the main text and the corresponding non-fine-tuned checkpoints.
\Cref{fig:v2t-full} shows results for spectral (\gls{cka} RBF) and neighborhood (\gls{mknn}, CKNNA) metrics.

The pattern mirrors the image--language findings.
For spectral metrics, raw scores suggest alignment, whereas calibrated scores drop significantly, indicating that much of the apparent alignment is attributable to width and depth confounders.
In contrast, neighborhood metrics retain significant alignment after calibration, confirming that video and language representations share local topological structure.
This local alignment strengthens with the capability of the video encoder: both fine-tuning on Kinetics and increasing scale raise the calibrated neighborhood alignment, in line with the Aristotelian hypothesis that more capable representations converge more in local structure. Calibration removes the spectral inflation throughout.

\begin{figure}[htbp]
  \centering
  \begin{subfigure}[t]{\linewidth}
    \centering
    \includegraphics[width=0.7\linewidth]{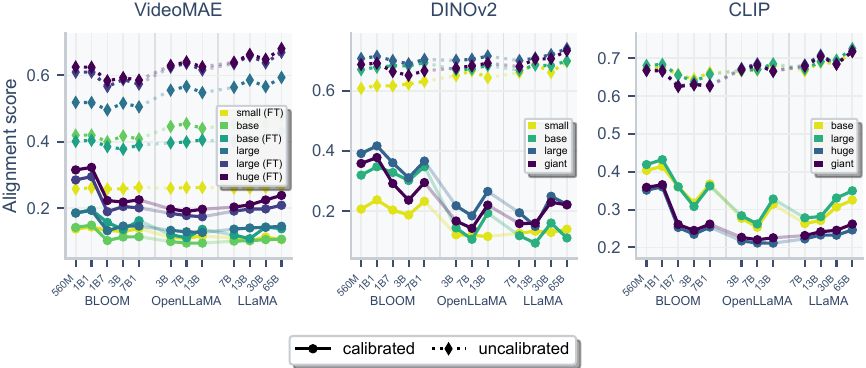}
    \caption{\Gls{cka} RBF (spectral).}
  \end{subfigure}

  \begin{subfigure}[t]{\linewidth}
    \centering
    \includegraphics[width=0.7\linewidth]{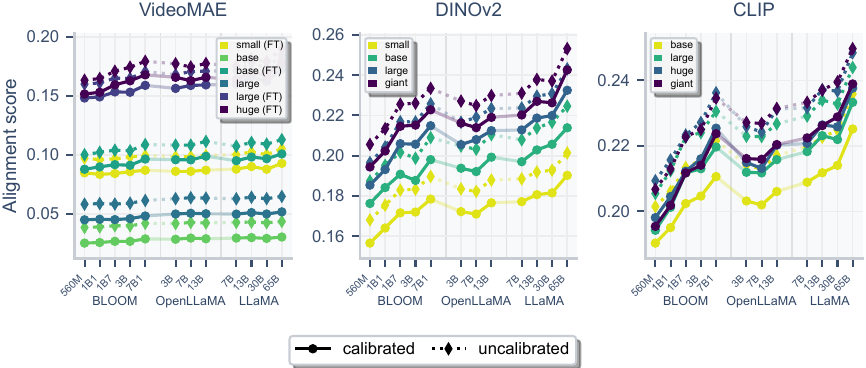}
    \caption{\gls{mknn} ($k=10$, neighborhood).}
  \end{subfigure}

  \begin{subfigure}[t]{\linewidth}
    \centering
    \includegraphics[width=0.7\linewidth]{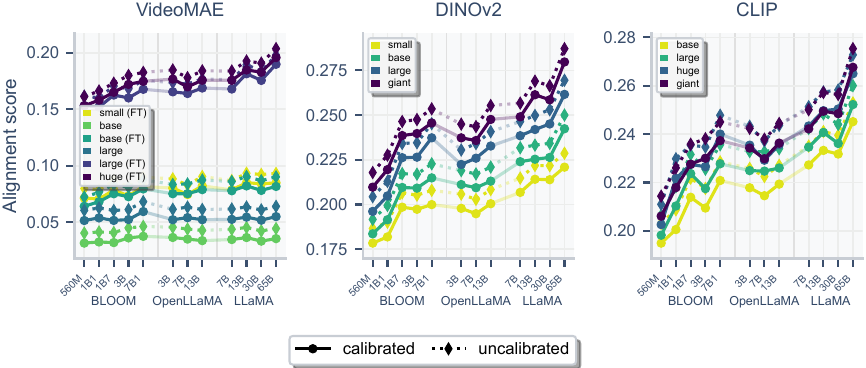}
    \caption{CKNNA ($k=10$, neighborhood).}
  \end{subfigure}
  \caption{
  \textbf{Extended video--language alignment.}
  Each panel shows VideoMAE (fine-tuned and non-fine-tuned), DINOv2, and CLIP (frame-level) against the language model families; raw scores are dotted, calibrated solid.
  (a)~Spectral alignment (\gls{cka} RBF) drops after calibration. (b,c)~Neighborhood alignment is retained and strengthens with video-encoder capability.
  }
  \label{fig:v2t-full}
\end{figure}

\subsection{Characterizing the locality of cross-modal alignment}
\label{app:locality-analysis}

The main text establishes that local neighborhood metrics retain significant alignment after calibration, while global spectral metrics do not.
We next characterize \emph{how local} this alignment is.
Both \gls{mknn} and \gls{cka}-RBF have hyperparameters that control their sensitivity to local versus global structure.
By varying these parameters, we can characterize the scale at which cross-modal alignment emerges.

\paragraph{Experimental setup.}
We vary two locality parameters: the neighborhood size $k$ in \gls{mknn}, testing $k \in \{10, 20, 50, 100\}$ where smaller values focus on immediate neighbors and larger values consider broader local structure, and the RBF kernel bandwidth $\sigma$ in \gls{cka}-RBF, testing $\sigma \in \{0.1, 0.5, 2.0, 5.0\}$, which controls the length scale over which the kernel assigns significant weight.

\paragraph{RBF bandwidth.}
The RBF (radial basis function) kernel is defined as $k(\vect{x}, \vect{y}) = \exp\left(-\|\vect{x} - \vect{y}\|^2 / 2\sigma^2\right)$.
The bandwidth $\sigma$ determines the \emph{length scale} of similarity.
When $\sigma$ is small (\textit{e.g.}, $0.1$), the kernel is sharply peaked: only very close points contribute significantly to the Gram matrix, making the similarity measure sensitive to \emph{exact pairwise distances} in the immediate neighborhood.
When $\sigma$ is large (\textit{e.g.}, $5.0$), the kernel is broad: even moderately distant points contribute, and the similarity measure aggregates information over larger neighborhoods, becoming sensitive to coarser geometric structure.

\paragraph{Neighborhood size.}
For \gls{mknn}, the parameter $k$ controls how many nearest neighbors are considered when measuring overlap.
Small $k$ (\textit{e.g.}, $10$) measures agreement on immediate neighbors, \textit{i.e.}, the closest points to each sample, capturing fine-grained local topology.
Large $k$ (\textit{e.g.}, $100$) measures agreement on a broader neighborhood.
With $n = 1000$ samples and $k = 100$, we ask whether the $10\%$ closest points agree across representations.
Crucially, mKNN is a \emph{rank-based} metric: it asks \emph{which} points are neighbors (ordinal information), not \emph{how close} they are (cardinal information).

\paragraph{mKNN across $k$ values.}
\Cref{fig:locality-mknn} shows the PRH alignment results for mKNN with varying $k$.
A consistent pattern emerges: all $k$ values show significant alignment after calibration, with calibrated scores remaining well above zero even at $k = 100$.
However, the scaling trend is most pronounced at small $k$.
For $k = 10$, raw scores show a clear upward trend with model capacity that persists after calibration.
At large $k$, this trend flattens even in raw scores. For $k = 100$, raw scores plateau for larger models, suggesting that broader neighborhood agreement is already saturated across model scales.
This pattern indicates that scaling-driven improvement in alignment is concentrated at the finest topological level.

\begin{figure}[htbp]
  \centering
  \begin{subfigure}[t]{\linewidth}
    \includegraphics[width=\linewidth]{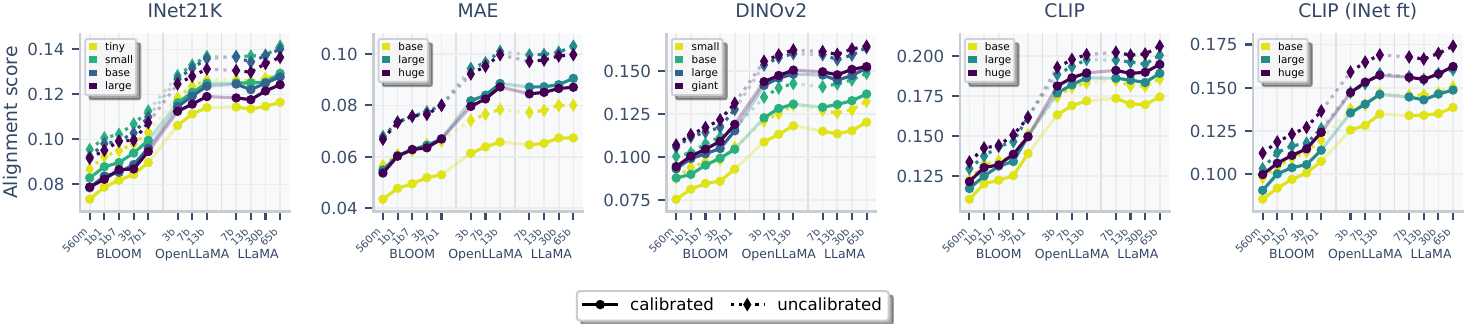}
    \caption{mKNN ($k=10$)}
  \end{subfigure}
  \begin{subfigure}[t]{\linewidth}
    \includegraphics[width=\linewidth]{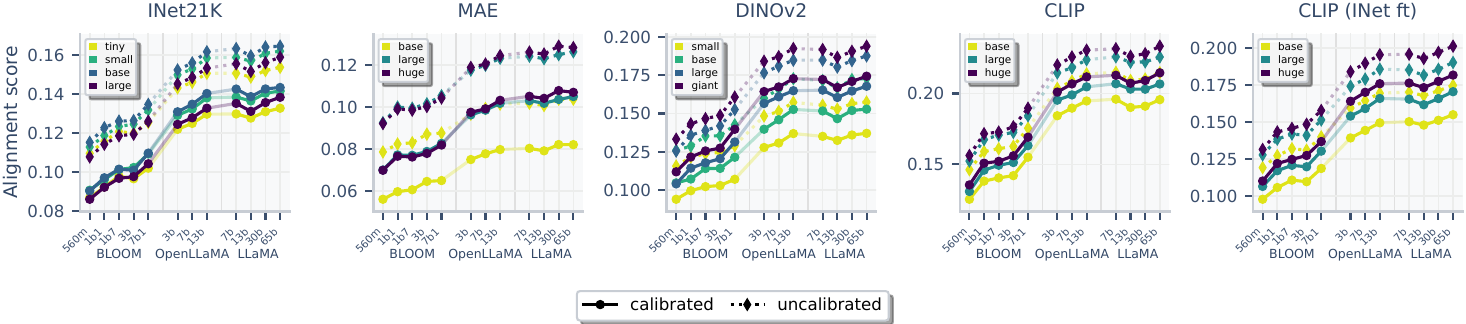}
    \caption{mKNN ($k=20$)}
  \end{subfigure}
  \begin{subfigure}[t]{\linewidth}
    \includegraphics[width=\linewidth]{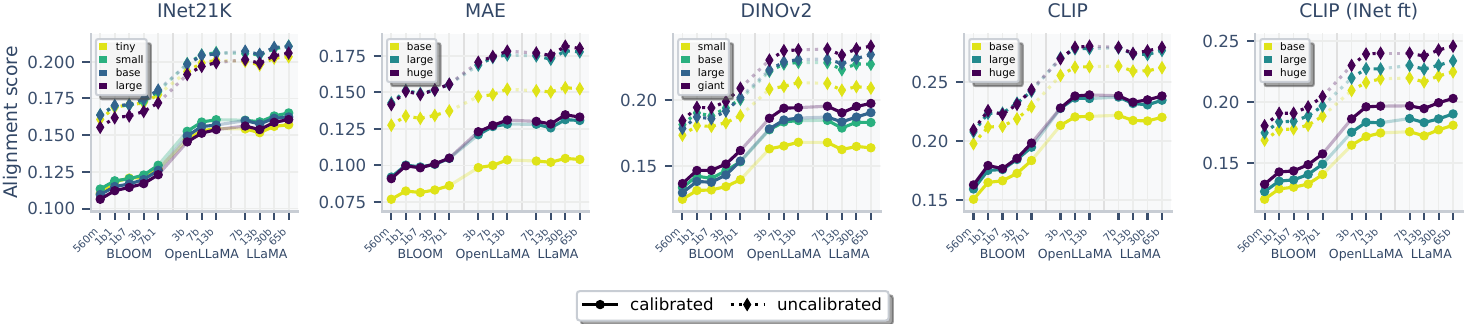}
    \caption{mKNN ($k=50$)}
  \end{subfigure}
  \begin{subfigure}[t]{\linewidth}
    \includegraphics[width=\linewidth]{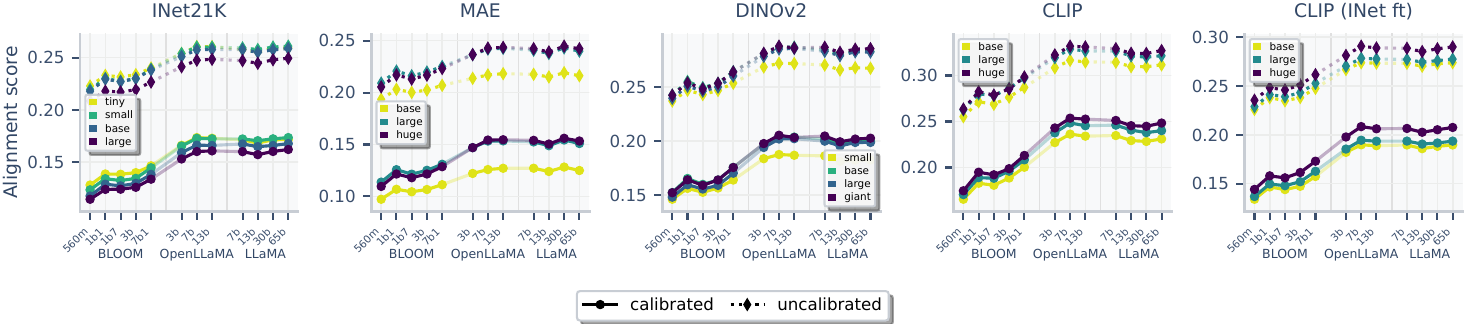}
    \caption{mKNN ($k=100$)}
  \end{subfigure}
  \caption{
  \textbf{PRH alignment with varying neighborhood size $k$ for mKNN.}
  All $k$ values show significant alignment after calibration.
  The scaling trend is clearest at small $k$ and flattens at large $k$, suggesting scaling improvements are concentrated at the finest local scale.
  }
  \label{fig:locality-mknn}
\end{figure}

\paragraph{CKA-RBF across bandwidth values.}
\Cref{fig:locality-cka-rbf}, and the accompanying $p$-values in \Cref{fig:locality-cka-rbf-pvalues}, show results for CKA-RBF with varying bandwidth $\sigma$, revealing a different pattern from mKNN.
At $\sigma = 0.1$ (very local), there is no significant alignment after calibration: raw scores are near $1.0$, reflecting the high similarity of any high-dimensional representations under a sharply peaked kernel. However, calibrated scores collapse to approximately zero with $p$-values exceeding $0.05$ for most model pairs, indicating that the observed similarity is indistinguishable from chance.
At $\sigma = 0.5$, alignment emerges, but with a flattening trend after calibration.
Calibrated scores initially rise with model scale, then plateau and slightly decline for the largest models.
At $\sigma = 2.0$ and $\sigma = 5.0$, significant alignment persists, but the calibrated trend also flattens, resembling the pattern observed for large-$k$ mKNN: alignment exists, but scaling-driven improvement disappears after calibration.

\begin{figure}[htbp]
  \centering
  \begin{subfigure}[t]{\linewidth}
    \includegraphics[width=\linewidth]{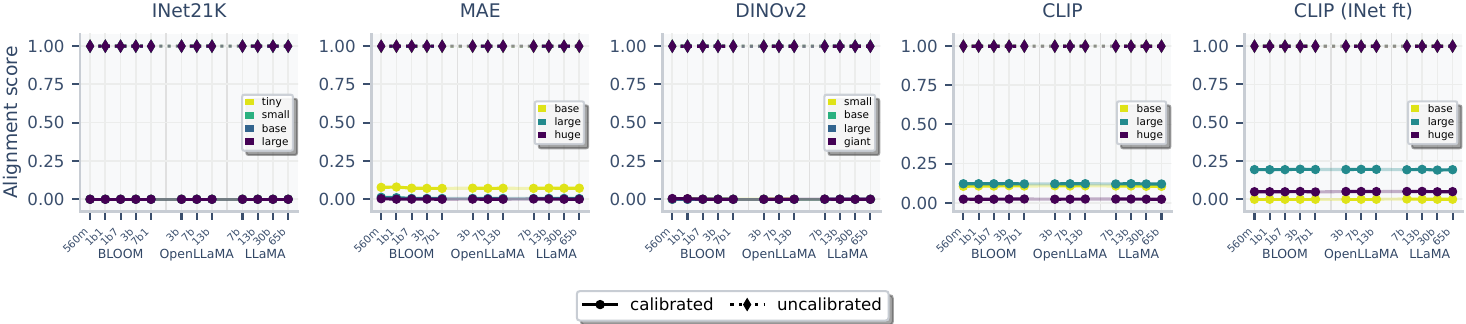}
    \caption{CKA-RBF ($\sigma=0.1$)}
  \end{subfigure}
  \begin{subfigure}[t]{\linewidth}
    \includegraphics[width=\linewidth]{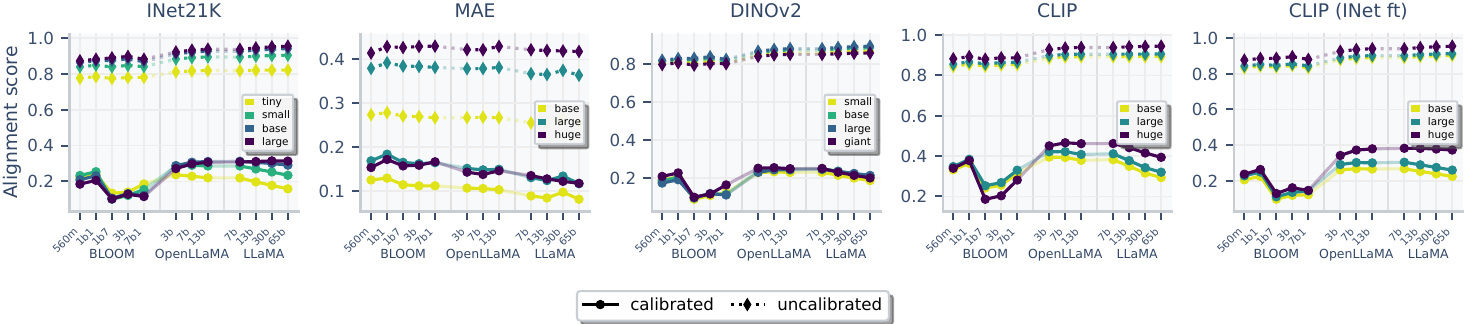}
    \caption{CKA-RBF ($\sigma=0.5$)}
  \end{subfigure}
  \begin{subfigure}[t]{\linewidth}
    \includegraphics[width=\linewidth]{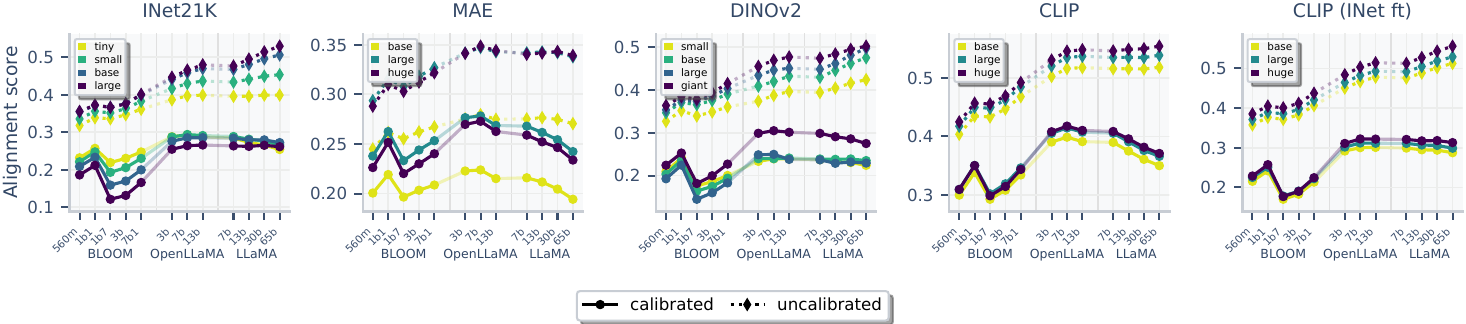}
    \caption{CKA-RBF ($\sigma=2.0$)}
  \end{subfigure}
  \begin{subfigure}[t]{\linewidth}
    \includegraphics[width=\linewidth]{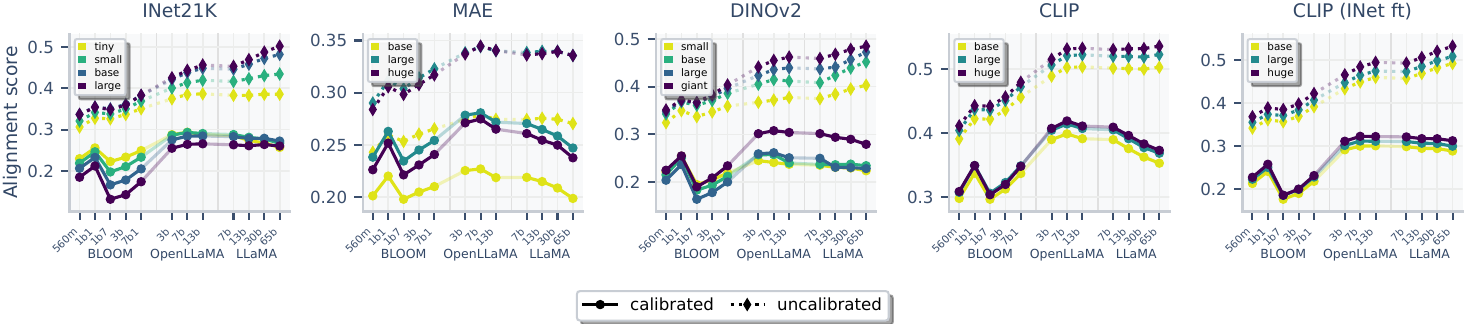}
    \caption{CKA-RBF ($\sigma=5.0$)}
  \end{subfigure}
  \caption{
  \textbf{PRH alignment with varying bandwidth $\sigma$ for CKA-RBF.}
  At very small $\sigma$ (a), no significant alignment remains after calibration.
  Larger $\sigma$ values (b--d) show significant alignment, but the scaling trend flattens after calibration.
  }
  \label{fig:locality-cka-rbf}
\end{figure}

\begin{figure}[htbp]
  \centering
  \begin{subfigure}[t]{\linewidth}
    \includegraphics[width=\linewidth]{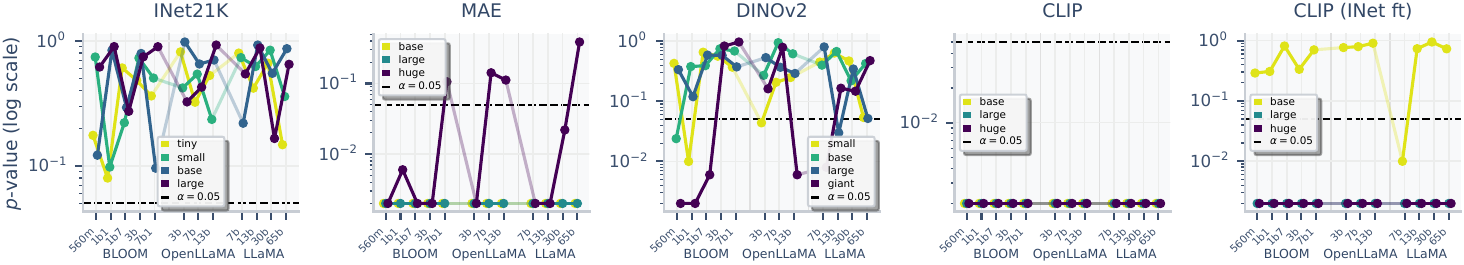}
    \caption{CKA-RBF ($\sigma=0.1$)}
  \end{subfigure}
  \begin{subfigure}[t]{\linewidth}
    \includegraphics[width=\linewidth]{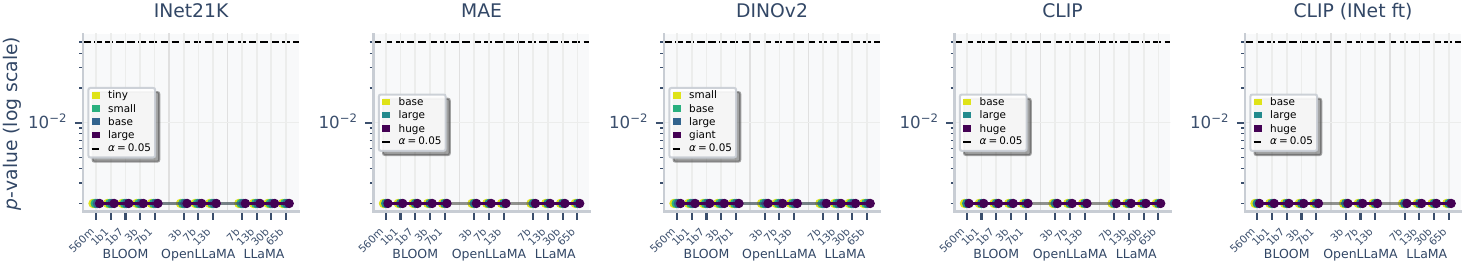}
    \caption{CKA-RBF ($\sigma=0.5$)}
  \end{subfigure}
  \begin{subfigure}[t]{\linewidth}
    \includegraphics[width=\linewidth]{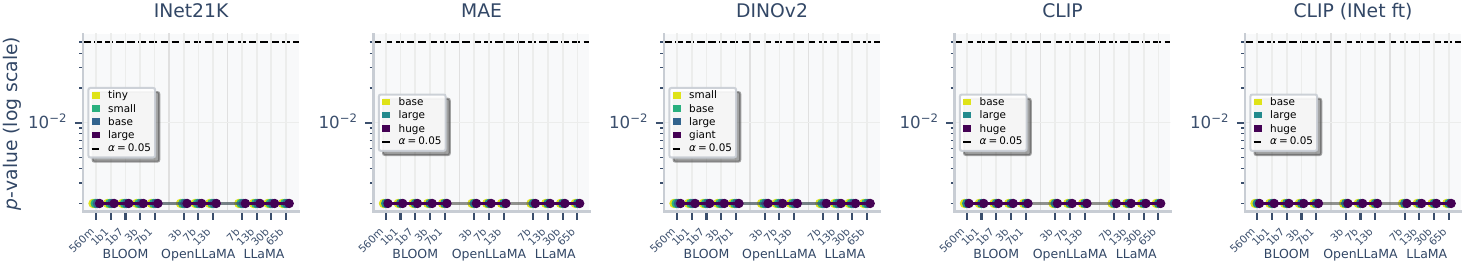}
    \caption{CKA-RBF ($\sigma=2.0$)}
  \end{subfigure}
  \begin{subfigure}[t]{\linewidth}
    \includegraphics[width=\linewidth]{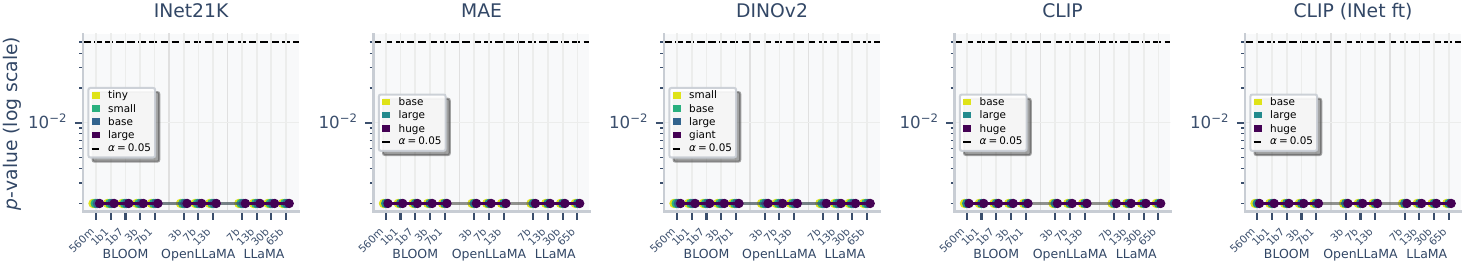}
    \caption{CKA-RBF ($\sigma=5.0$)}
  \end{subfigure}
  \caption{
  \textbf{Significance of PRH alignment with varying bandwidth $\sigma$ for CKA-RBF.}
  Alignment with $\sigma = 0.1$ (a) is not significant for multiple models where larger bandwidths have significance (b--d).
  }
  \label{fig:locality-cka-rbf-pvalues}
\end{figure}

\paragraph{Topological versus metric alignment.}
The contrasting behavior of mKNN and small-$\sigma$ CKA-RBF reveals a fundamental distinction in what ``local alignment'' means.
On one hand, mKNN measures \emph{topological} alignment: do the representations agree on \emph{which} points are neighbors?
This captures ordinal information where the ranking of distances matters but not their absolute values.
On the other hand, small-$\sigma$ CKA-RBF measures \emph{metric} alignment: do the representations agree on \emph{how close} neighbors are?
This captures cardinal information where exact distance values matter.

The fact that mKNN shows alignment at all $k$ values while small-$\sigma$ CKA-RBF shows no alignment reveals that cross-modal representations agree on neighborhood identity (which points are close) but not on exact local distances (how close they are).
This finding is consistent with the observation that different training objectives and architectures induce different distance scales in representation space while preserving the relative ordering of neighbors.
The \emph{Aristotelian} Representation Hypothesis should therefore be understood as convergence to shared \emph{topological} structure rather than shared \emph{metric} structure.


\subsection{Sensitivity to significance level $\alpha$}
\label{app:alpha-sensitivity}

The main text uses a significance level of $\alpha=0.05$ throughout.
To confirm that the \gls{prh} conclusions are not sensitive to this choice, we repeat the \gls{prh} evaluation from \Cref{sec:prh} with $\alpha \in \{0.01, 0.05, 0.10\}$ for representative global (CKA linear, CKA RBF) and local (\gls{mknn} with $k=10$) metrics.

\Cref{fig:alpha-cka-lin,fig:alpha-cka-rbf,fig:alpha-mknn} show that the conclusions are entirely invariant to the choice of $\alpha$.
For global metrics, calibrated scores show no convergence trend at any significance level.
For local metrics, calibrated scores retain their alignment trend across all three $\alpha$ values.
Stricter thresholds ($\alpha=0.01$) produce slightly lower calibrated scores, while more permissive thresholds ($\alpha=0.10$) produce slightly higher ones, but the qualitative pattern is unchanged.
This confirms that our findings are not an artifact of a particular significance level.

\begin{figure}[htbp]
  \centering
  \begin{subfigure}[t]{\linewidth}
    \includegraphics[width=\linewidth]{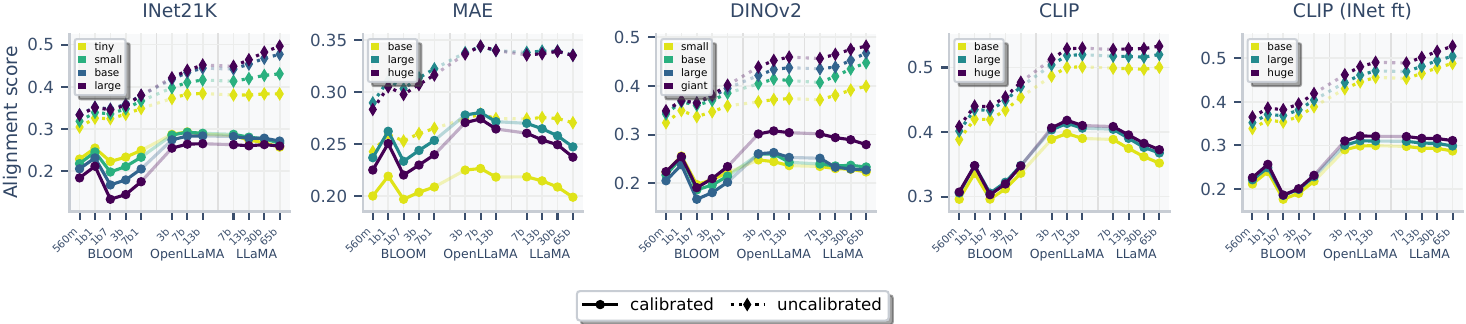}
    \caption{$\alpha = 0.01$}
  \end{subfigure}
  \begin{subfigure}[t]{\linewidth}
    \includegraphics[width=\linewidth]{figures/prh_alignment_lines_by_family_cka_lin.pdf}
    \caption{$\alpha = 0.05$ (default)}
  \end{subfigure}
  \begin{subfigure}[t]{\linewidth}
    \includegraphics[width=\linewidth]{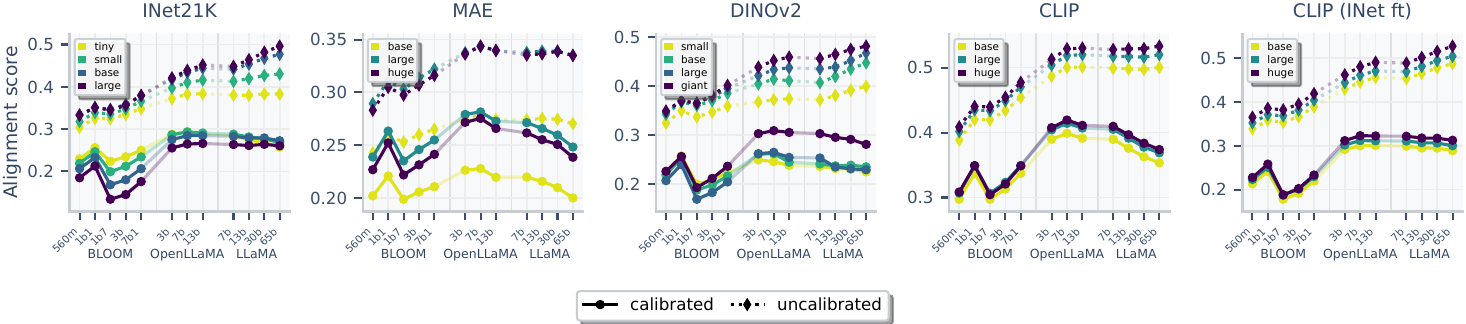}
    \caption{$\alpha = 0.10$}
  \end{subfigure}
  \caption{
  \textbf{Sensitivity to $\alpha$ for CKA linear.}
  Calibrated scores show no convergence trend regardless of significance level.
  }
  \label{fig:alpha-cka-lin}
\end{figure}

\begin{figure}[htbp]
  \centering
  \begin{subfigure}[t]{\linewidth}
    \includegraphics[width=\linewidth]{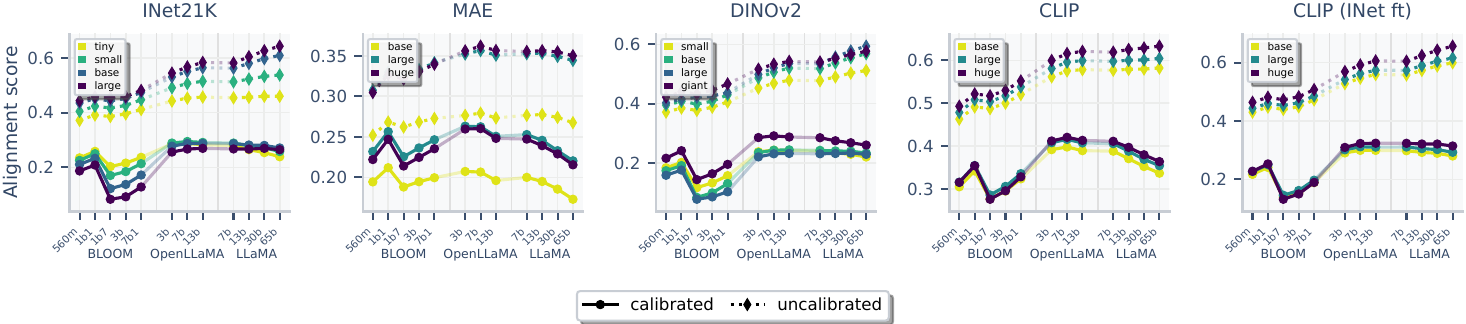}
    \caption{$\alpha = 0.01$}
  \end{subfigure}
  \begin{subfigure}[t]{\linewidth}
    \includegraphics[width=\linewidth]{figures/prh_alignment_lines_by_family_cka_rbf.pdf}
    \caption{$\alpha = 0.05$ (default)}
  \end{subfigure}
  \begin{subfigure}[t]{\linewidth}
    \includegraphics[width=\linewidth]{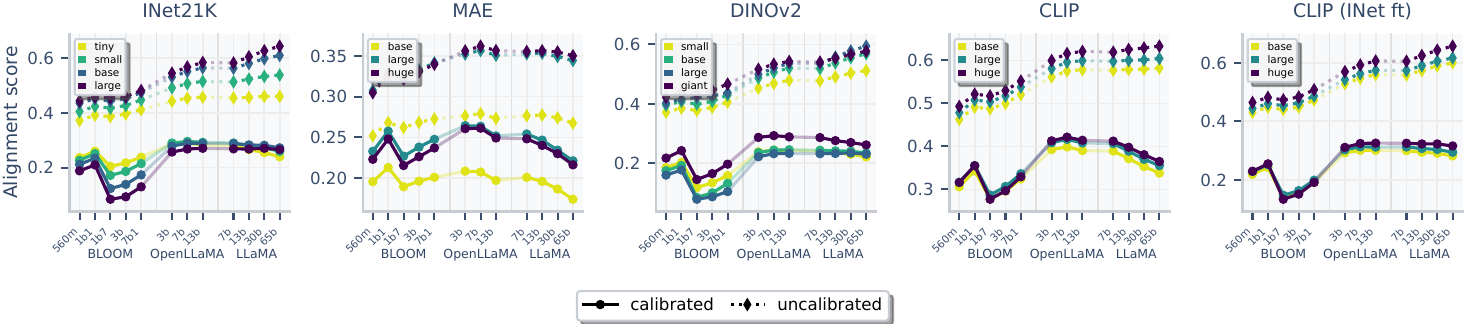}
    \caption{$\alpha = 0.10$}
  \end{subfigure}
  \caption{
  \textbf{Sensitivity to $\alpha$ for CKA RBF.}
  The same pattern holds: no convergence trend at any significance level.
  }
  \label{fig:alpha-cka-rbf}
\end{figure}

\begin{figure}[htbp]
  \centering
  \begin{subfigure}[t]{\linewidth}
    \includegraphics[width=\linewidth]{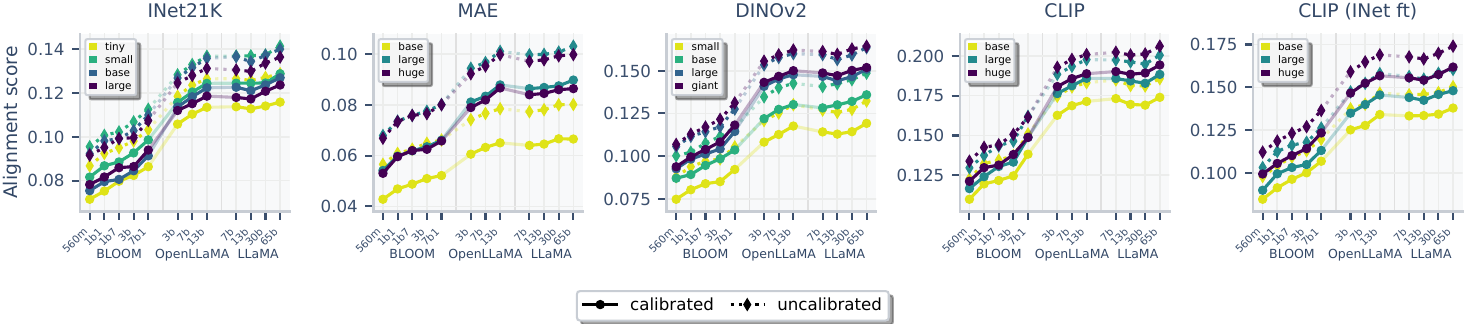}
    \caption{$\alpha = 0.01$}
  \end{subfigure}
  \begin{subfigure}[t]{\linewidth}
    \includegraphics[width=\linewidth]{figures/prh_alignment_lines_by_family_fdr_mutual_knn_k10.pdf}
    \caption{$\alpha = 0.05$ (default)}
  \end{subfigure}
  \begin{subfigure}[t]{\linewidth}
    \includegraphics[width=\linewidth]{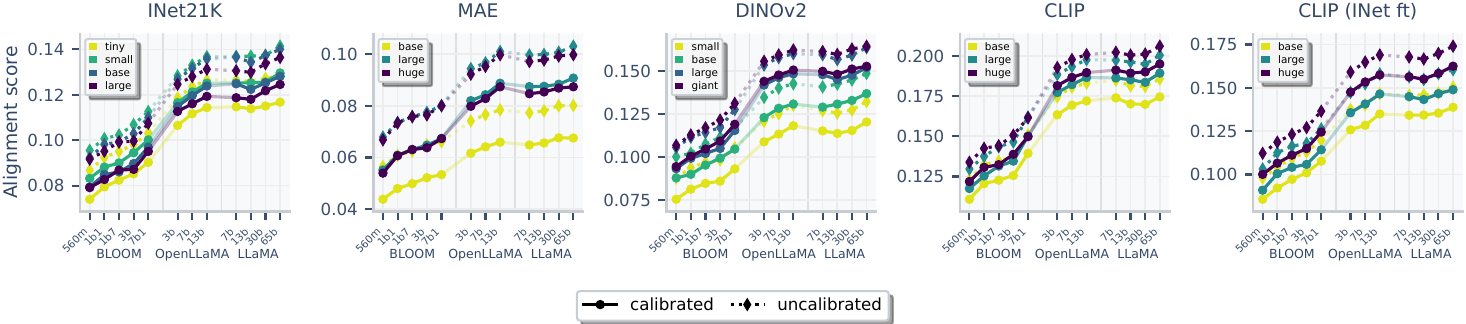}
    \caption{$\alpha = 0.10$}
  \end{subfigure}
  \caption{
  \textbf{Sensitivity to $\alpha$ for \gls{mknn} ($k=10$).}
  Local alignment and its scaling trend persist across all significance levels.
  }
  \label{fig:alpha-mknn}
\end{figure}

%% file: references.bib
@inproceedings{pmlr-v97-kornblith19a,
  title        = {{Similarity of neural network representations revisited}},
  author       = {Kornblith, Simon and Norouzi, Mohammad and Lee, Honglak and Hinton, Geoffrey},
  year         = 2019,
  booktitle    = {International Conference on Machine Learning}
}

@inproceedings{raghu2017svcca,
  title        = {{SVCCA: Singular Vector Canonical Correlation Analysis for Deep Learning Dynamics and Interpretability}},
  author       = {Raghu, Maithra and Gilmer, Justin and Yosinski, Jason and Sohl-Dickstein, Jascha},
  year         = 2017,
  booktitle    = {Advances in Neural Information Processing Systems}
}

@misc{aristotle,
  title        = {{Categories}},
  author       = {Aristotle},
  year         = {ca. 350 B.C.E}
}

@article{kriegeskorte2008rsa,
  title        = {{Representational similarity analysis--connecting the branches of systems neuroscience}},
  author       = {Kriegeskorte, Nikolaus and Mur, Marieke and Bandettini, Peter A.},
  year         = 2008,
  journal      = {Frontiers in Systems Neuroscience}
}

@inproceedings{williams2021shapemetrics,
  title        = {{Generalized Shape Metrics on Neural Representations}},
  author       = {Williams, Alex H. and Kunz, Erin and Kornblith, Simon and Linderman, Scott W.},
  year         = 2021,
  booktitle    = {Advances in Neural Information Processing Systems}
}

@inproceedings{pmlr-v235-huh24a,
  title        = {{Position: The platonic representation hypothesis}},
  author       = {Huh, Minyoung and Cheung, Brian and Wang, Tongzhou and Isola, Phillip},
  year         = 2024,
  booktitle    = {International Conference on Machine Learning}
}

@article{nichols2002permutation,
  title        = {{Nonparametric permutation tests for functional neuroimaging: a primer with examples}},
  author       = {Nichols, Thomas E. and Holmes, Andrew P.},
  year         = 2002,
  journal      = {Human Brain Mapping}
}

@article{phipson2010pvalues,
  title        = {{Permutation P-values Should Never Be Zero: Calculating Exact P-values When Permutations Are Randomly Drawn.}},
  author       = {Phipson, Belinda and Smyth, Gordon K},
  year         = 2010,
  journal      = {Statistical Applications in Genetics \& Molecular Biology}
}

@article{wachter1978strong,
  title        = {{The strong limits of random matrix spectra for sample matrices of independent elements}},
  author       = {Wachter, Kenneth W},
  year         = 1978,
  journal      = {The Annals of Probability}
}

@article{muller2002random,
  title        = {{A random matrix model of communication via antenna arrays}},
  author       = {M{\"u}ller, Ralf R},
  year         = 2002,
  journal      = {IEEE Transactions on information theory}
}

@book{cramer1999mathematical,
  title        = {{Mathematical methods of statistics}},
  author       = {Cram{\'e}r, Harald},
  year         = 1999
}

@book{good2005permutation,
  title        = {{Permutation, parametric and bootstrap tests of hypotheses}},
  author       = {Good, Phillip},
  year         = 2005
}

@book{westfall1993resampling,
  title        = {{Resampling-based multiple testing: Examples and methods for p-value adjustment}},
  author       = {Westfall, Peter H and Young, S Stanley},
  year         = 1993
}

@book{embrechts2013modelling,
  title        = {{Modelling Extremal Events: for Insurance and Finance}},
  author       = {Embrechts, Paul and Kl{\"u}ppelberg, Claudia and Mikosch, Thomas},
  year         = 2013,
  series       = {Stochastic Modelling and Applied Probability}
}

@article{klabunde2023survey,
  title        = {{Similarity of neural network models: A survey of functional and representational measures}},
  author       = {Klabunde, Max and Schumacher, Tobias and Strohmaier, Markus and Lemmerich, Florian},
  year         = 2025,
  journal      = {ACM Computing Surveys}
}

@article{ding2021grounding,
  title        = {{Grounding representation similarity through statistical testing}},
  author       = {Ding, Frances and Denain, Jean-Stanislas and Steinhardt, Jacob},
  year         = 2021,
  journal      = {Advances in Neural Information Processing Systems}
}

@inproceedings{harvey2024decodable,
  title        = {{What Representational Similarity Measures Imply about Decodable Information}},
  author       = {Harvey, Sarah E and Lipshutz, David and Williams, Alex H},
  year         = 2024,
  booktitle    = {Proceedings of UniReps: the Second Edition of the Workshop on Unifying Representations in Neural Models},
  organization = {PMLR}
}

@article{bo2024functional,
  title        = {{Evaluating representational similarity measures from the lens of functional correspondence}},
  author       = {Bo, Yiqing and Soni, Ansh and Srivastava, Sudhanshu and Khosla, Meenakshi},
  year         = 2024,
  journal      = {arXiv preprint arXiv:2411.14633}
}

@article{murphy2024biasedcka,
  title        = {{Correcting biased centered kernel alignment measures in biological and artificial neural networks}},
  author       = {Murphy, Alex and Zylberberg, Joel and Fyshe, Alona},
  year         = 2024,
  journal      = {arXiv preprint arXiv:2405.01012}
}

@article{chun2025sparsecka,
  title        = {{Estimating Neural Representation Alignment from Sparsely Sampled Inputs and Features}},
  author       = {Chun, Chanwoo and Canatar, Abdulkadir and Chung, SueYeon and Lee, Daniel D},
  year         = 2025,
  journal      = {arXiv preprint arXiv:2502.15104}
}

@article{cui2022deconfounded,
  title        = {{Deconfounded representation similarity for comparison of neural networks}},
  author       = {Cui, Tianyu and Kumar, Yogesh and Marttinen, Pekka and Kaski, Samuel},
  year         = 2022,
  journal      = {Advances in Neural Information Processing Systems}
}

@article{diedrichsen2020wuc,
  title        = {{Comparing representational geometries using whitened unbiased-distance-matrix similarity}},
  author       = {Diedrichsen, J{\"o}rn and Berlot, Eva and Mur, Marieke and Sch{\"u}tt, Heiko H and Shahbazi, Mahdiyar and Kriegeskorte, Nikolaus},
  year         = 2021,
  journal      = {Neurons, Behavior, Data Analysis, and Theory}
}

@article{cai2019brsa,
  title        = {{Representational structure or task structure? Bias in neural representational similarity analysis and a Bayesian method for reducing bias}},
  author       = {Cai, Ming Bo and Schuck, Nicolas W and Pillow, Jonathan W and Niv, Yael},
  year         = 2019,
  journal      = {PLoS Computational Biology}
}

@article{smilde2009rv,
  title        = {{Matrix correlations for high-dimensional data: the modified RV-coefficient}},
  author       = {Smilde, Age K and Kiers, Henk AL and Bijlsma, Sabina and Rubingh, CM and Van Erk, MJ},
  year         = 2009,
  journal      = {Bioinformatics}
}

@incollection{hotelling1992relations,
  title        = {{Relations between two sets of variates}},
  author       = {Hotelling, Harold},
  year         = 1992,
  booktitle    = {Breakthroughs in Statistics: Methodology and Distribution}
}

@article{morcos2018insights,
  title        = {{Insights on representational similarity in neural networks with canonical correlation}},
  author       = {Morcos, Ari and Raghu, Maithra and Bengio, Samy},
  year         = 2018,
  journal      = {Advances in Neural Information Processing Systems}
}

@article{song2012hsic,
  title        = {{Feature Selection via Dependence Maximization}},
  author       = {Song, Le and Smola, Alex and Gretton, Arthur and Bedo, Justin and Borgwardt, Karsten},
  year         = 2012,
  journal      = {Journal of Machine Learning Research}
}

@inproceedings{zhu2025dynamic,
  title        = {{Dynamic Reflections: Probing Video Representations with Text Alignment}},
  author       = {Zhu, Tyler and Han, Tengda and Guibas, Leonidas and P{\u{a}}tr{\u{a}}ucean, Viorica and Ovsjanikov, Maks},
  year         = 2026,
  booktitle    = {International Conference on Learning Representations}
}

@article{schrimpf2018brain,
  title        = {{Brain-score: Which artificial neural network for object recognition is most brain-like?}},
  author       = {Schrimpf, Martin and Kubilius, Jonas and Hong, Ha and Majaj, Najib J and Rajalingham, Rishi and Issa, Elias B and Kar, Kohitij and Bashivan, Pouya and Prescott-Roy, Jonathan and Geiger, Franziska and others},
  year         = 2018,
  journal      = {BioRxiv}
}

@article{raugel2025disentangling,
  title        = {{Disentangling the factors of convergence between brains and computer vision models}},
  author       = {Raugel, Jos{\'e}phine and Szafraniec, Marc and Vo, Huy V and Couprie, Camille and Labatut, Patrick and Bojanowski, Piotr and Wyart, Valentin and King, Jean-R{\'e}mi},
  year         = 2025,
  journal      = {arXiv preprint arXiv:2508.18226}
}

@article{marcos2025convergent,
  title        = {{Convergent transformations of visual representation in brains and models}},
  author       = {Marcos-Manch{\'o}n, Pablo and Fuentemilla, Llu{\'\i}s},
  year         = 2025,
  journal      = {arXiv preprint arXiv:2507.13941}
}

@inproceedings{beyer1999nearest,
  title        = {{When is ``nearest neighbor'' meaningful?}},
  author       = {Beyer, Kevin and Goldstein, Jonathan and Ramakrishnan, Raghu and Shaft, Uri},
  year         = 1999,
  booktitle    = {International Conference on Database Theory},
  organization = {Springer}
}

@inproceedings{aggarwal2001surprising,
  title        = {{On the surprising behavior of distance metrics in high dimensional space}},
  author       = {Aggarwal, Charu C and Hinneburg, Alexander and Keim, Daniel A},
  year         = 2001,
  booktitle    = {International Conference on Database Theory},
  organization = {Springer}
}

@article{holm1979simple,
  title        = {{A simple sequentially rejective multiple test procedure}},
  author       = {Holm, Sture},
  year         = 1979,
  journal      = {Scandinavian Journal of Statistics}
}

@article{benjamini1995controlling,
  title        = {{Controlling the false discovery rate: a practical and powerful approach to multiple testing}},
  author       = {Benjamini, Yoav and Hochberg, Yosef},
  year         = 1995,
  journal      = {Journal of the Royal Statistical Society: Series B (Methodological)}
}

@article{neyshabur2020being,
  title        = {{What is being transferred in transfer learning?}},
  author       = {Neyshabur, Behnam and Sedghi, Hanie and Zhang, Chiyuan},
  year         = 2020,
  journal      = {Advances in Neural Information Processing Systems}
}

@inproceedings{nguyen2021do,
  title        = {{Do Wide and Deep Networks Learn the Same Things? Uncovering How Neural Network Representations Vary with Width and Depth}},
  author       = {Thao Nguyen and Maithra Raghu and Simon Kornblith},
  year         = 2021,
  booktitle    = {International Conference on Learning Representations}
}

@article{bonferroni1936teoria,
  title        = {{Teoria statistica delle classi e calcolo delle probabilit\`a}},
  author       = {Bonferroni, Carlo},
  year         = 1936,
  journal      = {Pubblicazioni del R Istituto Superiore di Scienze Economiche e Commerciali di Firenze}
}

@article{robert1976unifying,
  title        = {{A unifying tool for linear multivariate statistical methods: the RV-coefficient}},
  author       = {Robert, Paul and Escoufier, Yves},
  year         = 1976,
  journal      = {Journal of the Royal Statistical Society: Series C (Applied Statistics)}
}

@book{lehmann2005testing,
  title        = {{Testing statistical hypotheses}},
  author       = {Lehmann, Erich Leo and Romano, Joseph P},
  year         = 2005
}

@article{bolya2025perception-encoder,
  title        = {{Perception Encoder: The best visual embeddings are not at the output of the network}},
  author       = {Daniel Bolya and Po-Yao Huang and Peize Sun and Jang Hyun Cho and Andrea Madotto and Chen Wei and Tengyu Ma and Jiale Zhi and Jathushan Rajasegaran and Hanoona Abdul Rasheed and Junke Wang and Marco Monteiro and Hu Xu and Shiyu Dong and Nikhila Ravi and Shang-Wen Li and Piotr Dollar and Christoph Feichtenhofer},
  year         = 2025,
  journal      = {Advances in Neural Information Processing Systems}
}

@article{cho2025perceptionlm,
  title        = {{Perception{LM}: Open-Access Data and Models for Detailed Visual Understanding}},
  author       = {Jang Hyun Cho and Andrea Madotto and Effrosyni Mavroudi and Triantafyllos Afouras and Tushar Nagarajan and Muhammad Maaz and Yale Song and Tengyu Ma and Shuming Hu and Suyog Jain and Miguel Martin and Huiyu Wang and Hanoona Abdul Rasheed and Peize Sun and Po-Yao Huang and Daniel Bolya and Nikhila Ravi and Shashank Jain and Tammy Stark and Seungwhan Moon and Babak Damavandi and Vivian Lee and Andrew Westbury and Salman Khan and Philipp Kraehenbuehl and Piotr Dollar and Lorenzo Torresani and Kristen Grauman and Christoph Feichtenhofer},
  year         = 2025,
  journal      = {Advances in Neural Information Processing Systems}
}

@article{tong2022videomae,
  title        = {{Video{MAE}: Masked autoencoders are data-efficient learners for self-supervised video pre-training}},
  author       = {Tong, Zhan and Song, Yibing and Wang, Jue and Wang, Limin},
  year         = 2022,
  journal      = {Advances in Neural Information Processing Systems}
}

@article{weenink2003canonical,
  title        = {{Canonical correlation analysis}},
  author       = {Weenink, David},
  year         = 2003,
  journal      = {Proceedings of the Institute of Phonetic Sciences of the University of Amsterdam}
}

@inproceedings{Maniparambil_2024_CVPR,
  title        = {{Do Vision and Language Encoders Represent the World Similarly?}},
  author       = {Maniparambil, Mayug and Akshulakov, Raiymbek and Djilali, Yasser Abdelaziz Dahou and El Amine Seddik, Mohamed and Narayan, Sanath and Mangalam, Karttikeya and O'Connor, Noel E.},
  year         = 2024,
  booktitle    = {Conference on Computer Vision and Pattern Recognition}
}

@inproceedings{tjandrasuwita2025understanding,
  title        = {{Understanding the Emergence of Multimodal Representation Alignment}},
  author       = {Megan Tjandrasuwita and Chanakya Ekbote and Liu Ziyin and Paul Pu Liang},
  year         = 2025,
  booktitle    = {International Conference on Machine Learning}
}

@inproceedings{srinivasan2021wit,
  title        = {{Wit: Wikipedia-based image text dataset for multimodal multilingual machine learning}},
  author       = {Srinivasan, Krishna and Raman, Karthik and Chen, Jiecao and Bendersky, Michael and Najork, Marc},
  year         = 2021,
  booktitle    = {International ACM SIGIR Conference on Research and Development in Information Retrieval}
}

@inproceedings{klabunde2025resi,
  title        = {{ReSi: A Comprehensive Benchmark for Representational Similarity Measures}},
  author       = {Max Klabunde and Tassilo Wald and Tobias Schumacher and Klaus Maier-Hein and Markus Strohmaier and Florian Lemmerich},
  year         = 2025,
  booktitle    = {International Conference on Learning Representations}
}

@article{soni2024conclusions,
  title        = {{Conclusions about Neural Network to Brain Alignment are Profoundly Impacted by the Similarity Measure}},
  author       = {Soni, Ansh and Srivastava, Sudhanshu and Kording, Konrad and Khosla, Meenakshi},
  year         = 2024,
  journal      = {bioRxiv}
}

@article{kapoor2025bridging,
  title        = {{Bridging Critical Gaps in Convergent Learning: How Representational Alignment Evolves Across Layers, Training, and Distribution Shifts}},
  author       = {Kapoor, Chaitanya and Srivastava, Sudhanshu and Khosla, Meenakshi},
  year         = 2025,
  journal      = {Advances in Neural Information Processing Systems}
}

@inproceedings{schaeffer2024position,
  title        = {{Position: Maximizing Neural Regression Scores May Not Identify Good Models of the Brain}},
  author       = {Schaeffer, Rylan and Khona, Mikail and Chandra, Sarthak and Ostrow, Mitchell and Miranda, Brando and Koyejo, Sanmi},
  year         = 2024,
  booktitle    = {NeurIPS 2024 Workshop on Unifying Representations in Neural Models (UniReps)}
}

@book{livan2018introduction,
  title        = {{Introduction to Random Matrices: Theory and Practice}},
  author       = {Livan, Giacomo and Novaes, Marcel and Vivo, Pierpaolo},
  year         = 2018,
  series       = {SpringerBriefs in Mathematical Physics}
}

@article{groger2026limited,
  title={With Limited Data for Multimodal Alignment, Let the {STRUCTURE} Guide You},
  author={Gr{\"o}ger, Fabian and Wen, Shuo and Le, Huyen and Brbic, Maria},
  journal={Advances in Neural Information Processing Systems},
  year={2025}
}
